%% file: neurips_2020.tex
\documentclass{article}

   \usepackage[nonatbib,preprint]{neurips_2020}

\usepackage[utf8]{inputenc} % allow utf-8 input
\usepackage[T1]{fontenc}    % use 8-bit T1 fonts
\usepackage{hyperref}       % hyperlinks
\usepackage{url}            % simple URL typesetting
\usepackage{booktabs}       % professional-quality tables
\usepackage{amsfonts}       % blackboard math symbols
\usepackage{nicefrac}       % compact symbols for 1/2, etc.
\usepackage{microtype}      % microtypography

\usepackage{graphicx}
\usepackage{subfigure}
\usepackage{amssymb}
\usepackage{amsthm}
\usepackage{amsmath}

\usepackage{algorithm}
\usepackage{algorithmic}

\usepackage{diagbox}

\usepackage{lipsum}

\newtheorem{theorem}{Theorem}
\newtheorem*{theorem*}{Theorem}
\newtheorem{Lemma}{Lemma}

\DeclareMathOperator*{\argmin}{arg\,min}

\title{ Black-box Optimizer with Implicit Natural Gradient}

\author{%
 Yueming Lyu \\
  Australian Artificial Intelligence Institute\\
  University of Technology Sydney\\
  \texttt{yueminglyu@gmail.com} \\
   \And
   Ivor Tsang \\
   Australian Artificial Intelligence Institute \\
   University of Technology Sydney \\
   \texttt{Ivor.Tsang@uts.edu.au} \\
}

\begin{document}

\maketitle

\begin{abstract}
  Black-box optimization is primarily important for many compute-intensive applications, including reinforcement learning (RL), robot control, etc.
This paper presents a novel theoretical framework for black-box optimization, in which our method performs stochastic update with implicit natural gradient  of an exponential-family distribution.  Theoretically, we prove the convergence rate of our framework with full matrix update for convex functions. Our theoretical results also hold for continuous non-differentiable black-box functions. Our methods are very simple and contain less hyper-parameters than CMA-ES~\cite{hansen2006cma}.  Empirically, our method with full matrix update achieves  a competitive performance compared with one of the state-of-the-art method CMA-ES on  benchmark test problems. Moreover, our methods can achieve high optimization precision  on some challenging test functions (e.g., $l_1$-norm ellipsoid test problem and Levy test problem), while methods with explicit natural gradient, i.e., IGO~\cite{ollivier2017information} with full matrix update can not.  This shows the efficiency of our methods.
\end{abstract}

\section{Introduction}
Given a proper function $f({\boldsymbol{x}}): \mathbb{R}^d \rightarrow \mathbb{R}$ such that $f({\boldsymbol{x}}) > -\infty$, we aim at minimizing $f({\boldsymbol{x}})$ by using function queries only, which is known as black-box optimization.  It has a wide range of applications, such as automatic hyper-parameters tuning in machine learning and  computer vision problems \cite{Nips2012practical}, adjusting parameters for robot control and reinforcement learning~\cite{choromanskicomplexity,robot,BORL}, black-box architecture search in engineering design~\cite{engineerDisign} and drug discovery~\cite{drugdiscovery}.

 Several kinds of approaches have been widely studied for black-box optimization, including Bayesian optimization (BO) methods~\cite{gpucb,bull,lyu2019efficient},  evolution strategies (ES)~\cite{back1991survey,hansen2006cma} and genetic algorithms (GA)~\cite{srinivas1994genetic}.  Among them, Bayesian optimization  methods are good at dealing with low-dimensional expensive black-box optimization, while    ES methods  are better for relatively high-dimensional problems with cheaper evaluations compared with BO  methods. ES-type algorithms can well support parallel evaluation, and 
  have drawn more and more attention because of its success in reinforcement learning problems~\cite{choromanski2018structured,salimans2017evolution,Liu2019}, recently.  
  
%akimoto2019diagonal
CMA-ES~\cite{hansen2006cma} is one of state-of-the-art ES methods with many successful applications. It uses second-order information to search candidate solutions by 
updating the mean and covariance matrix of the likelihood  of candidate distributions.
Despite  its successful performance, the update rule combines several sophisticated components, which is not well understood. 
Wierstra et al. show that directly applying standard reinforce gradient descent is very sensitive to variance  in high precision search for  black-box optimization~\cite{NES}. Thus, they propose Natural evolution strategies (NES)~\cite{NES} to  estimate the natural gradient for black-box optimization. However, they use the Monte Carlo sampling to approximate the Fisher information matrix (FIM), which incurs additional error and computation cost unavoidably. Along this line,
\cite{akimoto2010bidirectional} show the connection between the rank-$\mu$ update of CMA-ES and NES~\cite{NES}. 
\cite{ollivier2017information}  further show that several ES methods can be included in an unified  framework. Despite  these theoretical attempts, the practical performance of these methods is still inferior to CMA-ES. Moreover, these works do not provide any convergence rate analysis, which is the key insight to expedite  black-box optimizations.

Another line of research for ES-type algorithms is to reduce the variance of gradient estimators. Choromanski et al.~\cite{choromanski2018structured} 
proposed to employ Quasi Monte Carlo (QMC) sampling to achieve more accurate gradient estimates. 
Recently,  they further proposed to construct gradient estimators based on active subspace techniques~\cite{choromanski2019complexity}. Although these works can reduce sample complexity, how does the variance of these estimators influence the convergence rate remains unclear.  

To take advantage of second-order information for the acceleration of  black-box optimizations, we propose a novel theoretical  framework: stochastic Implicit Natural Gradient Optimization (INGO) algorithms,  from the perspective of information geometry.   Raskutti et al.~\cite{raskutti2015information} give a method to compute the Fisher information matrix implicitly using exact gradients, which is impossible for black-box optimization; while our methods and analysis focus on  black-box optimization. To the best of our knowledge, we are the first to design stochastic implicit natural gradient algorithms w.r.t natural parameters for black-box optimization. Our methods take a stochastic black-box estimate instead of the exact gradient to update.  Theoretically, this update is equivalent to a stochastic  natural gradient step w.r.t. natural parameters of an  exponential-family distribution. 
 Our contributions are summarized as follows:
 \begin{itemize}
     \item   We propose a novel stochastic implicit natural gradient descent framework  for black-box optimization (INGO).  To the best of our knowledge, we are the first to design stochastic implicit natural gradient algorithms w.r.t natural parameters for black-box optimization. We propose efficient algorithms   for both continuous  and discrete black-box optimization. Our methods construct stochastic black-box update  without computing the FIM.   Our method can adaptively control the stochastic update by taking advantage of the second-order information, which is able to accelerate convergence and is primarily important for ill-conditioned problems. Moreover,  our methods have  fewer  hyperparameters and are much simpler than CMA-ES.

     \item Theoretically, we prove the convergence rate of our continuous optimization methods for  convex functions. Our theoretical results also hold for non-differentiable convex black-box functions. This is distinct from  most literature works that need Lipschitz continuous gradients ($L$-smooth)  assumption.  Our theoretical results can include many interesting problems with non-smooth structures. We also show that reducing variance of the black-box gradient estimators by orthogonal  sampling can lead to  a small regret bound.

     \item Empirically, our continuous optimization method achieves  a competitive performances compared with the state-of-the-art method CMA-ES on  benchmark problems. We find that our method with full matrix update can obtain higher optimization precision compared with IGO~\cite{ollivier2017information} on some challenging problems.    We further show the effectiveness of our methods  on RL control problems. Moreover, our discrete optimization algorithm outperforms  GA method.
 \end{itemize}

 \section{Notation and Symbols}
 
 Denote $\|\cdot\|_2$ and $\|\cdot\|_F$ as the spectral norm and Frobenius norm for matrices, respectively. Define $\|Y\|_{tr}:=\sum_i {|\lambda_i|}$, where $\lambda_i$ denotes the $i^{th}$ eigenvalue of matrix $Y$.  Notation $\|\cdot\|_2$ will also denote $l_2$-norm for vectors. Symbol $\left<\cdot, \cdot \right>$ denotes inner product under $l_2$-norm  for vectors   and inner product under Frobenius norm for matrices. Define $\|\boldsymbol{x}\|_C: = \sqrt{ \left<\boldsymbol{x},C\boldsymbol{x} \right> }$. Denote $\mathcal{S}^{+}$ and $\mathcal{S}^{++}$ as the set of positive semi-definite matrices and the set of positive definite matrices, respectively.  
 Denote  $\Sigma^{\frac{1}{2}}$  as the symmetric positive semi-definite matrix such that $\Sigma = \Sigma^{\frac{1}{2}}\Sigma^{\frac{1}{2}}$ for $\Sigma \in \mathcal{S}^{+}$.

\section{Implicit Natural Gradient Optimization}

\subsection{Optimization with  Exponential-family Sampling}
We aim at minimizing a proper  function $f({\boldsymbol{x}})$, ${\boldsymbol{x}} \in \mathcal{X}$ with only function  queries, which is known as black-box optimization. 
Due to the lack of gradient information for black-box optimization, we here present an exponential-family sampling trick to relax any black-box optimization problem. Specifically, 
the objective is relaxed as the expectation of $f({\boldsymbol{x}})$ under a parametric distribution $p({\boldsymbol{x}};\eta)$ with parameter $\eta$, i.e., $J({\boldsymbol{\eta}}):=\mathbb{E}_{p({\boldsymbol{x}};{\boldsymbol{\eta}})}[f({\boldsymbol{x}})]$~\cite{NES}. The optimal parameter ${\boldsymbol{\eta}}$ is found by minimizing  $J({\boldsymbol{\eta}})$ as $\min _{{\boldsymbol{\eta}} } \left\{ \mathbb{E}_{p({\boldsymbol{x}};{\boldsymbol{\eta}})}[f({\boldsymbol{x}})]  \right \}$
% \begin{align}
% \label{Robj}
%     \min _{{\boldsymbol{\eta}} } \left\{ \mathbb{E}_{p({\boldsymbol{x}};{\boldsymbol{\eta}})}[f({\boldsymbol{x}})]  \right \} .
% \end{align}
This relaxed problem is minimized when the probability mass is all assigned on the minimum of $f({\boldsymbol{x}})$. The distribution $p$ is the sampling distribution for black-box function queries. Note, $p$ can be  either  continuous  or discrete.

In this work, we assume that the distribution $p({\boldsymbol{x}};{\boldsymbol{\eta}})$ is an  exponential-family distribution:
\begin{align}
  p({\boldsymbol{x}};{\boldsymbol{\eta}}) = h({\boldsymbol{x}})\exp{\{\left < \phi({\boldsymbol{x}}),{\boldsymbol{\eta}} \right > -A({\boldsymbol{\eta}}) \}} ,  
\end{align} 
where ${\boldsymbol{\eta}}$ and $\phi({\boldsymbol{x}})$ are the natural parameter and sufficient statistic, respectively. And $A(\eta)$ is the log partition function defined as $
   A(\eta) = \log \int \exp{\{\left < \phi({\boldsymbol{x}}),{\boldsymbol{\eta}} \right >} h({\boldsymbol{x}})  {d\boldsymbol{x}}$.
% \begin{align}
%   A(\eta) = \log \int \exp{\{\left < \phi({\boldsymbol{x}}),{\boldsymbol{\eta}} \right >} h({\boldsymbol{x}})  {d\boldsymbol{x}}.
% \end{align}

It is named as minimal exponential-family distribution   when there is a one-to-one mapping between the mean parameter ${\boldsymbol{m}}:= \mathbb{E}_{p}[\phi({\boldsymbol{x}})]$ and natural parameter  ${\boldsymbol{\eta}}$. This one-to-one mapping ensures that we can reparameterize $J({\boldsymbol{\eta}})$ as
$ \tilde{J}({\boldsymbol{m}}) = J({\boldsymbol{\eta}})$~\cite{amari2016information,khan2017conjugate}. $ \tilde{J}$ is w.r.t parameter $\boldsymbol{m}$, while $J$ is w.r.t parameter~$\boldsymbol{\eta}$.

To minimize the objective $ \tilde{J}({\boldsymbol{m}}) $, we desire the updated distribution  lying in a trust region of the previous distribution at each step. Formally, we update the mean parameters by solving the following optimization problem.
\begin{align}
 {\boldsymbol{m}}_{t+1} \!\!=\!\! \arg\min _{{\boldsymbol{m}}}
 \!\left<{\boldsymbol{m}}\! ,\!  \nabla _ {\boldsymbol{m}} \tilde{J}({\boldsymbol{m}_t})\! \right > \!+\! \frac{1}{\beta_t} \text{KL}\left( p_m \| p_{m_t}  \right),
 \label{Mop}
\end{align}
where $\nabla _ {\boldsymbol{m}} \tilde{J}({\boldsymbol{m}}_t)$ denotes the gradient at ${\boldsymbol{m}} = {\boldsymbol{m}}_t$.

The KL-divergence term measures how close the updated distribution and the previous distribution. For an  exponential-family distribution, the KL-divergence term in (\ref{Mop}) is equal to Bregman divergence between ${\boldsymbol{m}}$ and ${\boldsymbol{m}}_t$~\cite{azoury2001relative}:
\begin{equation}
% \resizebox{0.75\columnwidth}{!}{
%   $$,
%   }
\text{KL}\left( p_m \| p_{m_t}  \right) \!=\!    A^*({\boldsymbol{m}}) \!-\! A^*({\boldsymbol{m}}_t) \!-\!  \left<{\boldsymbol{m}}\!-\!{\boldsymbol{m}}_t, \nabla _{\boldsymbol{m}}A^*({\boldsymbol{m}}_t)  \right>,
\end{equation}
where $A^*({\boldsymbol{m}})$ is the convex conjugate of $A({\boldsymbol{\eta}})$. Thus,
the problem (\ref{Mop}) is a convex optimization problem, and it has a closed-form solution.

\subsection{Implicit Natural Gradient}

\textbf{Intractability of Natural Gradient for Black-box Optimization:} Natural gradient~\cite{amari1998natural}  can capture information geometry structure during optimization, which  enables us to take advantage of the second-order information to accelerate convergence. Direct computation of natural gradient needs the inverse of Fisher information matrix (FIM), which needs to estimate the FIM. The method in \cite{raskutti2015information} provides an alternative way to compute natural gradient without computation of FIM.  However, it relies on the exact gradient, which is  impossible for black-box optimization. 

%\ref{Opdiscrete}
 Hereafter, we propose a novel stochastic implicit natural gradient algorithms for black-box optimization of continuous and discrete variables in Section~\ref{OpGaussian} and Section~A (in the supplement), respectively.
We first show how to compute the implicit natural gradient. In problem Eq.(\ref{Mop}), we
take the derivative w.r.t ${\boldsymbol{m}}$, and set it to zero, also note that $\nabla _{\boldsymbol{m}}A^*({\boldsymbol{m}}) = {\boldsymbol{\eta}}$~\cite{raskutti2015information}, we can obtain that
\begin{align}
      \boldsymbol{\eta}_{t+1} = \boldsymbol{\eta}_{t} - \beta _t \nabla _ {\boldsymbol{m}} \tilde{J}({\boldsymbol{m}}_t)
\label{updateEta}      
\end{align}

Natural parameters $\boldsymbol{\eta}$ of the distribution lies on a Riemannian manifold with metric tensor specified by the Fisher Information Matrix:
\begin{align}
    {\boldsymbol{F}}({\boldsymbol{\eta}}):= \mathbb{E}_p \left[ \nabla _{\eta} \log {p({\boldsymbol{x}};{\boldsymbol{\eta}})}  \nabla _{\eta} \log {p({\boldsymbol{x}};{\boldsymbol{\eta}})}^\top   \right]
\end{align}
For exponential-family with the minimal representation, the natural gradient has a simple form for computation. 
\begin{theorem}
{\cite{khan2018fast,raskutti2015information}} For an exponential-family in the minimal representation, the natural gradient w.r.t ${\boldsymbol{\eta}}$ is equal to the gradient w.r.t. ${\boldsymbol{m}}$, i.e.,
\begin{align}
    {\boldsymbol{F}}({\boldsymbol{\eta}})^{-1} \nabla _{\boldsymbol{\eta}}J({\boldsymbol{\eta}}) = \nabla _ {\boldsymbol{m}} \tilde{J}({\boldsymbol{m}})
\end{align}
\end{theorem}
{\bf Remark:} Theorem 1 can be easily obtained by the chain rule and the fact ${\boldsymbol{F}}({\boldsymbol{\eta}})=\frac{\partial^2A({\boldsymbol{\eta}})}{\partial {\boldsymbol{\eta}} \partial {\boldsymbol{\eta}}^\top }$. It  enables us to compute the natural gradient implicitly without computing the inverse of the Fisher information matrix. As shown in  Theorem 1, the update rule in (\ref{updateEta}) is equivalent to the natural gradient update w.r.t ${\boldsymbol{\eta}}$ in (\ref{Nupdate}):
\begin{align}
\label{Nupdate}
    \boldsymbol{\eta}_{t+1} = \boldsymbol{\eta}_{t} - \beta _t {\boldsymbol{F}}({\boldsymbol{\eta}}_t)^{-1} \nabla _{\eta}J({\boldsymbol{\eta}}_t)
\end{align}
Thus, update rule in (\ref{updateEta})  selects the steepest descent direction along the Riemannian manifold induced by the Fisher information matrix as natural gradient descent. It can take the second-order information to accelerate convergence.

\section{Update Rule for Gaussian Sampling}
\label{OpGaussian}
We first present an update method for the case of Gaussian sampling for continuous optimization. For other distributions, we can derive the update rule in a similar manner. We present update methods for discrete optimization in the supplement due to the space limitation.

For a Gaussian distribution $p:= \mathcal{N}({\boldsymbol{\mu}}, \Sigma)$  with mean ${\boldsymbol{\mu}}$ and covariance matrix $\Sigma$, the natural parameters $\eta = \{ {\boldsymbol{\eta}}_1 , {\boldsymbol{\eta}}_2 \}$ are given as $\boldsymbol{\eta}_1 := \Sigma^{-1}{\boldsymbol{\mu}} 
  $ and $ \boldsymbol{\eta}_2 := -\frac{1}{2}\Sigma^{-1}
  \label{eta2}$
% \begin{align}
% \label{eta1}
%   &  \boldsymbol{\eta}_1 := \Sigma^{-1}{\boldsymbol{\mu}} \\
%   & \boldsymbol{\eta}_2 := -\frac{1}{2}\Sigma^{-1}
%   \label{eta2}
% \end{align}
The related mean parameters ${\boldsymbol{m}}= \{ {\boldsymbol{m}}_1 , {\boldsymbol{m}}_2  \}$ are given as $ {\boldsymbol{m}}_1 :=  \mathbb{E}_p[{\boldsymbol{x}}] = {\boldsymbol{\mu}} 
$ and $ {\boldsymbol{m}}_2 :=  \mathbb{E}_p[{\boldsymbol{x}}{\boldsymbol{x}^\top}] ={\boldsymbol{\mu}}{\boldsymbol{\mu}^\top} + \Sigma
\label{m2}$.
% \begin{align}
% \label{m1}
%   &  {\boldsymbol{m}}_1 :=  \mathbb{E}_p[{\boldsymbol{x}}] = {\boldsymbol{\mu}} \\
% & {\boldsymbol{m}}_2 :=  \mathbb{E}_p[{\boldsymbol{x}}{\boldsymbol{x}^\top}] ={\boldsymbol{\mu}}{\boldsymbol{\mu}^\top} + \Sigma
% \label{m2}
% \end{align}

Using the chain rule,  the gradient with respect to mean parameters can be expressed in terms of the gradients w.r.t ${\boldsymbol{\mu}}$ and $\Sigma$~\cite{khan2017conjugate,khan2018fast2} as:
\begin{align}
\label{g_m1}
  & \nabla _{\boldsymbol{m}_1}\tilde{J}({\boldsymbol{m}}) =
  \nabla _{\boldsymbol{\mu}}\tilde{J}({\boldsymbol{m}}) - 2 [\nabla _{\Sigma}\tilde{J}({\boldsymbol{m}})]{\boldsymbol{\mu}} \\ 
   & \nabla _{\boldsymbol{m}_2}\tilde{J}({\boldsymbol{m}}) = \nabla _{\Sigma}\tilde{J}({\boldsymbol{m}})
   \label{g_m2}
\end{align}
It follows that
\begin{align}
\label{sigmaupdatem}
  &  \Sigma^{-1}_{t+1} =  \Sigma^{-1}_{t} +  2\beta _t \nabla _ {\Sigma} \tilde{J}({\boldsymbol{m}}_t)   \\ 
  & {\boldsymbol{\mu}}_{t+1} = {\boldsymbol{\mu}}_{t} - \beta_t \Sigma_{t+1}\nabla _ {\boldsymbol{\mu}} \tilde{J}({\boldsymbol{m}}_t)
  \label{muupdatem}
\end{align}

Note that $\tilde{J}({\boldsymbol{m}}) = \mathbb{E}_{p}[f({\boldsymbol{x}})]  $, 
 the gradients of $\tilde{J}({\boldsymbol{m}})$ w.r.t $\mu$ and $\Sigma$ can be obtained by log-likelihood trick as  Theorem~\ref{theB}.

\begin{theorem}\label{theB}~\cite{NES} The gradient of the expectation of an integrable  function $f({\boldsymbol{x}})$ under a Gaussian distribution $p:= \mathcal{N}({\boldsymbol{\mu}}, \Sigma)$ with respect to the mean $\boldsymbol{\mu}$ and the covariance $\Sigma$ can be expressed as Eq.(\ref{mu}) and Eq.(\ref{gsigma}), respectively.
\begin{align}
\label{mu}
  &  \nabla _ {\boldsymbol{\mu}} \mathbb{E}_{p}[f({\boldsymbol{x}})] =  \mathbb{E}_{p}\left[ \Sigma^{-1}({\boldsymbol{x}}-{\boldsymbol{\mu}}) f({\boldsymbol{x}}) \right]  \\
 &     \nabla _ {\Sigma}  \mathbb{E}_{p}[f({\boldsymbol{x}})] 
  = \frac{1}{2}\mathbb{E}_{p}\left[ \left( \Sigma^{-1}({\boldsymbol{x}}-{\boldsymbol{\mu}})({\boldsymbol{x}}-{\boldsymbol{\mu}})^\top \Sigma^{-1} \!\!-\!\! \Sigma^{-1} \right) f({\boldsymbol{x}}) \right]
 \label{gsigma}
\end{align}
\end{theorem}

Together  Theorem 2 with Eq. (\ref{sigmaupdatem}) and (\ref{muupdatem}),  we present the update with only function queries as:
\begin{align}
\label{MC1}
   &  \Sigma^{-1}_{t+1}  = \Sigma^{-1}_{t}  \!  + \! \beta _t \mathbb{E}_{p}\left[ \left( \Sigma^{-1}_t({\boldsymbol{x}}\!-\!{\boldsymbol{\mu}}_t)({\boldsymbol{x}}\!-\!{\boldsymbol{\mu}}_t)^\top \Sigma^{-1}_t \!\!-\!\! \Sigma^{-1}_t \right) f({\boldsymbol{x}}) \right]   \\ 
  & {\boldsymbol{\mu}}_{t+1} = {\boldsymbol{\mu}}_{t} - \beta_t \Sigma_{t+1}\mathbb{E}_{p}\left[ \Sigma^{-1}_t({\boldsymbol{x}}-{\boldsymbol{\mu}}_t) f({\boldsymbol{x}}) \right] 
  \label{MC2}
\end{align}
{\bf Remark:} Our method updates the inverse of the covariance matrix instead of the covariance matrix itself.

\begin{minipage}[t]{0.47\textwidth}
\vspace{0pt}  
\begin{algorithm}[H]
   \caption{INGO}
   \label{alg:INGOu}
\begin{algorithmic}
    \STATE {\bf Input:} Number of Samples  $N$, step-size $\beta$.
    
  \WHILE{Termination condition not satisfied }
  \STATE  Take  i.i.d samples ${\boldsymbol{z}}_i \sim \mathcal{N}({\boldsymbol{0}}, {\boldsymbol{I}})$ for $i \in \{1,\cdots N\}$.
  \STATE Set ${\boldsymbol{x}}_{i}= {\boldsymbol{\mu}}_{t} + \Sigma_{t}^{\frac{1}{2}}  {\boldsymbol{z}}_{i}$ for $i \in \{1,\cdots N\}$.
 \STATE Query the batch   observations $\{f({\boldsymbol{x}}_1 ),...,f({\boldsymbol{x}}_N ) \}$ 
 
 \STATE Compute $\widehat{\sigma}= \text{std}(f({\boldsymbol{x}}_1 ),...,f({\boldsymbol{x}}_N ))$.
 
\STATE Compute $\widehat{\mu}= \frac{1}{N} \sum_{i=1}^{N} {f({\boldsymbol{x}}_i})$. 
 
 \STATE Set \resizebox{1\columnwidth}{!}{$\Sigma_{t+1}^{-1} = \Sigma_t^{-1} + \beta\sum_{i=1}^{N}{\frac{f(\boldsymbol{x}_i)-\widehat{\mu}}{N\widehat{\sigma}}\Sigma_t^{-\frac{1}{2}} \boldsymbol{z}_i\boldsymbol{z}_i^\top \Sigma_t^{-\frac{1}{2}}   }$}.
 
 \STATE Set \resizebox{1\columnwidth}{!}{$ {\boldsymbol{\mu}}_{t+1} = {\boldsymbol{\mu}}_{t} - \beta \sum _{i=1}^{N} \! \frac{f(\boldsymbol{x}_i)-\widehat{\mu}}{N\widehat{\sigma}} \Sigma_{t+1} \Sigma^{-\frac{1}{2}}_t{\boldsymbol{z}}_i   $ }
   \ENDWHILE
\end{algorithmic}
\end{algorithm}
\end{minipage}
\hfill
\begin{minipage}[t]{0.47\textwidth}
\vspace{0pt}  
\begin{algorithm}[H]
   \caption{INGOstep}
   \label{INGOstep}
\begin{algorithmic}
    \STATE {\bf Input:} Number of Samples  $N$, step-size $\beta$.
  \WHILE{Termination condition not satisfied }
  \STATE  Take  i.i.d samples ${\boldsymbol{z}}_i \sim \mathcal{N}({\boldsymbol{0}}, {\boldsymbol{I}})$ for $i \in \{1,\cdots N\}$.
  \STATE Set ${\boldsymbol{x}}_{i}= {\boldsymbol{\mu}}_{t} + \Sigma_{t}^{\frac{1}{2}}  {\boldsymbol{z}}_{i}$ for $i \in \{1,\cdots N\}$.
 \STATE Query the batch   observations $\{f({\boldsymbol{x}}_1 ),...,f({\boldsymbol{x}}_N ) \}$ 
 
 \STATE Compute $\widehat{\sigma}= \text{std}(f({\boldsymbol{x}}_1 ),...,f({\boldsymbol{x}}_N ))$.
 
\STATE Compute $\widehat{\mu}= \frac{1}{N} \sum_{i=1}^{N} {f({\boldsymbol{x}}_i})$. 
 
 \STATE Set \resizebox{1\columnwidth}{!}{$\Sigma_{t+1}^{-1} = \Sigma_t^{-1} + \beta\sum_{i=1}^{N}{\frac{f(\boldsymbol{x}_i)-\widehat{\mu}}{N\widehat{\sigma}}\Sigma_t^{-\frac{1}{2}} \boldsymbol{z}_i\boldsymbol{z}_i^\top \Sigma_t^{-\frac{1}{2}}   }$}.
 
 \STATE Set $ {\boldsymbol{\mu}}_{t+1} = {\boldsymbol{\mu}}_{t} - \beta \sum _{i=1}^{N} \! \frac{f(\boldsymbol{x}_i)-\widehat{\mu}}{N\widehat{\sigma}}  \Sigma^{\frac{1}{2}}_t{\boldsymbol{z}}_i   $ 
   \ENDWHILE
\end{algorithmic}
\end{algorithm}
\end{minipage}

\subsection{Stochastic Update}
The above gradient update needs the expectation of a black-box function. However, this expectation does not have a closed-form solution. Here, we 
 estimate the gradient w.r.t $\mu$ and $\Sigma$ by Monte Carlo sampling.
 Eq.(\ref{MC1}) and (\ref{MC2})  enable us to estimate the gradient  by  the function queries of $f(x)$ instead of $\nabla f(x)$. This property is very crucial for black-box optimization because gradient  ($\nabla f(x)$) is not available.

Update rules using Monte Carlo sampling are given as:
\begin{align}
\label{MonteS}
    &  \Sigma^{-1}_{t+1} \! = \Sigma^{-1}_{t} +   \frac{ \beta _t}{N}\! \sum _{i=1}^{N} \! \left[ \left( \Sigma^{-1}_t({\boldsymbol{x}}_i\!-\!{\boldsymbol{\mu}}_t)({\boldsymbol{x}}_i\!-\!{\boldsymbol{\mu}}_t)^\top \Sigma^{-1}_t \!\!-\!\! \Sigma^{-1}_t \!\right)\! f({\boldsymbol{x}}_i) \right]   \\ 
  & {\boldsymbol{\mu}}_{t+1} = {\boldsymbol{\mu}}_{t} -  \frac{ \beta _t}{N} \sum _{i=1}^{N} \! \left[ \Sigma_{t+1} \Sigma^{-1}_t({\boldsymbol{x}}_i-{\boldsymbol{\mu}}_t) f({\boldsymbol{x}}_i) \right] 
  \label{MonteU}
\end{align}
To avoid scaling problem, we employ monotonic transformation $h(f({\boldsymbol{x}}_i)=\frac{f({\boldsymbol{x}}_i)-\widehat{\mu}}{\widehat{\sigma}}$, where $\widehat{\mu}$ and $\widehat{\sigma}$ denote mean  and    stand deviation of function values in a batch of samples. This leads to  an unbiased estimator  for gradient. The update rule is given as Eq.(\ref{unSigmaUpdateM}) and Eq.(\ref{unMuupdate}). We present our black-box optimization algorithm in Alg.~\ref{alg:INGOu}.
\begin{align}
\label{unSigmaUpdateM}
& \Sigma_{t+1}^{-1} = \Sigma_t^{-1} + \beta\sum_{i=1}^{N}{\frac{f(\boldsymbol{x}_i)-\widehat{\mu}}{N\widehat{\sigma}} \left( \Sigma^{-1}_t({\boldsymbol{x}}_i\!-\!{\boldsymbol{\mu}}_t)({\boldsymbol{x}}_i\!-\!{\boldsymbol{\mu}}_t)^\top \Sigma^{-1}_t \!\right)   }  \\
   &  {\boldsymbol{\mu}}_{t+1} = {\boldsymbol{\mu}}_{t} - \beta_t \sum _{i=1}^{N} \! \frac{f({\boldsymbol{x}}_i)-\widehat{\mu}}{N \widehat{\sigma}} \Sigma_{t+1} \Sigma^{-1}_t({\boldsymbol{x}}_i-{\boldsymbol{\mu}}_t) 
     \label{unMuupdate}
\end{align}
The update of mean $\boldsymbol{\mu}$ in   Alg.~\ref{alg:INGOu} is properly scaled by $\Sigma$. Moreover, our method updates the inverse of the covariance matrix instead of the covariance matrix itself, which provides us a stable way to update covariance independent of its scale. Thus, our method can update properly when the algorithm adaptively reduces variance for high precision search.  In contrast, directly apply standard reinforce type gradient update is unstable as shown in \cite{NES}.

\subsection{Direct Update for ${\boldsymbol{\mu}}$ and $\Sigma$}
We provide an alternative updating equation with simple concept and derivation.
The implicit natural gradient algorithms are working on the natural parameter  space.  Alternatively, we can also directly work on the ${\boldsymbol{\mu}}$ and $\Sigma$ parameter space.
Formally, we 
 derive the update rule by solving the following trust region optimization problem.
 \begin{align}
     {\boldsymbol{\theta}}_{t+1} \!=\! \arg\min _{{\boldsymbol{\theta}}}
 \left<{\boldsymbol{\theta}} ,  \nabla _ {\boldsymbol{\theta}} \bar{J}({\boldsymbol{\theta}_t}) \right> \!+\! \frac{1}{\beta_t} \text{KL}\left( p_\theta \| p_{\theta_t}  \right)
 \label{MopMu}
 \end{align}
 where  ${\boldsymbol{\theta} } := \{{\boldsymbol{\mu}}, \Sigma\} $  and $\bar{J}({\boldsymbol{\theta}}):= \mathbb{E}_{p({\boldsymbol{x}};{\boldsymbol{\theta}})}[f({\boldsymbol{x}})] = J({\boldsymbol{\eta}})$.
 
 For Gaussian sampling, the optimization problem in (\ref{MopMu}) is a convex optimization problem. We can achieve a closed-form update given in Theorem~\ref{TheoremMu}:
 
\begin{theorem}\label{TheoremMu}
For Gaussian distribution with parameter ${\boldsymbol{\theta} } := \{{\boldsymbol{\mu}}, \Sigma\} $,  problem~(\ref{MopMu}) is convex w.r.t ${\boldsymbol{\theta} }$. The optimum of problem~(\ref{MopMu}) leads to closed-form update  (\ref{SigmaSigma}) and (\ref{MuMu}):
\begin{align}
\label{SigmaSigma}
    &  \Sigma^{-1}_{t+1} =  \Sigma^{-1}_{t} +  2\beta _t \nabla _ {\Sigma} \bar{J}({\boldsymbol{\theta}}_t)   \\ 
  & {\boldsymbol{\mu}}_{t+1} = {\boldsymbol{\mu}}_{t} - \beta_t \Sigma_{t}\nabla _ {\boldsymbol{\mu}} \bar{J}({\boldsymbol{\theta}}_t) \label{MuMu}
\end{align}
\end{theorem}
{\bf Remark:} Comparing the update rule in Theorem~\ref{TheoremMu} with Eq.(\ref{sigmaupdatem}) and (\ref{muupdatem}), we can observe that the only difference is in the update of  ${\boldsymbol{\mu}}$. In Eq.(\ref{MuMu}), the update employs $\Sigma_t$, while  the update in Eq.(\ref{muupdatem}) employs $\Sigma_{t+1}$. The update in Eq.(\ref{muupdatem}) takes one step look ahead  information of $\Sigma$, %it helps to improve sample efficiency.

We can obtain the black-box update for ${\boldsymbol{\mu}}$ and $\Sigma$ by  Theorem~\ref{TheoremMu} and Theorem~\ref{theB}.   The update rule is given as follows:
\begin{align}
     &  \Sigma^{-1}_{t+1} \!=\!  \Sigma^{-1}_{t} \! +\!  \beta _t \mathbb{E}_{p}\left[ \left( \Sigma^{-1}_t({\boldsymbol{x}}\!-\!{\boldsymbol{\mu}}_t)({\boldsymbol{x}}\!-\!{\boldsymbol{\mu}}_t)^\top \Sigma^{-1}_t \!\!-\!\! \Sigma^{-1}_t \right) f({\boldsymbol{x}}) \right]   \\ 
  & {\boldsymbol{\mu}}_{t+1} = {\boldsymbol{\mu}}_{t} - \beta_t \mathbb{E}_{p}\left[ ({\boldsymbol{x}}-{\boldsymbol{\mu}}_t) f({\boldsymbol{x}}) \right] 
\end{align}
Using the normalization transformation function  $h(f(x)) = (f(x)-\widehat{\mu})/\widehat{\sigma} $, we can obtain Monte Carlo approximation update as
\begin{align}
\label{SigmaSigmaUpdate}
    & \Sigma_{t+1}^{-1} = \Sigma_t^{-1} + \beta\sum_{i=1}^{N}{\frac{f(\boldsymbol{x}_i)-\widehat{\mu}}{N\widehat{\sigma}} \left( \Sigma^{-1}_t({\boldsymbol{x}}_i\!-\!{\boldsymbol{\mu}}_t)({\boldsymbol{x}}_i\!-\!{\boldsymbol{\mu}}_t)^\top \Sigma^{-1}_t \!\right)   }  \\
   &  {\boldsymbol{\mu}}_{t+1} = {\boldsymbol{\mu}}_{t} - \beta_t \sum _{i=1}^{N} \! \frac{f({\boldsymbol{x}}_i)-\widehat{\mu}}{N \widehat{\sigma}} ({\boldsymbol{x}}_i-{\boldsymbol{\mu}}_t) 
  \label{MumuUpdate}
\end{align}
We present the algorithm in alg.\ref{INGOstep} compared with our INGO. The only difference between INGO (alg.\ref{alg:INGOu}) and INGOstep (alg.\ref{INGOstep}) is the update rule of $\boldsymbol{\mu}$. INGO employs information of $\Sigma_{t+1}$, while INGOstep only uses $\Sigma_{t}$.

\begin{algorithm}[t]
  \caption{General Framework }
  \label{GF}
\begin{algorithmic}
    \STATE {\bf Input:} Number of Samples  $N$, step-size $\beta$.

  \WHILE{Termination condition not satisfied }

 \STATE Construct unbiased estimator $\widehat{g}_t$ of gradient w.r.t $\boldsymbol{\mu}$. 
 
 \STATE Construct unbiased/biased estimator $\widehat{G}_t \in \mathcal{S}^{++}$  such that $ b I \preceq \widehat{G}_t \preceq \frac{\gamma}{2} I$ 
 
 \STATE Set $\Sigma_{t+1}^{-1} = \Sigma_{t}^{-1} + 2\beta \widehat{G}_t$.
 \STATE Set  ${\boldsymbol{\mu}}_{t+1} = {\boldsymbol{\mu}}_{t} - \beta \Sigma_{t+1} \widehat{g}_t$.
  \ENDWHILE
\end{algorithmic}
\end{algorithm}

\section{Convergence Rate}

We first show a general framework for continuous optimization in Alg.~\ref{GF}. Alg.~\ref{GF} employs an unbiased estimator ($\widehat{g}_t$) for  gradient $\nabla _ {\boldsymbol{\mu}} \bar{J}({\boldsymbol{\theta}}_t)$. In contrast, it can employ both the unbiased and biased estimators $\widehat{G}_t$ for update.
It is worth  noting that $\widehat{g}_t$ can be both the first-order estimate (stochastic gradient) and the zeroth-order estimate (function value based estimator).

The update step of $\boldsymbol{\mu}$ and $\Sigma$ is achieved by solving the following convex minimization problem. 
\begin{align}
\label{EstimatorOP}
     {\boldsymbol{m}}^{t+1} = \argmin_ { {\boldsymbol{m}} \in  {\boldsymbol{\mathcal{M}} } } \beta_t \left<{\boldsymbol{m}} ,   \widehat{v}_t  \right> +   \text{KL}\left( p_{ {\boldsymbol{m}}} \| p_{ {\boldsymbol{m}}^t}  \right)
\end{align}
where ${\boldsymbol{m}}:=\{{\boldsymbol{m}}_1,{\boldsymbol{m}}_2\}=\{{\boldsymbol{\mu}} , \Sigma + {\boldsymbol{\mu}}{\boldsymbol{\mu}}^\top\} \in  {\boldsymbol{\mathcal{M}} } $,  $\boldsymbol{\mathcal{M}}$ denotes a convex set,  and $\widehat{v}_t = \{ \widehat{g}_t - 2\widehat{G}_t {\boldsymbol{\mu}}_t   ,\widehat{G}_t\}$.

 The optimum of problem~(\ref{EstimatorOP}) leads to closed-form update  (\ref{SigmaSigmaEstimator}) and~(\ref{MuMuEstimator}):
\begin{align}
\label{SigmaSigmaEstimator}
    &  \Sigma^{-1}_{t+1} =  \Sigma^{-1}_{t} +  2\beta _t  \widehat{G}_t   \\ 
  & {\boldsymbol{\mu}}_{t+1} = {\boldsymbol{\mu}}_{t} - \beta_t \Sigma_{t+1}\widehat{g}_t \label{MuMuEstimator}
\end{align}

\textbf{General Stochastic Case:} The convergence rate of Algorithm~\ref{GF} is shown in Theorem~\ref{GerneralConv}.

\begin{theorem}
\label{GerneralConv}
Given a  convex function $f({\boldsymbol{x}})$, define $\bar{J}({\boldsymbol{\theta}}):=\mathbb{E}_{p({\boldsymbol{x}};{\boldsymbol{\theta}})}[f({\boldsymbol{x}})]$ 
 for Gaussian distribution with parameter ${\boldsymbol{\theta} } := \{{\boldsymbol{\mu}} , \Sigma^{\frac{1}{2}}\} \in {\boldsymbol{\Theta} } $ and  ${\boldsymbol{\Theta} }: = \{{\boldsymbol{\mu}} , \Sigma^{\frac{1}{2}}  \big|\; {\boldsymbol{\mu}} \in \mathcal{R}^d ,  \Sigma \in \mathcal{S}^{+} \}   $. Suppose $\bar{J}({\boldsymbol{\theta}})$ be  $\gamma$-strongly convex. Let $\widehat{G}_t$ be positive semi-definite matrix such that  $  b I \preceq \widehat{G}_t \preceq \frac{\gamma}{2} I$.  Suppose $\Sigma_1 \in \mathcal{S}^{++} $ and $\|\Sigma_1 \| \le \rho$, $  \mathbb{E} \widehat{g}_t = \nabla _ {\boldsymbol{\mu} = \boldsymbol{\mu}_t}\bar{J}$. Assume furthermore  $ \| \nabla _ {\Sigma = \Sigma_t}\bar{J} \|_{tr} \le B_1$ and $ \| {\boldsymbol{\mu}}^* - {\boldsymbol{\mu}}_{1}    \|^2_{\Sigma_{1}^{-1}} \le R$,  $\mathbb{E} \| \widehat{g}_t \|_2^2 \le \mathcal{B}$ . Set $\beta_t = \beta$, then Algorithm~\ref{GF} can achieve
 \begin{equation}
   \resizebox{1\columnwidth}{!}{   
 $\frac{1}{T}\left[  \sum_{t=1}^{T} {\mathbb{E}f({\boldsymbol{\mu}}_t) }  \right] - f({\boldsymbol{\mu}}^*) \le     \frac{ 2bR \! + \! 2b\beta\rho(4B_1 \! +\! \beta\mathcal{B})  \!+ \!  4 B_1 (1\! + \!\log T) \!  +\! (1 \!+ \!\log T) \beta \mathcal{B} }{ 4\beta b T}   = \mathcal{O}\left( \frac{\log T}{T} \right)$
 }
 \end{equation}
\end{theorem}
\textbf{Remark:} Theorem~\ref{GerneralConv} does not require the function $f(\boldsymbol{x})$  be differentiable. It holds for non-smooth function $f(\boldsymbol{x})$. Theorem~\ref{GerneralConv} holds for convex function $f(\boldsymbol{x})$, as long as $\bar{J}({\boldsymbol{\theta}}):=\mathbb{E}_{p({\boldsymbol{x}};{\boldsymbol{\theta}})}[f({\boldsymbol{x}})]$ be $\gamma$-strongly convex. Particularly, when $f(\boldsymbol{x})$ is $\gamma$-strongly convex, we know $\bar{J}({\boldsymbol{\theta}})$ is $\gamma$-strongly convex~\cite{domke2019provable}. Thus, the assumption here is weaker than strongly convex assumption of $f(\boldsymbol{x})$. Moreover, Theorem~\ref{GerneralConv} does not require the boundedness of the domain. It only requires the boundedness of the distance between the initialization point and an optimal point. Theorem~\ref{GerneralConv} shows that the bound depends on the bound of $\mathbb{E} \| \widehat{g}_t \|_2^2$, which means that reducing variance of the gradient estimators can leads to a small regret bound.

\textbf{Black-box Case:}
For black-box optimization, we can only access the function value instead of the gradient. Let  $\boldsymbol{z } \sim \mathcal{N}(\boldsymbol{0},\boldsymbol{I})$,  we give an unbiased estimator of $\nabla _ {\boldsymbol{\mu}} \bar{J}({\boldsymbol{\theta}}_t)$ using function values as
\begin{align}
\label{singleg}
   \widehat{g}_t =\Sigma_t^{-\frac{1}{2}}\boldsymbol{z} \left( {f}(\boldsymbol{\mu}_t+\Sigma_t^{\frac{1}{2}}\boldsymbol{z}) -{f}(\boldsymbol{\mu}_t) \right) 
\end{align}
% where $\boldsymbol{z } \sim \mathcal{N}(\boldsymbol{0},\boldsymbol{I})$.

The estimator $ \widehat{g}_t$ is unbiased, i.e.,  $\mathbb{E}[\widehat{g}_t]= \nabla _ {\boldsymbol{\mu}} \bar{J}({\boldsymbol{\theta}}_t)$.
The proof of unbiasedness of the estimator $\widehat{g}_t$  is given in Lemma~7 in the supplement. With this estimator, we give the convergence rate of Algorithm~\ref{GF} for convex black-box optimization as in Theorem~\ref{BObound}.

\begin{theorem}
\label{BObound}
For a $L$-Lipschitz continuous convex  black box function $f(\boldsymbol{x})$,  define $\bar{J}({\boldsymbol{\theta}}):=\mathbb{E}_{p({\boldsymbol{x}};{\boldsymbol{\theta}})}[f({\boldsymbol{x}})]$
 for Gaussian distribution with parameter ${\boldsymbol{\theta} } := \{{\boldsymbol{\mu}} , \Sigma^{\frac{1}{2}}\} \in {\boldsymbol{\Theta} } $ and  ${\boldsymbol{\Theta} }: = \{{\boldsymbol{\mu}} , \Sigma^{\frac{1}{2}}  \big|\; {\boldsymbol{\mu}} \in \mathcal{R}^d ,  \Sigma \in \mathcal{S}^{+} \}   $. Suppose $\bar{J}({\boldsymbol{\theta}})$ be  $\gamma$-strongly convex. Let $\widehat{G}_t$ be positive semi-definite matrix such that  $  b  \boldsymbol{I}  \preceq \widehat{G}_t \preceq \frac{\gamma}{2} \boldsymbol{I}$.  Suppose $\Sigma_1 \in \mathcal{S}^{++} $ and $\|\Sigma_1 \|_2 \le \rho$,  Assume furthermore  $ \| \nabla _ {\Sigma = \Sigma_t}\bar{J} \|_{tr} \le B_1$ and $ \| {\boldsymbol{\mu}}^* - {\boldsymbol{\mu}}_{1}    \|^2_{\Sigma_{1}^{-1}} \le R$,   . Set $\beta_t = \beta$ and employ estimator $\widehat{g}_t$ in Eq.(\ref{singleg}), then Algorithm~\ref{GF} can achieve
 \begin{equation}
   \resizebox{1\columnwidth}{!}{
   $ \frac{1}{T}\left[  \sum_{t=1}^{T} {\mathbb{E}f({\boldsymbol{\mu}}_t) }  \right] - f({\boldsymbol{\mu}}^*)  \le  \frac{ bR + b\beta\rho(4B_1 + 2\beta L^2(d+4)^2)  }{ 2\beta b T}  + \frac{  4 B_1 (1 + \log T)   + (1+ \log T) \beta L^2(d+4)^2}{4\beta b T} = \mathcal{O} \left( \frac{d^2 \log T}{T} \right)$
  }   
 \end{equation}
\end{theorem}
\textbf{Remark:} Theorem~\ref{BObound} holds for non-differentiable function $f(\boldsymbol{x})$. Thus, Theorem~\ref{BObound}
can cover more interesting cases e.g. sparse black box optimization.  In contrast,  Balasubramanian et al.~(\cite{balasubramanian2018zeroth})  require function $f(\boldsymbol{x})$ has Lipschitz continuous gradients.  

 Alg.~\ref{alg:INGOu}
 employs an unbiased gradient estimator. When further ensure $  b  \boldsymbol{I}  \preceq \widehat{G}_t \preceq \frac{\gamma}{2} \boldsymbol{I}$, Theorem~\ref{BObound} holds for  Alg.~\ref{alg:INGOu}
Theorem~\ref{BObound} is derived for single sample per iteration.
We can reduce the variance of estimators by constructing a set of structured samples that are conjugate  of inverse covariance matrix  in a batch, i.e., $\boldsymbol{z}_i \Sigma_t^{-1}\boldsymbol{z}_j =0 , i \ne j$. Particularly, when we use $\widehat{\Sigma}_t = \sigma_t \boldsymbol{I}$,  sampling $N=d$ orthogonal samples~\cite{choromanski2018structured} per iteration can lead to a convergence rate  $\mathcal{O}\left( \frac{d \log T}{T} \right)$. For $N>d$ samples, we can use the method in~\cite{lyu2017spherical} with  a random rotation to reduce variance. For very large $N$, we can use the construction in Eq.(23) in \cite{lyu2017spherical} to transform the complex sampling matirx~\cite{xu2011deterministic}  onto  sphere $\mathbb{S}^{d-1}$, then scale samples by i.i.d  variables from Chi distribution.  This construction  has a mutual coherence bound.

\section{Empirical Study}
\label{experiments}

\begin{figure}[t]
\centering
\subfigure[\scriptsize{Ellipsoid}]{
\label{Ellipsoid}
\includegraphics[width=0.3\linewidth]{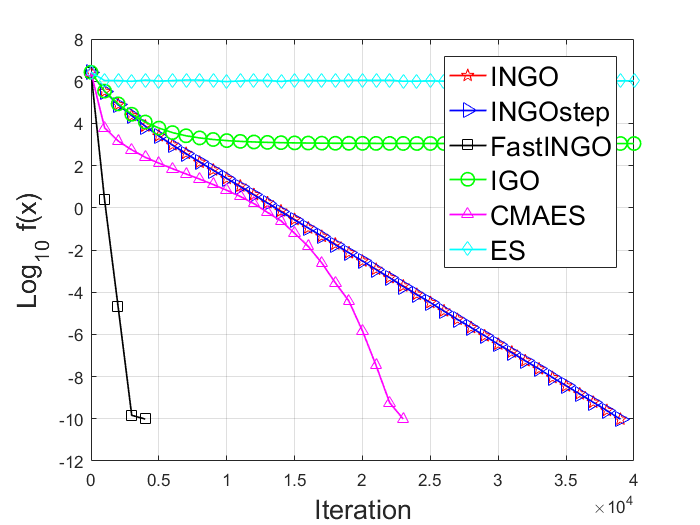}}
\subfigure[\scriptsize{$\ell_1$-Ellipsoid }]{
\label{l1Ellipsoid}
\includegraphics[width=0.3\linewidth]{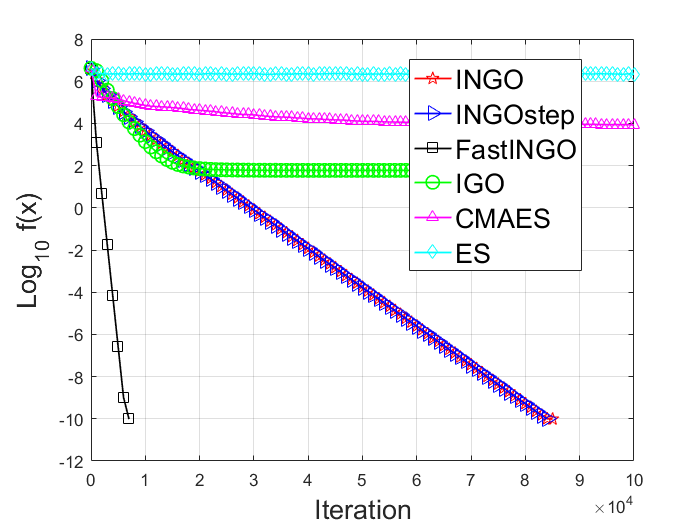}}
%\caption{ }
\subfigure[\scriptsize{$\ell_\frac{1}{2}$-Ellipsoid  }]{
\label{L12Ellipsoid}
\includegraphics[width=0.3\linewidth]{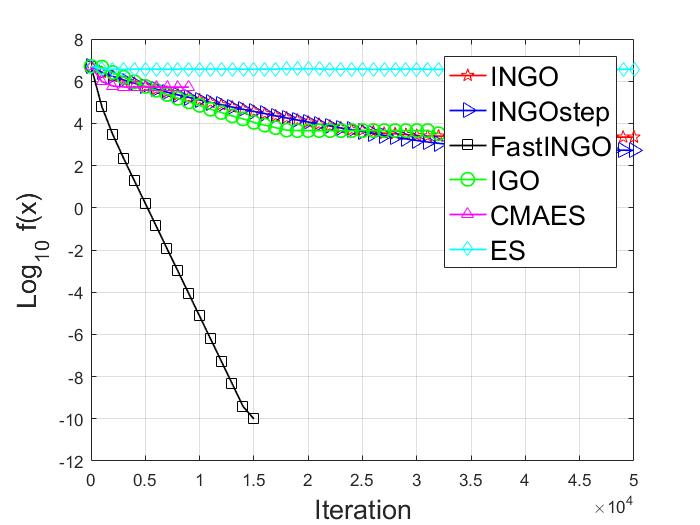}}
%\caption{ }
\subfigure[\scriptsize{ Discus}]{
\label{Discus.3}
\includegraphics[width=0.3\linewidth]{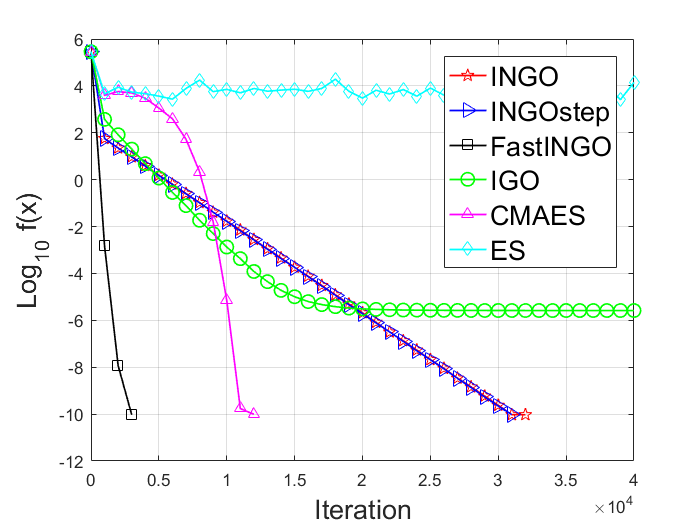}}
\subfigure[\scriptsize{Levy}]{
\label{Levy}
\includegraphics[width=0.3\linewidth]{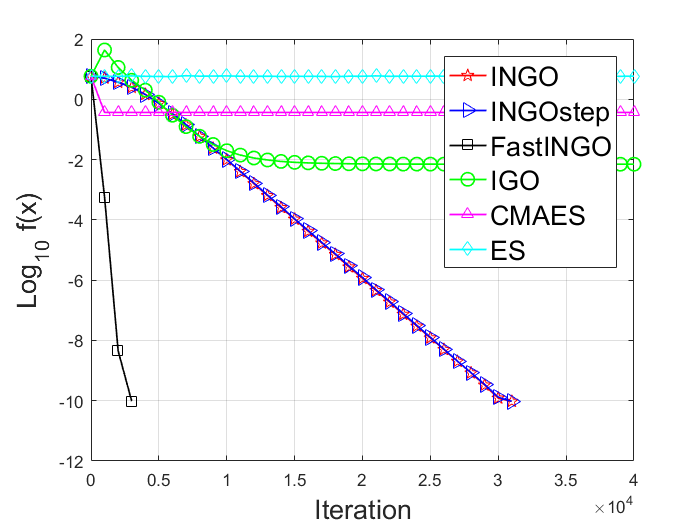}}
%\caption{ }
\subfigure[\scriptsize{ Rastrigin10 }]{
\label{Rastrigin10}
\includegraphics[width=0.3\linewidth]{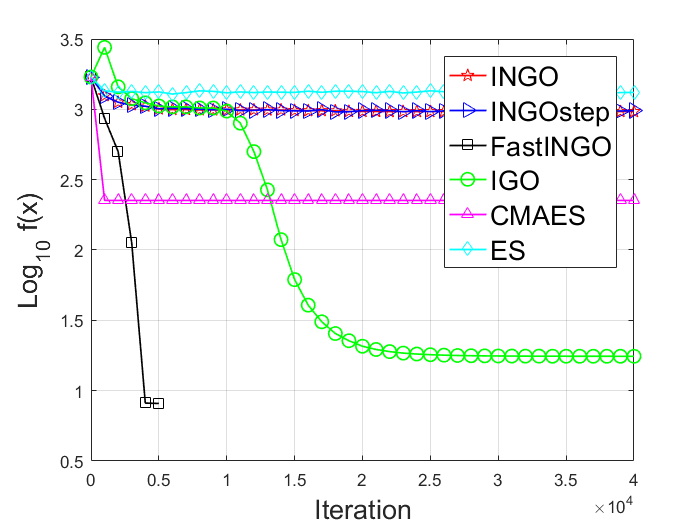}}
\caption{   Mean value of $f(\boldsymbol{x})$ in $\log_{10}$ scale over 20 independent runs for 100-dimensional problems. }
\label{TestFunc}
\end{figure}

\begin{figure*}[t]
\centering
\subfigure[\scriptsize{Swimmer}]{
\label{Swimmer}
\includegraphics[width=0.224\linewidth]{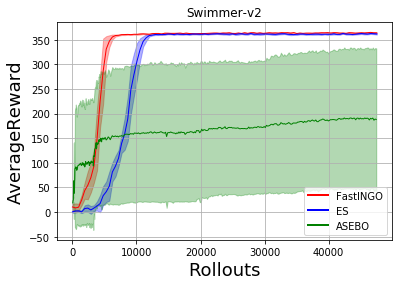}}
\subfigure[\scriptsize{HalfCheetah}]{
\label{HalfCheetah}
\includegraphics[width=0.224\linewidth]{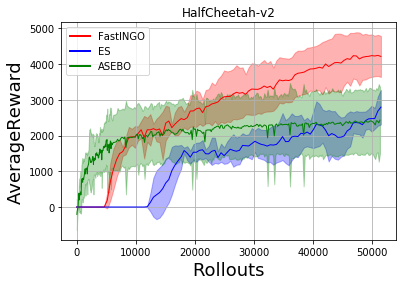}}
\subfigure[\scriptsize{InvertedDoublePendulum}]{
\label{InvertedDoublePendulum}
\includegraphics[width=0.224\linewidth]{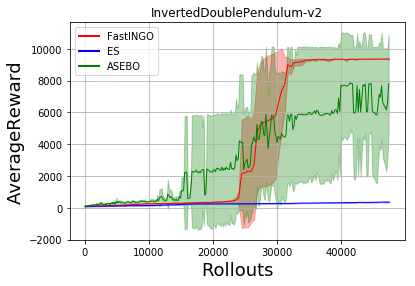}}
\subfigure[\scriptsize{HumanoidStandup}]{
\label{HumanoidStandup}
\includegraphics[width=0.224\linewidth]{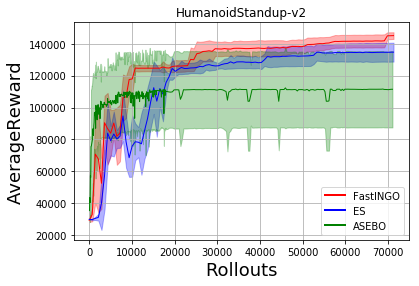}}

% \subfigure[\scriptsize{BipedalWalker}]{
% \label{BipedalWalker}
% \includegraphics[width=0.45\linewidth]{./results/BipedalWalker-v2.png}}
% \subfigure[\scriptsize{LunarLanderContinuous}]{
% \label{LunarLanderContinuous}
% \includegraphics[width=0.45\linewidth]{./results/LunarLanderContinuous-v2.png}}
\caption{  Average Reward over 5 independent  runs on  benchmark RL environments  }
\label{RLtest}
\end{figure*}

\textbf{Evaluation on synthetic continuous test benchmarks.}
We evaluate the proposed INGO , INGOstep  and Fast-INGO (diagonal case of INGO) by comparing with one of the  state-of-the-art method CMA-ES~\cite{hansen2006cma} and IGO~\cite{ollivier2017information} with full covariance matrix update, and vanilla ES with antithetic gradient estimators~\cite{salimans2017evolution}  on several synthetic benchmark test problems.  All  the  test problems are listed in Table~1 in the supplement.

{\textbf{ Parameter Settings:}} For INGO, INGOstep and IGO, we use the same normalization transformation $h(f({\boldsymbol{x}}_i)=\frac{f({\boldsymbol{x}}_i)-\widehat{\mu}}{\widehat{\sigma}}$ and all same hyper-parameters to test the effect of implicit natural gradient.  We set step size $\beta = 1/d$ for all of them. For  Fast-INGO, we set step size $\beta = 1/\sqrt{d}$, where $d$ is the dimension of the test problems. The number of samples per iteration is set to $N =2 \lfloor 3 + \lfloor 3 \times \ln{d} \rfloor /2  \rfloor$ for all the methods,
 where $ \lfloor \cdot \rfloor $ denotes the floor function. This setting ensures $N$ to be an even number.
We set ${\boldsymbol{\sigma}}_{1} = 0.5 \times {\boldsymbol{1}}$ and sample ${\boldsymbol{\mu}}_{1} \sim Uni[{\boldsymbol{0}},{\boldsymbol{1}}]$  as the same initialization for all the methods, where $Uni[0,1]$ denotes the uniform distribution in $[0,1]$. For ES~\cite{salimans2017evolution}, we use the default step-size hyper-parameters. 

The mean value of $f(\boldsymbol{x})$ over 20 independent runs for 100-dimensional problems are show in Figure~\ref{TestFunc}.  From Figure~\ref{TestFunc}, we can see that INGO, INGOstep and Fast-INGO
 converge linearly in log scale. Fast-INGO can arrive  $10^{-10}$ precision  on five cases except the highly non-convex Rastrigin10 problem. Fast-INGO employs the separate structure of the problems, thus it obtains better performance than the other methods with full matrix update. It is worth to note that Fast-INGO is not rotation invariant compared with Full-INGO.  The  INGO and INGOstep (with full matrix update) can arrive $10^{-10}$ on four cases, while IGO with  full matrix update can not achieve high precision. This shows that the update of inverse of covariance matrix is more stable.  Moreover, CMA-ES converge linearly in log scale for the convex Ellipsoid problem but slower than Fast-INGO. In addition, CMAES converge slowly on  the non-smooth   $\ell_1$-Ellipsoid and the non-convex $\ell_\frac{1}{2}$-Ellipsoid problem. Furthermore, CMAES fails on the non-convex Levy problem, while INGO, INGOstep and  Fast-INGO  obtain $10^{-10}$. CMAES converges faster or achieves smaller value than ES. On the non-convex Rastrigin10 problem, all methods fail to obtain $10^{-10}$ precision. Fast-INGO obtains smaller value.  The results on synthetic test problems show that methods employing second-order information  converge faster than first-order method ES. And employing second-order information is important  to obtain high optimization  precision, i.e., $10^{-10}$. Moreover,   taking stochastic implicit natural gradient update  can converge faster than   IGO.
The test functions are highly ill-conditioned and non-convex; the experimental results show that it is challenging for ES to optimize them well without adaptively update covariance and mean.

\textbf{Evaluation on RL test problems.} 
We further evaluate  the proposed Fast-INGO  by comparing  AESBO~\cite{choromanski2019complexity} and  ES with antithetic gradient estimators~\cite{salimans2017evolution}  on  MuJoCo  control problems: Swimmer,  HalfCheetah, HumanoidStandup, InvertedDoublePendulum,      in Open-AI Gym environments.  CMA-ES is too slow due to the computation of eigendecomposition for high-dimensional  problems. 

We use one hidden layer feed-forward neural network with tanh activation function as policy architecture.   The number of hidden unit is set to  $h=16$ for all problems. The goal is to find the parameters of this policy network to achieve large reward.   The same policy architecture is used for all the methods on all  test problems. 
The number of samples per iteration is set to $N = 20 +  4  \lfloor\lfloor 3 \times \ln{d} \rfloor /2  \rfloor$ for all the methods.  For Fast-INGO, we set step-size $\beta = 0.3$ . We set   ${\boldsymbol{\sigma}}_{1} = 0.1 \times {\boldsymbol{1}}$ and ${\boldsymbol{\mu}}_{1} = \boldsymbol{0}$ as the initialization for both Fast-INGO and ES. For ES~\cite{salimans2017evolution}, we use the default step-size hyper-parameters. Five independent runs are performed. 
The experimental results are shown in 
Figure~\ref{RLtest}. We can observe that Fast-INGO increase AverageReward faster than ES on all four cases. This shows that the update using seconder order information in Fast-INGO  can help accelerate convergence.

 \section{Conclusions}
%\vspace{-2mm}

%\section{Broader Impact}

We proposed a novel stochastic implicit natural gradient frameworks for black-box optimization. Under this framework,  we presented algorithms for both continuous and discrete black-box optimization.   Theoretically, we proved the $\mathcal{O}\left(\log T / T\right)$ convergence rate of our continuous algorithms with stochastic update for non-differentiable convex function under expectation $\gamma$-strongly convex assumption. We proved $\mathcal{O}\left(d^2\log T /T\right)$ converge rate for black-box function under same assumptions above. For isometric Gaussian case, we proved the $\mathcal{O}\left(d\log T / T\right)$ converge rate when using $d$ orthogonal samples per iteration, which well supports parallel evaluation.  Our method is very simple, and it contains less hyper-parameters than CMA-ES. Empirically,  our methods obtain  a competitive performance compared with CMA-ES. Moreover, our INGO and INGOstep with full matrix update can achieve high precision on Levy test problem and Ellipsoid problems, while IGO~\cite{ollivier2017information} with full matrix update can not.  This shows the efficiency of our methods.
On RL control problems, our algorithms increase average reward faster than ASEBO~\cite{choromanski2019complexity} and ES, which shows employing second order information can help accelerate convergence. Moreover, our discrete algorithm outperforms than GA on test functions. %We will explore other divergences  instead of KL divergence for closeness  measurement.  

%\textbf{ Impact:}

% The  impacts of our methods are  summarized as follows:

% \textbf{A simple and efficient alternative to CMA-ES:} Our INGO and INGOstep are much simpler than CMA-ES, and they have fewer inner hyperparameters compared with CMA-ES. Our methods may serve as a competitive alternative to CMA-ES for black-box optimization. Our methods may impact a wide range of downstream applications that rely on black-box optimization, including black-box optimization for reinforcement learning and black-box attack. 

% \textbf{Framework for designing new algorithms for black-box optimization:} Our methods provide a framework for designing algorithms based on implicit natural gradient w.r.t natural parameter.  By choosing different sampling distribution based on the problems' prior knowledge, we can obtain different algorithms. This may be useful for different  applications rely on black-box optimization.

% \textbf{Theoretical guidance for designing structured samples:} Our theoretical results show that structured samples can effectively reduce variance, which reduces the convergence rate.  Our  Theoretical analysis may be useful for the construction of structured samples.

 %Our INGO and INGOstep are much simpler than CMA-ES, and they have fewer inner hyperparameters compared with CMA-ES. Our methods may serve as a competitive alternative to CMA-ES for black-box optimization. This may have impact on a wide range of downstream applications rely on black-box optimization. 

\bibliography{example_paper}
\bibliographystyle{plain}

\input{appendix.tex}

\end{document}

%% file: appendix.tex
\newpage
\onecolumn

\appendix

\section{ Optimization for  Discrete Variable}
\label{Opdiscrete}

\textbf{Binary Optimization:}
For function $f(\boldsymbol{x})$ over binary variable $\boldsymbol{x} \in \{0,1\}^d$, 
we employ Bernoulli distribution with parameter $\boldsymbol{\boldsymbol{p}}= [p_1, \cdots, p_d  ]^\top$ as the underlying distribution, where $p_i$ denote the probability of  $x_i=1$. Let $\boldsymbol{\eta}$ denote the natural parameter, then we know $\boldsymbol{\boldsymbol{p}} = \frac{1}{1+e^{-\boldsymbol{\eta}}}$. The mean parameter is  ${\boldsymbol{m}} ={\boldsymbol{p}} $.

From Eq.(\ref{updateEta}), we know that
\begin{align}
    \boldsymbol{\eta}_{t+1} = \boldsymbol{\eta}_{t} - \beta_t  \nabla _ {\boldsymbol{p}} \mathbb{E}_{\boldsymbol{p}} [ f(\boldsymbol{x}) ]
\end{align}
Approximate the gradient by Monte Carlo sampling, we obtain that 
\begin{align}
     \boldsymbol{\eta}_{t+1} = \boldsymbol{\eta}_{t} - \beta_t  \frac{1}{N}\sum _{n=1}^{N} f(\boldsymbol{x}^n)\boldsymbol{h}^n
\end{align}
where $\boldsymbol{h}^n_i = \frac{1}{p_i}\boldsymbol{1}(\boldsymbol{x}^n_i=1) - \frac{1}{1-p_i}\boldsymbol{1}(\boldsymbol{x}^n_i=0) $.

In order to achieve stable update, we normalize function value by its mean $\widehat{\mu}$ and standard deviation $\widehat{\sigma}$  in a batch. The normalized update is given as follows.
\begin{align}
     \boldsymbol{\eta}_{t+1} = \boldsymbol{\eta}_{t} - \beta_t  \sum _{n=1}^{N} \frac{ f(\boldsymbol{x}^n) -\widehat{\mu}}{N\widehat{\sigma}}\boldsymbol{h}^n 
\end{align}

\textbf{General Discrete Optimization:} Similarly,  for function $f(\boldsymbol{x})$ over  discrete variable $\boldsymbol{x} \in \{1,\cdots,K\}^d$, 
we employ categorical distribution with parameter $\boldsymbol{\boldsymbol{P}}= [\boldsymbol{p}_1, \cdots, \boldsymbol{p}_d  ]^\top$ as the underlying distribution, where the ${ij}$-th element of $\boldsymbol{P}$ ($\boldsymbol{P}_{ij}$) denote the probability of  $\boldsymbol{x}_i=j$. Let $\boldsymbol{\eta} \in \mathcal{R}^{d \times K}$ denote the natural parameter, then we know  $\boldsymbol{P}_{ij} = \frac{e^{\boldsymbol{\eta}_{ij}}}{\sum_{j=1}^K{e^{\boldsymbol{\eta}_{ij}}}}$. The mean parameter is  ${\boldsymbol{m}} ={\boldsymbol{P}} $.

From Eq.(\ref{updateEta}), we know that
\begin{align}
    \boldsymbol{\eta}_{t+1} = \boldsymbol{\eta}_{t} - \beta_t  \nabla _ {\boldsymbol{P}} \mathbb{E}_{\boldsymbol{P}} [ f(\boldsymbol{x}) ]
\end{align}
Approximate the gradient by Monte Carlo sampling,
\begin{align}
     \boldsymbol{\eta}_{t+1} = \boldsymbol{\eta}_{t} - \beta_t  \frac{1}{N}\sum _{n=1}^{N} f(\boldsymbol{x}^n)\boldsymbol{H}^n
\end{align}
where $\boldsymbol{H}^n_{ij} = \frac{1}{\boldsymbol{P}_{ij}}\boldsymbol{1}(\boldsymbol{x}^n_{i}=j)$. We can also normalize the update by the mean $\widehat{\mu}$ and std $\widehat{\sigma}$. 

\textbf{Evaluation on discrete test problems}

We evaluate our discrete INGO by comparing with GA method on binary reconstruction benchmark problem, i.e., $f(\boldsymbol{x}):=\|\text{sign}(\boldsymbol{x}-0.5)-\boldsymbol{w}\|^2_2 - \|\text{sign}(\boldsymbol{w})-\boldsymbol{w}\|^2_2$ with $\boldsymbol{x} \in \{0,1\}^d$ . We construct $\boldsymbol{w}$ by sampling from standard Gaussian.  The dimension $d$ of test problem is set to $\{100, 500, 1000, 2000\}$, respectively. For our discrete INGO, we set the stepsize $\beta = 1/d$. The number of samples per iteration is same as INGO, i.e., $N = 20 +  4 \lfloor 3 + \lfloor 3 \times \ln{d} \rfloor /2  \rfloor$.

The experimental resutls are shown in Fig.~\ref{DTestFunc}. We can observe that our discrete INGO achieves much smaller regret compared with GA. Our discrete INGO converges to near zero regret on  test problems, while GA decrease very slowly after a short initial greedy phase.

\begin{figure}[t]
\centering
\subfigure[\scriptsize{100-dimensional problem}]{
\includegraphics[width=0.45\linewidth]{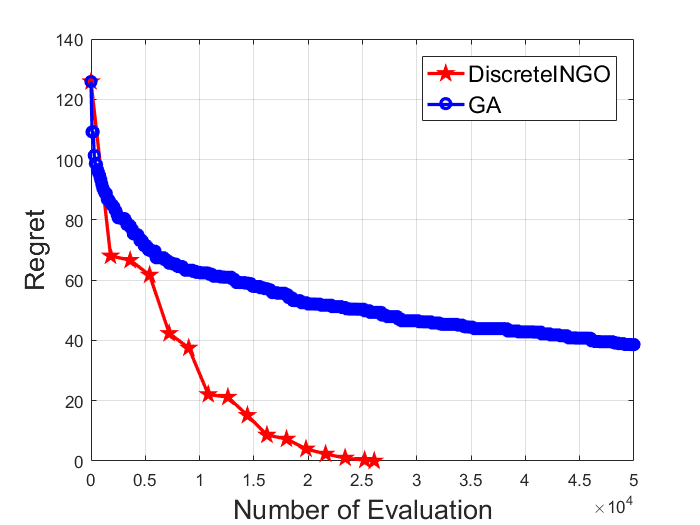}}
\subfigure[\scriptsize{500-dimensional problem }]{
\includegraphics[width=0.45\linewidth]{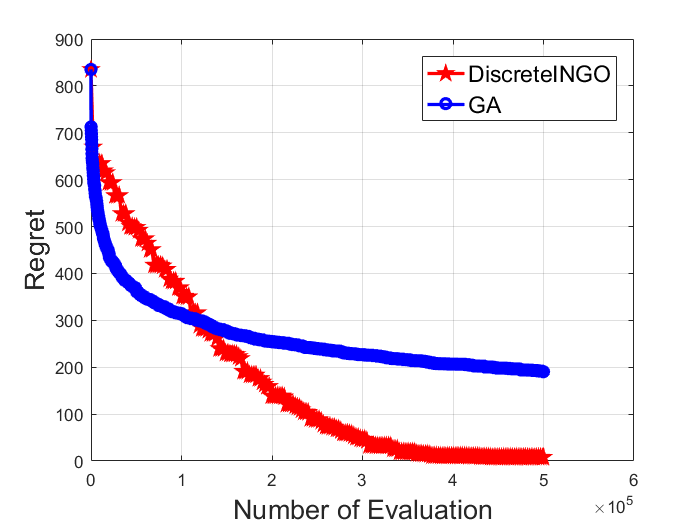}}
\subfigure[\scriptsize{1000-dimensional problem  }]{
\includegraphics[width=0.45\linewidth]{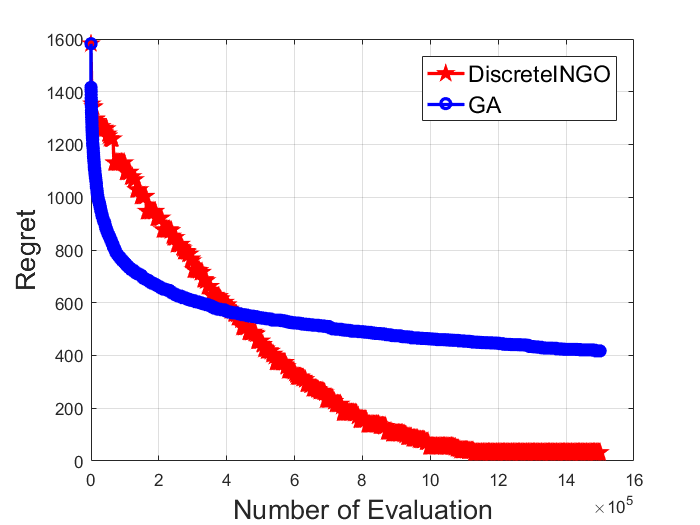}}
\subfigure[\scriptsize{ 2000-dimensional problem}]{
\includegraphics[width=0.45\linewidth]{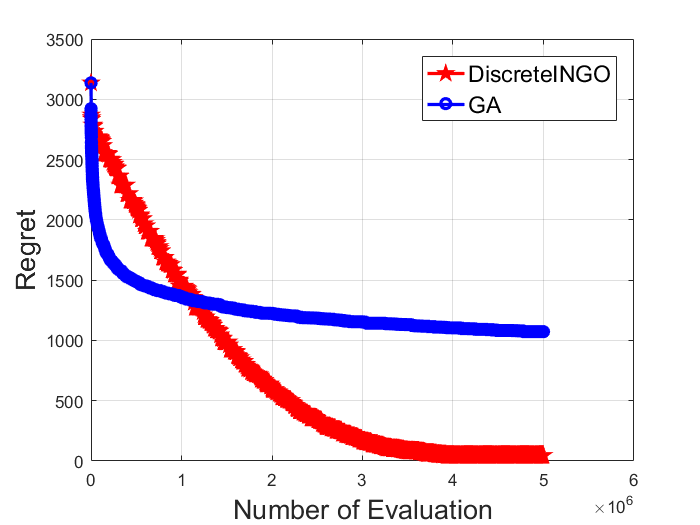}}
\caption{   Mean value of regret over 10 independent runs for different dimensional discrete optimization problems }
\label{DTestFunc}
\end{figure}

\newpage
\section{Proof of Theorem 2}
\begin{proof}
For  Gaussian distribution $p:= \mathcal{N}({\boldsymbol{\mu}}, \Sigma)$, the gradient of $ \mathbb{E}_{p}[f({\boldsymbol{x}})]$ w.r.t ${\boldsymbol{\mu}}$ can be derived as follows:
\begin{align}
    \nabla _ {\boldsymbol{\mu}} \mathbb{E}_{p}[f({\boldsymbol{x}})] & =  \mathbb{E}_{p}[f({\boldsymbol{x}})\nabla _ {\boldsymbol{\mu}} \log (p({\boldsymbol{x}};\boldsymbol{\mu},\Sigma))   ]  \\ 
    & = \mathbb{E}_{p}\left [f({\boldsymbol{x}})\nabla _ {\boldsymbol{\mu}} \left [ -\frac{1}{2}({\boldsymbol{x}}-{\boldsymbol{\mu}})^\top \Sigma^{-1}({\boldsymbol{x}}-{\boldsymbol{\mu}}) \right]  \right]   \\
    & = \mathbb{E}_{p}\left[ \Sigma^{-1}({\boldsymbol{x}}-{\boldsymbol{\mu}}) f({\boldsymbol{x}}) \right]
\end{align}
The gradient  of $ \mathbb{E}_{p}[f({\boldsymbol{x}})]$ w.r.t $\Sigma$ can be derived as follows:
\begin{align}
     \nabla _ {\Sigma} \mathbb{E}_{p}[f({\boldsymbol{x}})] & = \mathbb{E}_{p}[f({\boldsymbol{x}})\nabla _ {\Sigma} \log (p({\boldsymbol{x}};\boldsymbol{\mu},\Sigma))   ] \\
     & =  \mathbb{E}_{p}\left [f({\boldsymbol{x}})\nabla _ {\Sigma} \left [ -\frac{1}{2}({\boldsymbol{x}}-{\boldsymbol{\mu}})^\top \Sigma^{-1}({\boldsymbol{x}}-{\boldsymbol{\mu}}) -\frac{1}{2} \log \text{det}(\Sigma) \right]  \right]\\
     & =\frac{1}{2}\mathbb{E}_{p}\left[ \left( \Sigma^{-1}({\boldsymbol{x}}-{\boldsymbol{\mu}})({\boldsymbol{x}}-{\boldsymbol{\mu}})^\top \Sigma^{-1} \!\!-\!\! \Sigma^{-1} \right) f({\boldsymbol{x}}) \right]
 \label{sigma}
\end{align}
\end{proof}

\section{Proof of Theorem B}

\begin{theorem*}
\label{Gstein}
Suppose $f({\boldsymbol{x}})$ be an integrable and twice differentiable function under a Gaussian distribution $p:= \mathcal{N}({\boldsymbol{\mu}}, \Sigma)$ such that $\mathbb{E}_{p}\left[  \nabla _{\boldsymbol{x}} f({\boldsymbol{x}})\right]$ and $\mathbb{E}_{p}\left[  \frac{\partial^2f({\boldsymbol{x}})}{\partial {\boldsymbol{x}} \partial {\boldsymbol{x}}^\top }    \right]$ exists. Then, the expectation of the gradient and Hessian of $f({\boldsymbol{x}})$ can be expressed as Eq.(\ref{mu2}) and Eq.(\ref{sigma2}), respectively.
\begin{align}
\label{mu2}
 &    \mathbb{E}_{p}\left[  \nabla _{\boldsymbol{x}} f({\boldsymbol{x}})\right] = \mathbb{E}_{p}\left[ \Sigma^{-1}({\boldsymbol{x}}-{\boldsymbol{\mu}}) f({\boldsymbol{x}}) \right]  \\
 &  \mathbb{E}_{p}\left[  \frac{\partial^2f({\boldsymbol{x}})}{\partial {\boldsymbol{x}} \partial {\boldsymbol{x}}^\top }    \right]   = \mathbb{E}_{p}\left[ \left( \Sigma^{-1}({\boldsymbol{x}}-{\boldsymbol{\mu}})({\boldsymbol{x}}-{\boldsymbol{\mu}})^\top \Sigma^{-1} \!\!-\!\! \Sigma^{-1} \right) f({\boldsymbol{x}}) \right]
 \label{sigma2}
\end{align}
\end{theorem*}

\begin{proof} 
For Gaussian distribution,
from Bonnet's theorem~\cite{rezende2014stochastic}, we know that 
\begin{align}
      \nabla _ {\boldsymbol{\mu}} \mathbb{E}_{p}[f({\boldsymbol{x}})]  =   \mathbb{E}_{p}\left[  \nabla _{\boldsymbol{x}} f({\boldsymbol{x}})\right]   ].
\end{align}
From Theorem 2, we know that 
\begin{align}
     \nabla _ {\boldsymbol{\mu}} \mathbb{E}_{p}[f({\boldsymbol{x}})] =  \mathbb{E}_{p}\left[ \Sigma^{-1}({\boldsymbol{x}}-{\boldsymbol{\mu}}) f({\boldsymbol{x}}) \right].
\end{align}
Thus, we can obtain that 
\begin{align}
     \mathbb{E}_{p}\left[  \nabla _{\boldsymbol{x}} f({\boldsymbol{x}})\right] = \mathbb{E}_{p}\left[ \Sigma^{-1}({\boldsymbol{x}}-{\boldsymbol{\mu}}) f({\boldsymbol{x}}) \right].
\end{align}

From Price's Theorem~\cite{rezende2014stochastic}, we know that 
\begin{align}
 \nabla _ {\Sigma} \mathbb{E}_{p}[f({\boldsymbol{x}})]    =  \frac{1}{2} \mathbb{E}_{p}\left[  \frac{\partial^2f({\boldsymbol{x}})}{\partial {\boldsymbol{x}} \partial {\boldsymbol{x}}^\top }    \right].
\end{align}
From Theorem 2, we know that 
\begin{align}
    \nabla _ {\Sigma} \mathbb{E}_{p}[f({\boldsymbol{x}})] = \frac{1}{2}\mathbb{E}_{p}\left[ \left( \Sigma^{-1}({\boldsymbol{x}}-{\boldsymbol{\mu}})({\boldsymbol{x}}-{\boldsymbol{\mu}})^\top \Sigma^{-1} \!\!-\!\! \Sigma^{-1} \right) f({\boldsymbol{x}}) \right] 
\end{align}
It follows that
\begin{align}
    \mathbb{E}_{p}\left[  \frac{\partial^2f({\boldsymbol{x}})}{\partial {\boldsymbol{x}} \partial {\boldsymbol{x}}^\top }    \right]  = \mathbb{E}_{p}\left[ \left( \Sigma^{-1}({\boldsymbol{x}}-{\boldsymbol{\mu}})({\boldsymbol{x}}-{\boldsymbol{\mu}})^\top \Sigma^{-1} \!\!-\!\! \Sigma^{-1} \right) f({\boldsymbol{x}}) \right]
\end{align}

\end{proof}

\section{Proof of Theorem 3}
\begin{proof}
For Guassian distribution with parameter ${\boldsymbol{\theta} } := \{{\boldsymbol{\mu}}, \Sigma\} $,  problem~(\ref{MopMu}) can be rewrited as 
\begin{align}
    \left<{\boldsymbol{\theta}} ,  \nabla _ {\boldsymbol{\theta}} \bar{J}({\boldsymbol{\theta}_t}) \right> \!+\! \frac{1}{\beta_t} \text{KL}\left( p_\theta \| p_{\theta_t}  \right) 
    & = {\boldsymbol{\mu}}^\top\nabla _ {\boldsymbol{\mu}} \bar{J}({\boldsymbol{\theta}}_t) \!+ \!\text{tr}(\Sigma \nabla _ {\Sigma} \bar{J}({\boldsymbol{\theta}}_t))  \nonumber \\  & + \frac{1}{2\beta_t}\!\left[ \text{tr}(\Sigma_t^{-1}\Sigma) \!+\!  ({\boldsymbol{\mu}}\!-\!{\boldsymbol{\mu}}_t)^\top \Sigma_t^{-1}({\boldsymbol{\mu}}\!-\!{\boldsymbol{\mu}}_t) \!+\! \log\!{\frac{|\Sigma_t|}{|\Sigma|}} \!-\! d \right] 
    \label{Theom4conv}
\end{align}
where $\nabla _ {\boldsymbol{\mu}} \bar{J}({\boldsymbol{\theta}}_t)$ denotes the derivative w.r.t ${\boldsymbol{\mu}}$ taking at ${\boldsymbol{\mu}} = {\boldsymbol{\mu}}_t, \Sigma = \Sigma_t$.  $ \nabla _ {\Sigma} \bar{J}({\boldsymbol{\theta}}_t)$ denotes the derivative w.r.t $\Sigma$ taking at ${\boldsymbol{\mu}} = {\boldsymbol{\mu}}_t, \Sigma = \Sigma_t$. Note that $\nabla _ {\boldsymbol{\mu}} \bar{J}({\boldsymbol{\theta}}_t)$ and  $ \nabla _ {\Sigma} \bar{J}({\boldsymbol{\theta}}_t)$ are not functions now.

From Eq.(\ref{Theom4conv}), we can see that the problem is convex with respect to  ${\boldsymbol{\mu}}$ and $\Sigma$. Taking the derivative of (\ref{Theom4conv}) w.r.t ${\boldsymbol{\mu}}$ and $\Sigma$, and setting them to zero, we can obtain that 
\begin{align}
  &  \nabla _ {\boldsymbol{\mu}} \bar{J}({\boldsymbol{\theta}}_t) + \frac{1}{\beta_t}\Sigma_t^{-1}({\boldsymbol{\mu}}\!-\!{\boldsymbol{\mu}}_t) = {\boldsymbol{0}} \\
  &  \nabla _ {\Sigma} \bar{J}({\boldsymbol{\theta}}_t)) + \frac{1}{2\beta_t}\left[ \Sigma_t^{-1} - \Sigma^{-1} \right] =  {\boldsymbol{0}}
\end{align}
It follows that 
\begin{align}
 & {\boldsymbol{\mu}} = {\boldsymbol{\mu}}_{t} - \beta_t \Sigma_{t}\nabla _ {\boldsymbol{\mu}} \bar{J}({\boldsymbol{\theta}}_t) \\
       &  \Sigma^{-1} =  \Sigma^{-1}_{t} +  2\beta _t \nabla _ {\Sigma} \bar{J}({\boldsymbol{\theta}}_t)  
\end{align}
By definition,  ${\boldsymbol{\mu}}_{t+1} $ and  $\Sigma_{t+1}$ are the optimum of this convex optimization problem. Thus, we achieve that
\begin{align}
     & {\boldsymbol{\mu}}_{t+1} = {\boldsymbol{\mu}}_{t} - \beta_t \Sigma_{t}\nabla _ {\boldsymbol{\mu}} \bar{J}({\boldsymbol{\theta}}_t) \\
       &  \Sigma^{-1}_{t+1} =  \Sigma^{-1}_{t} +  2\beta _t \nabla _ {\Sigma} \bar{J}({\boldsymbol{\theta}}_t)  
\end{align}

\end{proof}

\section{Proof of Theorem D}

\begin{theorem*}\label{TheoremEstimator}
For Gaussian distribution with parameter ${\boldsymbol{m}}:=\{{\boldsymbol{m}}_1,{\boldsymbol{m}}_2\}=\{{\boldsymbol{\mu}} , \Sigma + {\boldsymbol{\mu}}{\boldsymbol{\mu}}^\top\} $, let $\widehat{v}_t = \{ \widehat{g}_t - 2\widehat{G}_t {\boldsymbol{\mu}}_t   ,\widehat{G}_t\}$, then the optimum of problem~(\ref{EstimatorOP1}) leads to the  closed-form update (\ref{SigmaSigmaEstimator1}) and (\ref{MuMuEstimator1}):
\begin{align}
\label{EstimatorOP1}
     {\boldsymbol{m}}^{t+1} = \argmin_ { {\boldsymbol{m}} \in  {\boldsymbol{\mathcal{M}} } } \beta_t \left<{\boldsymbol{m}} ,   \widehat{v}_t  \right> +   \text{KL}\left( p_{ {\boldsymbol{m}}} \| p_{ {\boldsymbol{m}}^t}  \right)
\end{align}
\begin{align}
\label{SigmaSigmaEstimator1}
    &  \Sigma^{-1}_{t+1} =  \Sigma^{-1}_{t} +  2\beta _t  \widehat{G}_t   \\ 
  & {\boldsymbol{\mu}}_{t+1} = {\boldsymbol{\mu}}_{t} - \beta_t \Sigma_{t+1}\widehat{g}_t \label{MuMuEstimator1}
\end{align}
\end{theorem*}

\begin{proof}
For Guassian distribution with mean parameter ${\boldsymbol{m}}:=\{{\boldsymbol{m}}_1,{\boldsymbol{m}}_2\}=\{{\boldsymbol{\mu}} , \Sigma + {\boldsymbol{\mu}}{\boldsymbol{\mu}}^\top\} $, also note that $\widehat{v}_t :=\{\widehat{g}_t-2\widehat{G}_t\boldsymbol{\mu}_t, \widehat{G}_t \}$, the problem~(\ref{EstimatorOP}) can be rewrited as 
\begin{align}
  \beta_t \left<{\boldsymbol{m}} ,   \widehat{v}_t  \right> +   \text{KL}\left( p_{ {\boldsymbol{m}}} \| p_{ {\boldsymbol{m}}^t}  \right) = \beta_t \left<{\boldsymbol{m}}_1 ,  \widehat{g}_t-2\widehat{G}_t\boldsymbol{\mu}_t  \right> + \beta_t \left< {\boldsymbol{m}}_2, \widehat{G}_t \right> + \text{KL}\left( p_{ {\boldsymbol{m}}} \| p_{ {\boldsymbol{m}}^t}  \right) 
\end{align}
Taking derivative and set to zero, also note that $\nabla _{\boldsymbol{m}}\text{KL}\left( p_{ {\boldsymbol{m}}} \| p_{ {\boldsymbol{m}}^t}  \right) ) = {\boldsymbol{\eta}} - \boldsymbol{\eta}^t$, $\boldsymbol{\eta}_1 := \Sigma^{-1}{\boldsymbol{\mu}}$ and $\boldsymbol{\eta}_2 := -\frac{1}{2}\Sigma^{-1}$,  we can obtain that

\begin{align}
    &  -\frac{1}{2}\Sigma^{-1}_{t+1} =  -\frac{1}{2}\Sigma^{-1}_{t} - \beta _t \widehat{G}_t    \\
 & \Sigma^{-1}_{t+1} {\boldsymbol{\mu}}_{t+1} = \Sigma^{-1}_{t} {\boldsymbol{\mu}}_{t} -\beta_t \left(\widehat{g}_t-2\widehat{G}_t\boldsymbol{\mu}_t  \right)
\end{align}
Rearrange terms, we can obtain that
\begin{align}
    &  \Sigma^{-1}_{t+1} =  \Sigma^{-1}_{t} + 2 \beta _t \widehat{G}_t   \\
 &   {\boldsymbol{\mu}}_{t+1} = \Sigma^{}_{t+1}\Sigma^{-1}_{t} {\boldsymbol{\mu}}_{t} -\beta_t\Sigma^{}_{t+1} \left(\widehat{g}_t-2\widehat{G}_t\boldsymbol{\mu}_t  \right) \label{Muinner}
\end{align}
Merge terms in Eq.(\ref{Muinner}), we get that
\begin{align}
     {\boldsymbol{\mu}}_{t+1} & = \Sigma^{}_{t+1}\Sigma^{-1}_{t} {\boldsymbol{\mu}}_{t} -\beta_t\Sigma^{}_{t+1} \left(\widehat{g}_t-2\widehat{G}_t\boldsymbol{\mu}_t  \right) \\ & = \Sigma^{}_{t+1} \left(\Sigma^{-1}_{t} + 2 \beta_t \widehat{G}_t  \right)\boldsymbol{\mu}_t - \beta_t\Sigma^{}_{t+1}\widehat{g}_t \\ & =\Sigma^{}_{t+1}\Sigma^{-1}_{t+1}\boldsymbol{\mu}_t - \beta_t\Sigma^{}_{t+1}\widehat{g}_t \\
     & = \boldsymbol{\mu}_t - \beta_t\Sigma^{}_{t+1}\widehat{g}_t
\end{align}

\end{proof}

\section{Proof of Theorem 4}

\begin{Lemma}\label{Lemma1} For Gaussian distribution with parameter
${\boldsymbol{\theta} } := \{{\boldsymbol{\mu}} , \Sigma\} \in {\boldsymbol{\Theta} } $. Let $F_t({\boldsymbol{m}})= \beta_t \left<{\boldsymbol{m}} ,   \widehat{v}_t  \right>   $  for all $t \ge 1$, where ${\boldsymbol{m}}:=\{{\boldsymbol{m}}_1,{\boldsymbol{m}}_2\}=\{{\boldsymbol{\mu}} , \Sigma + {\boldsymbol{\mu}}{\boldsymbol{\mu}}^\top\} \in  {\boldsymbol{\mathcal{M}} } $,  $\boldsymbol{\mathcal{M}}$ denotes a convex set.
Let ${\boldsymbol{m}}^{t+1}$ as the solution of 
\begin{align}
\label{opKL}
   {\boldsymbol{m}}^{t+1} = \argmin_ { {\boldsymbol{m}} \in  {\boldsymbol{\mathcal{M}} } }  F_t({\boldsymbol{m}}) +   \text{KL}\left( p_{ {\boldsymbol{m}}} \| p_{ {\boldsymbol{m}}^t}  \right)
\end{align}
Then, for $\forall { {\boldsymbol{m}}} \in  {\boldsymbol{\mathcal{M}} }$, we have 
\begin{align}
    F({\boldsymbol{m}}) + \text{KL}\left( p_{ {\boldsymbol{m}}} \| p_{ {\boldsymbol{m}}^t}  \right) \ge  F({\boldsymbol{m}}^{t+1}) + \text{KL}\left( p_{ {\boldsymbol{m}}^{t+1}} \| p_{ {\boldsymbol{m}}^t}  \right) + \text{KL}\left( p_{ {\boldsymbol{m}}} \| p_{ {\boldsymbol{m}}^{t+1}}  \right)
\end{align}

\end{Lemma}

\begin{proof}
Since KL-divergence of Gaussian is a Bregman divergence associated with base function $A^*({\boldsymbol{m}})$   w.r.t mean parameter ${\boldsymbol{m}}$, we know problem in Eq.(\ref{opKL}) is convex.  Since ${\boldsymbol{m}}^{t+1}$ is the optimum of the convex optimization problem in Eq.(\ref{opKL}), we have that 
\begin{align}
  &  \left<  \beta_t  \widehat{v}_t + \nabla _ {\boldsymbol{m}=\boldsymbol{m}^{t+1}}\text{KL}\left( p_{ {\boldsymbol{m}}} \| p_{ {\boldsymbol{m}}^t}  \right)  , {\boldsymbol{m}}  - {\boldsymbol{m}}^{t+1}\right>  \ge 0 , \forall{{\boldsymbol{m}}} \in  {\boldsymbol{\mathcal{M}} } 
\end{align}
Note that $\nabla _ {\boldsymbol{m}=\boldsymbol{m}^{t+1}}\text{KL}\left( p_{ {\boldsymbol{m}}} \| p_{ {\boldsymbol{m}}^t}  \right) = \nabla A^*({\boldsymbol{m}}^{t+1}) - \nabla A^*({\boldsymbol{m}}^t) $. For $\forall{{\boldsymbol{m}}} \in  {\boldsymbol{\mathcal{M}} } $  we have that
\begin{align}
     F({\boldsymbol{m}}) & = \beta_t\left<    \widehat{v}_t ,{\boldsymbol{m}}^{t+1} \right>  + \left< \beta_t \widehat{v}_t , {\boldsymbol{m}}  -{\boldsymbol{m}}^{t+1}   \right> \\
     & \ge \beta_t\left<    \widehat{v}_t ,{\boldsymbol{m}}^{t+1} \right> -  \left<  \nabla A^*({\boldsymbol{m}}^{t+1}) - \nabla A^*({\boldsymbol{m}}^t), {\boldsymbol{m}}  -{\boldsymbol{m}}^{t+1}  \right> \label{Fbreg}
\end{align}
Rewritten the term $-  \left<  \nabla A^*({\boldsymbol{m}}^{t+1}) - \nabla A^*({\boldsymbol{m}}^t), {\boldsymbol{m}}  -{\boldsymbol{m}}^{t+1}  \right> $, we have that 
% \begin{align}
%     -  \left<  \nabla A^*({\boldsymbol{m}}^{t+1}) - \nabla A^*({\boldsymbol{m}}^t), {\boldsymbol{m}}  -{\boldsymbol{m}}^{t+1}  \right>   & = A^*({\boldsymbol{m}}^{t+1}) - A^*({\boldsymbol{m}}^{t}) - \left<  \nabla A^*({\boldsymbol{m}}^{t}) , {\boldsymbol{m}}^{t+1}  -{\boldsymbol{m}}^{t}  \right> \\
%     & \;\;\; - A^*({\boldsymbol{m}}) + A^*({\boldsymbol{m}}^{t}) +  \left<  \nabla A^*({\boldsymbol{m}}^t) , {\boldsymbol{m}}  -{\boldsymbol{m}}^{t}  \right> \\ 
%      & \;\;\; + A^*({\boldsymbol{m}}) - A^*({\boldsymbol{m}}^{t+1}) -  \left<  \nabla A^*({\boldsymbol{m}}^{t+1}) , {\boldsymbol{m}}  -{\boldsymbol{m}}^{t+1}  \right> \\
%      & = \text{KL}\left( p_{ {\boldsymbol{m}}^{t+1}} \| p_{ {\boldsymbol{m}}^t}  \right)  - \text{KL}\left( p_{ {\boldsymbol{m}}} \| p_{ {\boldsymbol{m}}^t}  \right) + \text{KL}\left( p_{ {\boldsymbol{m}}} \| p_{ {\boldsymbol{m}}^{t+1}}  \right)  \label{KLthree}
% \end{align}

\begin{small}
\begin{align}
    -  \left<  \nabla A^*({\boldsymbol{m}}^{t+1}) - \nabla A^*({\boldsymbol{m}}^t), {\boldsymbol{m}}  -{\boldsymbol{m}}^{t+1}  \right>   & = A^*({\boldsymbol{m}}^{t+1}) - A^*({\boldsymbol{m}}^{t}) - \left<  \nabla A^*({\boldsymbol{m}}^{t}) , {\boldsymbol{m}}^{t+1}  -{\boldsymbol{m}}^{t}  \right> \\
    & \;\;\; - A^*({\boldsymbol{m}}) + A^*({\boldsymbol{m}}^{t}) +  \left<  \nabla A^*({\boldsymbol{m}}^t) , {\boldsymbol{m}}  -{\boldsymbol{m}}^{t}  \right> \\ 
     & \;\;\; + A^*({\boldsymbol{m}}) - A^*({\boldsymbol{m}}^{t+1}) -  \left<  \nabla A^*({\boldsymbol{m}}^{t+1}) , {\boldsymbol{m}}  -{\boldsymbol{m}}^{t+1}  \right> \\
     & = \text{KL}\left( p_{ {\boldsymbol{m}}^{t+1}} \| p_{ {\boldsymbol{m}}^t}  \right)  - \text{KL}\left( p_{ {\boldsymbol{m}}} \| p_{ {\boldsymbol{m}}^t}  \right) + \text{KL}\left( p_{ {\boldsymbol{m}}} \| p_{ {\boldsymbol{m}}^{t+1}}  \right)  \label{KLthree}
\end{align}
\end{small}

Plug Eq.(\ref{KLthree}) into (\ref{Fbreg}), we obtain that
\begin{align}
    F({\boldsymbol{m}}) + \text{KL}\left( p_{ {\boldsymbol{m}}} \| p_{ {\boldsymbol{m}}^t}  \right) \ge  F({\boldsymbol{m}}^{t+1}) + \text{KL}\left( p_{ {\boldsymbol{m}}^{t+1}} \| p_{ {\boldsymbol{m}}^t}  \right) + \text{KL}\left( p_{ {\boldsymbol{m}}} \| p_{ {\boldsymbol{m}}^{t+1}}  \right)
\end{align}

\end{proof}

\begin{Lemma}\label{lemma2}
Let $\widehat{v}_t = \{ \widehat{g}_t - 2\widehat{G}_t {\boldsymbol{\mu}}_t   ,\widehat{G}_t\}$, updating parameter as (\ref{opKL}), then we have 
% \begin{align}
%   \frac{1}{2}  \| {\boldsymbol{\mu}}^* - {\boldsymbol{\mu}}_{t+1}    \|^2_{\Sigma_{t+1}^{-1}} \le \frac{1}{2}\| {\boldsymbol{\mu}}^* - {\boldsymbol{\mu}}_{t}    \|^2_{\Sigma_{t}^{-1}} + \beta_t  \left< \widehat{g}_t, {\boldsymbol{\mu}}^* - {\boldsymbol{\mu}}_{t+1} \right> -  \frac{1}{2}\| {\boldsymbol{\mu}}_{t+1} - {\boldsymbol{\mu}}_{t}    \|^2_{\Sigma_{t+1}^{-1}} + \beta_t   \| {\boldsymbol{\mu}}^* - {\boldsymbol{\mu}}_{t}    \|^2_{\widehat{G}_t} 
% \end{align}
\begin{equation}
    \resizebox{0.95\hsize}{!}{ $ \frac{1}{2}  \| {\boldsymbol{\mu}}^* - {\boldsymbol{\mu}}_{t+1}    \|^2_{\Sigma_{t+1}^{-1}} \le \frac{1}{2}\| {\boldsymbol{\mu}}^* - {\boldsymbol{\mu}}_{t}    \|^2_{\Sigma_{t}^{-1}} + \beta_t  \left< \widehat{g}_t, {\boldsymbol{\mu}}^* - {\boldsymbol{\mu}}_{t+1} \right> -  \frac{1}{2}\| {\boldsymbol{\mu}}_{t+1} - {\boldsymbol{\mu}}_{t}    \|^2_{\Sigma_{t+1}^{-1}} + \beta_t   \| {\boldsymbol{\mu}}^* - {\boldsymbol{\mu}}_{t}    \|^2_{\widehat{G}_t} $ }
\end{equation}

\end{Lemma}
\begin{proof}
First, recall that the KL-divergence is defined as 
\begin{align}
   \text{KL}\left( p_{ {\boldsymbol{m}}} \| p_{ {\boldsymbol{m}}^t}  \right) =  \frac{1}{2} \left\{  \| {\boldsymbol{\mu}} - {\boldsymbol{\mu}}_{t}    \|^2_{\Sigma_{t}^{-1}} + \textbf{tr}\left( \Sigma\Sigma_{t}^{-1} \right) + \log \frac{| \Sigma_t  |}{| \Sigma |} -d  \right\}
\end{align}

From Lemma~\ref{Lemma1}, we know that 
\begin{align}
 \text{KL}\left( p_{ {\boldsymbol{m}}^*} \| p_{ {\boldsymbol{m}}^{t+1}}  \right) \le      \text{KL}\left( p_{ {\boldsymbol{m}}^*} \| p_{ {\boldsymbol{m}}^t}  \right) - \text{KL}\left( p_{ {\boldsymbol{m}}^{t+1}} \| p_{ {\boldsymbol{m}}^t}  \right) + F({\boldsymbol{m}}^*)  -  F({\boldsymbol{m}}^{t+1})
\end{align}
It follows that 
% \begin{align}
%   \frac{1}{2} \left\{  \| {\boldsymbol{\mu}}^* - {\boldsymbol{\mu}}_{t+1}    \|^2_{\Sigma_{t+1}^{-1}} + \textbf{tr}\left( \Sigma^*\Sigma_{t+1}^{-1} \right) + \log \frac{| \Sigma_{t+1}  |}{| \Sigma^* |}  \right\}  & \le \frac{1}{2} \left\{  \| {\boldsymbol{\mu}}^* - {\boldsymbol{\mu}}_{t}    \|^2_{\Sigma_{t}^{-1}} + \textbf{tr}\left( \Sigma^*\Sigma_{t}^{-1} \right) + \log \frac{| \Sigma_{t}  |}{| \Sigma^* |}  \right\}  \\ \nonumber
%   & - \frac{1}{2} \left\{  \| {\boldsymbol{\mu}}_{t+1} - {\boldsymbol{\mu}}_{t}    \|^2_{\Sigma_{t}^{-1}} + \textbf{tr}\left( \Sigma_{t+1}\Sigma_{t}^{-1} \right) + \log \frac{| \Sigma_{t}  |}{| \Sigma_{t+1} |}  -d \right\} \\ \nonumber
%   & + \beta_t\left<  \widehat{v}_t , {\boldsymbol{m}}^*  -{\boldsymbol{m}}^{t+1}   \right>
% \end{align}
\begin{small}
\begin{align}
   \frac{1}{2} \! \big\{  \| {\boldsymbol{\mu}}^* - {\boldsymbol{\mu}}_{t+1}    \|^2_{\Sigma_{t+1}^{-1}} \!+\! \textbf{tr}\left( \Sigma^*\Sigma_{t+1}^{-1} \right) \!+\! \log \frac{| \Sigma_{t+1}  |}{| \Sigma^* |}  \big\} \!  & \!\!\le \frac{1}{2} \left\{  \| {\boldsymbol{\mu}}^* - {\boldsymbol{\mu}}_{t}    \|^2_{\Sigma_{t}^{-1}} + \textbf{tr}\left( \Sigma^*\Sigma_{t}^{-1} \right) + \log \frac{| \Sigma_{t}  |}{| \Sigma^* |}  \right\}  \\ \nonumber
   & - \!\frac{1}{2} \!\! \big\{  \| {\boldsymbol{\mu}}_{t+1} \!-\! {\boldsymbol{\mu}}_{t}    \|^2_{\Sigma_{t}^{-1}} \!+\! \textbf{tr}\left( \Sigma_{t+1}\Sigma_{t}^{-1} \right) \!+\! \log \frac{| \Sigma_{t}  |}{| \Sigma_{t+1} |}  \!-\!d \big\} \\ \nonumber
   & + \beta_t\left<  \widehat{v}_t , {\boldsymbol{m}}^*  -{\boldsymbol{m}}^{t+1}   \right>
\end{align}
\end{small}

Then, we obtain that 
\begin{align}
\label{DistanceT}
   \frac{1}{2} \left\{  \| {\boldsymbol{\mu}}^* - {\boldsymbol{\mu}}_{t+1}    \|^2_{\Sigma_{t+1}^{-1}} + \textbf{tr}\left( \Sigma^*\Sigma_{t+1}^{-1} \right)   \right\}  & \le \frac{1}{2} \left\{  \| {\boldsymbol{\mu}}^* - {\boldsymbol{\mu}}_{t}    \|^2_{\Sigma_{t}^{-1}} + \textbf{tr}\left( \Sigma^*\Sigma_{t}^{-1} \right)   \right\}  \\ \nonumber
   & - \frac{1}{2} \left\{  \| {\boldsymbol{\mu}}_{t+1} - {\boldsymbol{\mu}}_{t}    \|^2_{\Sigma_{t}^{-1}} + \textbf{tr}\left( \Sigma_{t+1}\Sigma_{t}^{-1} \right)  -d \right\} \\ \nonumber
   & + \beta_t\left<  \widehat{v}_t , {\boldsymbol{m}}^*  -{\boldsymbol{m}}^{t+1}   \right> 
\end{align}
In addition, we have that
\begin{small}
\begin{align}
  \textbf{tr}\left( \Sigma^*\Sigma_{t}^{-1} \right)  \!-\!  \textbf{tr}\left( \Sigma^*\Sigma_{t+1}^{-1} \right) \!-\! \textbf{tr}\left( \Sigma_{t+1}\Sigma_{t}^{-1} \right)  \!+\! d & =  \textbf{tr}\left( \Sigma^*(\Sigma_{t}^{-1} - \Sigma_{t+1}^{-1}  ) \right) - \textbf{tr}\left( \Sigma_{t+1}\Sigma_{t}^{-1} - I \right) \\
  & = \textbf{tr}\left( \Sigma^*(\Sigma_{t}^{-1} \!-\! \Sigma_{t+1}^{-1}  ) \right) \!-\! \textbf{tr}\left( \Sigma_{t+1}(\Sigma_{t}^{-1} \!-\! \Sigma_{t+1}^{-1}) \right)  \\
  & = \textbf{tr}\left( (\Sigma^* - \Sigma_{t+1}  )(\Sigma_{t}^{-1} - \Sigma_{t+1}^{-1}  ) \right) \\
  & = \textbf{tr}\big( ({\boldsymbol{m}}^*_2 \!-\! {\boldsymbol{\mu}}^*{{\boldsymbol{\mu}}^*}^\top \!\!-\! {\boldsymbol{m}}^{t+1}_2 \!+\! {\boldsymbol{\mu}}_{t+1}{\boldsymbol{\mu}}_{t+1}^\top )(\Sigma_{t}^{-1} \!-\! \Sigma_{t+1}^{-1}  ) \big)
\end{align}
\end{small}
Note that $\Sigma_{t}^{-1} - \Sigma_{t+1}^{-1} = -2\beta_t \widehat{G}_t$ by updating rule, it follows that 
\begin{align}
\label{trTerm}
   \textbf{tr}\left( \Sigma^*\Sigma_{t}^{-1} \right)  \!-\!  \textbf{tr}\left( \Sigma^*\Sigma_{t+1}^{-1} \right) \!-\! \textbf{tr}\left( \Sigma_{t+1}\Sigma_{t}^{-1} \right) \! +\! d =   \!-\!2\beta_t\textbf{tr}\big( ({\boldsymbol{m}}^*_2 \!-\! {\boldsymbol{\mu}}^*{{\boldsymbol{\mu}}^*}^\top \!-\! {\boldsymbol{m}}^{t+1}_2 \!+\! {\boldsymbol{\mu}}_{t+1}{\boldsymbol{\mu}}_{t+1}^\top )\widehat{G}_t \big)
\end{align}
Then, recall that 
\begin{align}
\label{linear}
    \left<  \widehat{v}_t , {\boldsymbol{m}}^*  -{\boldsymbol{m}}^{t+1}   \right> = \left<  \widehat{g}_t -2\widehat{G}_t {\boldsymbol{\mu}}_t, {\boldsymbol{\mu}}^*  -{\boldsymbol{\mu}}_{t+1}   \right> +  \textbf{tr}\left( ({\boldsymbol{m}}^*_2  - {\boldsymbol{m}}^{t+1}_2  )\widehat{G}_t \right)
\end{align}
Plug (\ref{linear}) and (\ref{trTerm}) into (\ref{DistanceT}), we can get that 
\begin{align}
\label{inner}
    \frac{1}{2}  \| {\boldsymbol{\mu}}^* - {\boldsymbol{\mu}}_{t+1}    \|^2_{\Sigma_{t+1}^{-1}} & \le \frac{1}{2}\| {\boldsymbol{\mu}}^* - {\boldsymbol{\mu}}_{t}    \|^2_{\Sigma_{t}^{-1}}  -  \frac{1}{2}\| {\boldsymbol{\mu}}_{t+1} - {\boldsymbol{\mu}}_{t}    \|^2_{\Sigma_{t}^{-1}} + \beta_t  \left< \widehat{g}_t, {\boldsymbol{\mu}}^* - {\boldsymbol{\mu}}_{t+1} \right> \\ \nonumber 
    & - 2 \beta_t \left< \widehat{G}_t {\boldsymbol{\mu}}_t, {\boldsymbol{\mu}}^*  -{\boldsymbol{\mu}}_{t+1}   \right> + \beta_t \textbf{tr}\left( ( {\boldsymbol{\mu}}^*{{\boldsymbol{\mu}}^*}^\top  - {\boldsymbol{\mu}}_{t+1}{\boldsymbol{\mu}}_{t+1}^\top     )\widehat{G}_t \right)
\end{align}
Note that 
\begin{align}
  &  - 2 \left< \widehat{G}_t {\boldsymbol{\mu}}_t, {\boldsymbol{\mu}}^*  -{\boldsymbol{\mu}}_{t+1}   \right> +  \textbf{tr}\left( ( {\boldsymbol{\mu}}^*{{\boldsymbol{\mu}}^*}^\top  - {\boldsymbol{\mu}}_{t+1}{\boldsymbol{\mu}}_{t+1}^\top     )\widehat{G}_t \right) \\ &  = \big< \widehat{G}_t {\boldsymbol{\mu}}^*, {\boldsymbol{\mu}}^*   \big> \!-\!2\big< \widehat{G}_t {\boldsymbol{\mu}}_t, {\boldsymbol{\mu}}^*    \big> + \big< \widehat{G}_t {\boldsymbol{\mu}}_t, {\boldsymbol{\mu}}_t    \big> 
     - \big< \widehat{G}_t {\boldsymbol{\mu}}_{t+1}, {\boldsymbol{\mu}}_{t+1}   \big>  + 2\big< \widehat{G}_t {\boldsymbol{\mu}}_t, {\boldsymbol{\mu}}_{t+1}    \big>  -  \big< \widehat{G}_t {\boldsymbol{\mu}}_t, {\boldsymbol{\mu}}_t    \big>  \\ 
   & =  \| {\boldsymbol{\mu}}^* - {\boldsymbol{\mu}}_{t}    \|^2_{\widehat{G}_t} -  \| {\boldsymbol{\mu}}_{t+1} - {\boldsymbol{\mu}}_{t}    \|^2_{\widehat{G}_t}
\end{align}
Plug into (\ref{inner}), we can obtain that
\begin{align}
     \frac{1}{2}  \| {\boldsymbol{\mu}}^* - {\boldsymbol{\mu}}_{t+1}    \|^2_{\Sigma_{t+1}^{-1}} & \le \frac{1}{2}\| {\boldsymbol{\mu}}^* - {\boldsymbol{\mu}}_{t}    \|^2_{\Sigma_{t}^{-1}}  -  \frac{1}{2}\| {\boldsymbol{\mu}}_{t+1} - {\boldsymbol{\mu}}_{t}    \|^2_{\Sigma_{t}^{-1}} + \beta_t  \left< \widehat{g}_t, {\boldsymbol{\mu}}^* - {\boldsymbol{\mu}}_{t+1} \right> \\ \nonumber 
    & + \beta_t  \| {\boldsymbol{\mu}}^* - {\boldsymbol{\mu}}_{t}    \|^2_{\widehat{G}_t} -  \beta_t \| {\boldsymbol{\mu}}_{t+1} - {\boldsymbol{\mu}}_{t}    \|^2_{\widehat{G}_t} 
\end{align}
Also note that $  \frac{1}{2}\| {\boldsymbol{\mu}}_{t+1} - {\boldsymbol{\mu}}_{t}    \|^2_{\Sigma_{t}^{-1}}  +  \beta_t \| {\boldsymbol{\mu}}_{t+1} - {\boldsymbol{\mu}}_{t}    \|^2_{\widehat{G}_t} = \frac{1}{2}\| {\boldsymbol{\mu}}_{t+1} - {\boldsymbol{\mu}}_{t}    \|^2_{\Sigma_{t+1}^{-1}}$, we obtain that
\begin{align}
   \frac{1}{2}  \| {\boldsymbol{\mu}}^* - {\boldsymbol{\mu}}_{t+1}    \|^2_{\Sigma_{t+1}^{-1}} & \le \frac{1}{2}\| {\boldsymbol{\mu}}^* \!-\! {\boldsymbol{\mu}}_{t}    \|^2_{\Sigma_{t}^{-1}}  \!-\!  \frac{1}{2}\| {\boldsymbol{\mu}}_{t+1} \!-\! {\boldsymbol{\mu}}_{t}    \|^2_{\Sigma_{t+1}^{-1}} + \beta_t  \left< \widehat{g}_t, {\boldsymbol{\mu}}^* \!-\! {\boldsymbol{\mu}}_{t+1} \right> 
     \!+\! \beta_t  \| {\boldsymbol{\mu}}^* - {\boldsymbol{\mu}}_{t}    \|^2_{\widehat{G}_t}   
\end{align}

\end{proof}

\begin{Lemma}
Given a convex function $f(x)$,  for Gaussian distribution with parameters  ${\boldsymbol{\theta} } := \{{\boldsymbol{\mu}}, \Sigma^{\frac{1}{2}}\} $,   let $\bar{J}({\boldsymbol{\theta}}):= \mathbb{E}_{p({\boldsymbol{x}};{\boldsymbol{\theta}})}[f({\boldsymbol{x}})]$. Then $\bar{J}({\boldsymbol{\theta}})$ is a convex function with respect to ${\boldsymbol{\theta}}$.
\end{Lemma}
\begin{proof}
For $\lambda \in [0,1]$, we have 
\begin{align}
    \lambda \bar{J}({\boldsymbol{\theta}}_1) + (1-\lambda)\bar{J}({\boldsymbol{\theta}}_2) & = \lambda \mathbb{E}[f({\boldsymbol{\mu}_1} + \Sigma_1^{\frac{1}{2}}{\boldsymbol{z}}) ]   + (1-\lambda) \mathbb{E}[f({\boldsymbol{\mu}_2} + \Sigma_2^{\frac{1}{2}}{\boldsymbol{z}}) ] \\
    & = \mathbb{E}[ \lambda f({\boldsymbol{\mu}_1} + \Sigma_1^{\frac{1}{2}}{\boldsymbol{z}}) + (1-\lambda) f({\boldsymbol{\mu}_2} +  \Sigma_2^{\frac{1}{2}}{\boldsymbol{z}}) ]  \\
    & \ge \mathbb{E}[  f\left(\lambda{\boldsymbol{\mu}_1} + (1-\lambda)\lambda{\boldsymbol{\mu}_2} + (\lambda\Sigma_1^{\frac{1}{2}}+ (1-\lambda)\Sigma_2^{\frac{1}{2}}){\boldsymbol{z}} \right) ] \\
    & = \bar{J}(\lambda{\boldsymbol{\theta}}_1 + (1-\lambda){\boldsymbol{\theta}}_2)
\end{align}
\end{proof}

\begin{Lemma}
 Let  $\bar{J}({\boldsymbol{\theta}}):=\mathbb{E}_{p({\boldsymbol{x}};{\boldsymbol{\theta}})}[f({\boldsymbol{x}})]$ 
 for Guassian distribution with parameter ${\boldsymbol{\theta} } := \{{\boldsymbol{\mu}} , \Sigma^{\frac{1}{2}}\} \in {\boldsymbol{\Theta} } $ and  ${\boldsymbol{\Theta} }: = \{{\boldsymbol{\mu}} , \Sigma^{\frac{1}{2}}   \big|\; {\boldsymbol{\mu}} \in \mathcal{R}^d ,   \Sigma \in \mathcal{S}^{+} \}   $ be a $\gamma$-strongly convex function. Suppose $  b I \preceq \widehat{G}_t \preceq \frac{\gamma}{2} I$ be positive definite matrix and $\Sigma_1 \in {\boldsymbol{\Theta} }$, then we have
\begin{align}
     \frac{1}{2} \mathbb{E} \| {\boldsymbol{\mu}}^* \!-\! {\boldsymbol{\mu}}_{t+1}    \|^2_{\Sigma_{t+1}^{-1}} \! \le \! \frac{1}{2} \mathbb{E}\| {\boldsymbol{\mu}}^* \!-\! {\boldsymbol{\mu}}_{t}    \|^2_{\Sigma_{t}^{-1}}  \!+\! \beta_t  \mathbb{E} (\bar{J}({\boldsymbol{\theta}}^*) \!-\!\bar{J}({\boldsymbol{\theta}}_t)) \!+\! \beta_t \mathbb{E} \left< G_t,  2{\boldsymbol{\Sigma}}_{t} \right> \! +\! \frac{\beta^2_t}{2} \mathbb{E} \| \Sigma_{t+1} \|_2 \| \widehat{g}_t \|_2^2  
\end{align}
\end{Lemma}

\begin{proof}

From Lemma~\ref{lemma2}, we know that 
\begin{align}
     \frac{1}{2}  \| {\boldsymbol{\mu}}^* - {\boldsymbol{\mu}}_{t+1}    \|^2_{\Sigma_{t+1}^{-1}} & \le \frac{1}{2}\| {\boldsymbol{\mu}}^* - {\boldsymbol{\mu}}_{t}    \|^2_{\Sigma_{t}^{-1}}  -  \frac{1}{2}\| {\boldsymbol{\mu}}_{t+1} - {\boldsymbol{\mu}}_{t}    \|^2_{\Sigma_{t+1}^{-1}} \! +\! \beta_t  \left< \widehat{g}_t, {\boldsymbol{\mu}}^* - {\boldsymbol{\mu}}_{t+1} \right> 
     + \beta_t  \| {\boldsymbol{\mu}}^* - {\boldsymbol{\mu}}_{t}    \|^2_{\widehat{G}_t} 
\end{align}
It follows that
% \begin{align}
%   \frac{1}{2}  \| {\boldsymbol{\mu}}^* - {\boldsymbol{\mu}}_{t+1}    \|^2_{\Sigma_{t+1}^{-1}} & \le \frac{1}{2}\| {\boldsymbol{\mu}}^* - {\boldsymbol{\mu}}_{t}    \|^2_{\Sigma_{t}^{-1}}  -  \frac{1}{2}\| {\boldsymbol{\mu}}_{t+1} - {\boldsymbol{\mu}}_{t}    \|^2_{\Sigma_{t+1}^{-1}} + \beta_t  \left< \widehat{g}_t, {\boldsymbol{\mu}}^* - {\boldsymbol{\mu}}_{t} \right>   + \beta_t  \left< \widehat{g}_t, {\boldsymbol{\mu}}_{t} - {\boldsymbol{\mu}}_{t+1} \right>  
%      + \beta_t  \| {\boldsymbol{\mu}}^* - {\boldsymbol{\mu}}_{t}    \|^2_{\widehat{G}_t}  
%      \label{beforeSig}
% \end{align}
\begin{equation}
    \resizebox{1.05\hsize}{!}{$\frac{1}{2}  \| {\boldsymbol{\mu}}^* \!-\! {\boldsymbol{\mu}}_{t+1}    \|^2_{\Sigma_{t+1}^{-1}}  \le \frac{1}{2}\| {\boldsymbol{\mu}}^* \!-\! {\boldsymbol{\mu}}_{t}    \|^2_{\Sigma_{t}^{-1}} \! -\!  \frac{1}{2}\| {\boldsymbol{\mu}}_{t+1} \!-\! {\boldsymbol{\mu}}_{t}    \|^2_{\Sigma_{t+1}^{-1}} \!+\! \beta_t  \left< \widehat{g}_t, {\boldsymbol{\mu}}^* \!-\! {\boldsymbol{\mu}}_{t} \right>   \!+\! \beta_t  \left< \widehat{g}_t, {\boldsymbol{\mu}}_{t} - {\boldsymbol{\mu}}_{t+1} \right>  
    \! +\! \beta_t  \| {\boldsymbol{\mu}}^* \!-\! {\boldsymbol{\mu}}_{t}    \|^2_{\widehat{G}_t}$ 
    } \label{beforeSig}
\end{equation}

Note that
\begin{align}
    -  \frac{1}{2}\| {\boldsymbol{\mu}}_{t+1} - {\boldsymbol{\mu}}_{t}    \|^2_{\Sigma_{t+1}^{-1}} + \beta_t  \left< \widehat{g}_t, {\boldsymbol{\mu}}_{t} - {\boldsymbol{\mu}}_{t+1} \right> & \!= \!-\!  \frac{1}{2}\| {\boldsymbol{\mu}}_{t+1} - {\boldsymbol{\mu}}_{t}    \|^2_{\Sigma_{t+1}^{-1}} \!+\! \beta_t  \left< \Sigma_{t+1} \widehat{g}_t, \Sigma_{t+1}^{-1}({\boldsymbol{\mu}}_{t} - {\boldsymbol{\mu}}_{t+1}) \right> \\
    & \!=\!  -\!  \frac{1}{2}\| {\boldsymbol{\mu}}_{t+1} - {\boldsymbol{\mu}}_{t}  + \beta_t \Sigma_{t+1} \widehat{g}_t  \|^2_{\Sigma_{t+1}^{-1}}  + \frac{\beta^2_t}{2} \| \Sigma_{t+1}\widehat{g}_t \|^2_{\Sigma_{t+1}^{-1}} \\
    & \le  \frac{\beta^2_t}{2} \| \Sigma_{t+1}\widehat{g}_t \|^2_{\Sigma_{t+1}^{-1}} \le \frac{\beta^2_t}{2} \| \Sigma_{t+1} \|_2  \| \widehat{g}_t \|_2^2 
    \label{BlackInq}
\end{align}
 
Note that $\Sigma^{-1}_{t+1} = \Sigma^{-1}_{t} + 2 \beta_t \widehat{G}_t$ and $\widehat{G}_t \succeq  b I $, we have smallest eigenvalues $ \lambda_{min} (\Sigma^{-1}_{t+1})  \ge   \lambda_{min} (\Sigma^{-1}_{t}) \ge \cdots \ge \lambda_{min} (\Sigma^{-1}_{1})$. Then, we know  $ \| \Sigma_{t+1} \|_2  \le   \| \Sigma_{1} \|_2 $. In addition, $\Sigma_{t+1}$ is positive definite matrix, thus  $\Sigma_{t+1}  \in \boldsymbol{\Theta} $ for $t \in \{1,2,3 \cdots\}$.

Plug (\ref{BlackInq}) into (\ref{beforeSig}),
 we can achieve that 
\begin{align}
   \frac{1}{2}  \| {\boldsymbol{\mu}}^* - {\boldsymbol{\mu}}_{t+1}    \|^2_{\Sigma_{t+1}^{-1}} & \le \frac{1}{2}\| {\boldsymbol{\mu}}^* - {\boldsymbol{\mu}}_{t}    \|^2_{\Sigma_{t}^{-1}}  + \beta_t  \left< \widehat{g}_t, {\boldsymbol{\mu}}^* - {\boldsymbol{\mu}}_{t} \right>   
     + \beta_t  \| {\boldsymbol{\mu}}^* - {\boldsymbol{\mu}}_{t}    \|^2_{\widehat{G}_t}  +  \frac{\beta^2_t}{2} \| \Sigma_{t+1} \|_2  \| \widehat{g}_t \|_2^2  
\end{align}
Since $  b I \preceq \widehat{G}_t \preceq \frac{\gamma}{2} I$, we get that
\begin{align}
    \frac{1}{2}  \| {\boldsymbol{\mu}}^* - {\boldsymbol{\mu}}_{t+1}    \|^2_{\Sigma_{t+1}^{-1}} & \le \frac{1}{2}\| {\boldsymbol{\mu}}^* - {\boldsymbol{\mu}}_{t}    \|^2_{\Sigma_{t}^{-1}}  + \beta_t  \left< \widehat{g}_t, {\boldsymbol{\mu}}^* - {\boldsymbol{\mu}}_{t} \right>   
     + \beta_t \frac{\gamma}{2}  \| {\boldsymbol{\mu}}^* - {\boldsymbol{\mu}}_{t}    \|^2_2  + \frac{\beta^2_t}{2} \| \Sigma_{t+1} \|_2 \| \widehat{g}_t \|_2^2 
\end{align}

Taking conditional expectation on both sides, we obtain that
\begin{align}
     \frac{1}{2} \mathbb{E} \| {\boldsymbol{\mu}}^* - {\boldsymbol{\mu}}_{t+1}    \|^2_{\Sigma_{t+1}^{-1}}
     &  \! \le \! \frac{1}{2} \mathbb{E}\| {\boldsymbol{\mu}}^* \!-\! {\boldsymbol{\mu}}_{t}    \|^2_{\Sigma_{t}^{-1}} \! + \! \beta_t  \left< \mathbb{E} \widehat{g}_t, {\boldsymbol{\mu}}^* \!-\! {\boldsymbol{\mu}}_{t} \right>    \! +\! \beta_t \frac{\gamma}{2} \| {\boldsymbol{\mu}}^* \!-\! {\boldsymbol{\mu}}_{t}    \|^2_2  \!+\! \frac{\beta^2_t}{2}  \mathbb{E} \| \Sigma_{t+1} \|_2\| \widehat{g}_t \|_2^2 \\  \nonumber 
     & \le \frac{1}{2} \mathbb{E}\| {\boldsymbol{\mu}}^* - {\boldsymbol{\mu}}_{t}    \|^2_{\Sigma_{t}^{-1}}  + \beta_t  \left< \mathbb{E} \widehat{g}_t, {\boldsymbol{\mu}}^* - {\boldsymbol{\mu}}_{t} \right>   - \beta_t  \left< G_t,  2{\boldsymbol{\Sigma}}_{t} \right> + \beta_t  \left< G_t,  2{\boldsymbol{\Sigma}}_{t} \right>  \\
    &   + \beta_t \frac{\gamma}{2} \| {\boldsymbol{\mu}}^* - {\boldsymbol{\mu}}_{t}    \|^2_2  + \frac{\beta^2_t}{2}  \mathbb{E} \| \Sigma_{t+1} \|_2 \| \widehat{g}_t \|_2^2
   \label{Eineq}
\end{align}
Note that $ g_t = \mathbb{E} \widehat{g}_t = \nabla _ {\boldsymbol{\mu} = \boldsymbol{\mu}_t}\bar{J}$ and $ G_t    = \nabla _ {\Sigma = \Sigma_t}\bar{J}$ and $\nabla _ {\Sigma^{\frac{1}{2}}} \bar{J} = \Sigma^{\frac{1}{2}} \nabla _ {\Sigma} {\bar{J}} + \nabla _ {\Sigma} {\bar{J}} \Sigma^{\frac{1}{2}} $, where $G_t$,  $\nabla _ {\Sigma} {\bar{J}}$ and ${\Sigma^{\frac{1}{2}}} $ are symmetric matrix. Since $\bar{J}({\boldsymbol{\theta}})$ is a $\gamma$-strongly  convex function with optimum at ${\boldsymbol{\theta}}^* = \{ \boldsymbol{\mu}^*, \boldsymbol{0} \}$,  we have that
\begin{align}
 \left<\nabla _ {\boldsymbol{\mu} = \boldsymbol{\mu}_t}\bar{J} ,  {\boldsymbol{\mu}}^* - {\boldsymbol{\mu}}_{t}  \right> + \left< \nabla _ {\Sigma^{\frac{1}{2}}  = \Sigma_t^{\frac{1}{2}} } \bar{J} , \boldsymbol{0} - \Sigma_t^{\frac{1}{2}}  \right> & =
   \left<g_t ,  {\boldsymbol{\mu}}^* - {\boldsymbol{\mu}}_{t}  \right> + \left< \Sigma_t^{\frac{1}{2}}G_t + G_t\Sigma_t^{\frac{1}{2}} , \boldsymbol{0} - \Sigma_t^{\frac{1}{2}}  \right> \\
   & =  \left<g_t ,  {\boldsymbol{\mu}}^* - {\boldsymbol{\mu}}_{t}  \right> - \left< G_t,  2\Sigma_t  \right> \\
   & \le (\bar{J}({\boldsymbol{\theta}}^*)-\bar{J}({\boldsymbol{\theta}}_t)) - \frac{\gamma}{2} \| {\boldsymbol{\mu}}^* - {\boldsymbol{\mu}}_{t}    \|^2_2
\end{align}
Plug it into (\ref{Eineq}), we can obtain that 
\begin{align}
     \frac{1}{2} \mathbb{E} \| {\boldsymbol{\mu}}^* \!-\! {\boldsymbol{\mu}}_{t+1}    \|^2_{\Sigma_{t+1}^{-1}}  \!\le \! \frac{1}{2} \mathbb{E}\| {\boldsymbol{\mu}}^* \! - \! {\boldsymbol{\mu}}_{t}    \|^2_{\Sigma_{t}^{-1}} \! +\! \beta_t   (\bar{J}({\boldsymbol{\theta}}^*) \!-\! \bar{J}({\boldsymbol{\theta}}_t)) \!+\! \beta_t  \left< G_t,  2{\boldsymbol{\Sigma}}_{t} \right> 
    \! +\! \frac{\beta^2_t}{2}  \mathbb{E} \| \Sigma_{t+1} \|_2 \| \widehat{g}_t \|_2^2   
\end{align}
Taking expectation on both sides, we know that
\begin{align}
    \frac{1}{2} \mathbb{E} \| {\boldsymbol{\mu}}^* \!-\! {\boldsymbol{\mu}}_{t+1}    \|^2_{\Sigma_{t+1}^{-1}} \!\le\! \frac{1}{2} \mathbb{E}\| {\boldsymbol{\mu}}^* \!-\! {\boldsymbol{\mu}}_{t}    \|^2_{\Sigma_{t}^{-1}}  \!+\! \beta_t  \mathbb{E} (\bar{J}({\boldsymbol{\theta}}^*) \!- \!\bar{J}({\boldsymbol{\theta}}_t)) \!+\! \beta_t \mathbb{E} \left< G_t,  2{\boldsymbol{\Sigma}}_{t} \right> \!+\! \frac{\beta^2_t}{2} \mathbb{E} \| \Sigma_{t+1} \|_2 \| \widehat{g}_t \|_2^2  
\end{align}

\end{proof}

\begin{Lemma}
\label{normEq}
Given a symmetric matrix $X$ and a symmetric positive semi-definite matrix $Y$, then we have $\textbf{tr}\left(XY\right) \le \| Y\|_2 \| X \|_{tr}$, where $\|X \|_{tr} := \sum_{i=1}^{d} |\lambda_i| $ with $\lambda_i$ denotes the eigenvalues.
\end{Lemma}
\begin{proof}
Since $X$ is symmetric, it can be orthogonal  diagonalized as $X = U \Lambda U^\top$, where $\Lambda$ is a diagonal matrix contains eigenvalues $\lambda_i, i \in \{1,\cdots,d\}$. Since $Y$ is a symmetric positive semi-definite matrix, it can be written as $Y = Y^{\frac{1}{2}}Y^{\frac{1}{2}}$.
It follows that 
\begin{align}
    \textbf{tr}\left(XY\right) & = \textbf{tr}\left( U \Lambda U^\top Y^{\frac{1}{2}}Y^{\frac{1}{2}}\right) = \textbf{tr}\left(Y^{\frac{1}{2}} U \Lambda U^\top Y^{\frac{1}{2}}\right) = \sum_{i=1}^{d} \lambda_i \boldsymbol{a}_i^\top \boldsymbol{a}_i
\end{align}
where $\boldsymbol{a}_i$ denotes the $i^{th}$ column of the matrix $A= U^\top Y^{\frac{1}{2}}$. Then, we have
\begin{align}
   \textbf{tr}\left(XY\right) \le \sum_{i=1}^{d} |\lambda_i| \boldsymbol{a}_i^\top \boldsymbol{a}_i   = \textbf{tr}\left(Y^{\frac{1}{2}} U | \Lambda | U^\top Y^{\frac{1}{2}} \right) = \textbf{tr}\left(Y^{\frac{1}{2}} \bar{X}^{\frac{1}{2}}\bar{X}^{\frac{1}{2}} Y^{\frac{1}{2}} \right) = \| Y^{\frac{1}{2}} \bar{X}^{\frac{1}{2}} \|_F^2
\end{align}
where $\bar{X}^{\frac{1}{2}}= U | \Lambda |^{\frac{1}{2}} U^\top$.

Using the fact $\|Y^{\frac{1}{2}} \bar{X}^{\frac{1}{2}}  \|_F^2 \le \| Y^{\frac{1}{2}}  \|_2^2 \|\bar{X}^{\frac{1}{2}}  \|_F^2$, we can obtain that
\begin{align}
    \textbf{tr}\left(XY\right) = \| Y^{\frac{1}{2}} \bar{X}^{\frac{1}{2}} \|_F^2 \le  \| Y^{\frac{1}{2}}  \|_2^2 \|\bar{X}^{\frac{1}{2}}  \|_F^2 = \| Y  \|_2 \| \bar{X} \|_{tr}  = \| Y  \|_2 \| X \|_{tr}
\end{align}

\end{proof}

\begin{Lemma}
\label{innerBound}
Suppose gradients $ \| G_t \|_{tr} \le B_1$ and $    \widehat{G}_t \succeq b I$ with $b>0$,  by setting $\beta_t = \beta$ as a constant step size, we have 
\begin{align}
     \sum_{t=1}^{T} {\beta_t \mathbb{E} \left< G_t,  2{\boldsymbol{\Sigma}}_{t} \right> } \le   2 B_1 \left ( \beta \|\Sigma_1 \|_2 +    \frac{1 + \log T}{2b} \right)
\end{align}

\end{Lemma}
\begin{proof}
Note that $\Sigma_{t+1}^{-1} - \Sigma_{t}^{-1} = 2\beta_t \widehat{G}_t$ and $\widehat{G}_t \succeq b I$ with $b >0$, we know the smallest eigenvalue of $\Sigma^{-1}_{t+1}$, i.e. $\lambda_{min}(\Sigma^{-1}_{t+1})$ satisfies that 
\begin{align}
  \lambda_{min}(\Sigma^{-1}_{t+1}) \ge \lambda_{min}(\Sigma^{-1}_{t}) + 2 \beta_t b \ge  \lambda_{min}(\Sigma^{-1}_{1}) + 2\sum_{i=1}^{t}{\beta_i}b \ge 2\sum_{i=1}^{t}{\beta_i}b 
\end{align}
Thus, we know that
\begin{align}
    \|  \Sigma_{t+1} \|_2 = \frac{1}{\lambda_{min}( \Sigma_{t+1}^{-1})} \le \frac{1}{2\sum_{i=1}^{t}{\beta_i}b } = \frac{1}{2t\beta b} 
    \label{SigmaRate}
\end{align}

Note that  $\Sigma_t$ is symmetric positive semi-definite  and  $G_t$ is symmetric. From Lemma~\ref{normEq}, we know that
 $\textbf{tr}(G_t\Sigma_t) \le \|\Sigma_t \|_2 \| G_t\|_{tr}$. 
It follows that
\begin{align}
    \sum_{t=1}^{T} { \beta_t  \mathbb{E} \left< G_t,  2{\boldsymbol{\Sigma}}_{t} \right> }   & \le 2 \beta  \sum_{t=1}^{T} {\mathbb{E}[\|G_t\|_{tr} \|\Sigma_t\|_2 ]} \le 2 \beta  B_1  \sum_{t=1}^{T}{\mathbb{E}\|\Sigma_t\|_2} \\
    & \le  2 \beta B_1 \|\Sigma_1 \|_2    + 2 B_1 (\sum_{t=1}^{T-1}{\frac{1}{2bt}} ) 
\end{align}
Since $\sum_{t=1}^{T}{\frac{1}{t}} \le 1+ \log T $, we know that
\begin{align}
    \sum_{t=1}^{T} {\beta_t\mathbb{E}\left< G_t,  2{\boldsymbol{\Sigma}}_{t} \right> } \le  2 \beta B_1 \|\Sigma_1 \|_2 +  2 B_1 \left (   \frac{1 + \log T}{2b} \right) =  2 B_1 \left ( \beta \|\Sigma_1 \|_2 +    \frac{1 + \log T}{2b} \right)
\end{align}

\end{proof}

\begin{theorem*}
Given a  convex function $f({\boldsymbol{x}})$, define $\bar{J}({\boldsymbol{\theta}}):=\mathbb{E}_{p({\boldsymbol{x}};{\boldsymbol{\theta}})}[f({\boldsymbol{x}})]$ 
 for Guassian distribution with parameter ${\boldsymbol{\theta} } := \{{\boldsymbol{\mu}} , \Sigma^{\frac{1}{2}}\} \in {\boldsymbol{\Theta} } $ and  ${\boldsymbol{\Theta} }: = \{{\boldsymbol{\mu}} , \Sigma^{\frac{1}{2}}  \big|\; {\boldsymbol{\mu}} \in \mathcal{R}^d ,  \Sigma \in \mathcal{S}^{+} \}   $. Suppose $\bar{J}({\boldsymbol{\theta}})$ be  $\gamma$-strongly convex. Let $\widehat{G}_t$ be positive semi-definite matrix such that  $  b I \preceq \widehat{G}_t \preceq \frac{\gamma}{2} I$.  Suppose $\Sigma_1 \in \mathcal{S}^{++} $ and $\|\Sigma_1 \| \le \rho$, $  \mathbb{E} \widehat{g}_t = \nabla _ {\boldsymbol{\mu} = \boldsymbol{\mu}_t}\bar{J}$. Assume furthermore  $ \| \nabla _ {\Sigma = \Sigma_t}\bar{J} \|_{tr} \le B_1$ and $ \| {\boldsymbol{\mu}}^* - {\boldsymbol{\mu}}_{1}    \|^2_{\Sigma_{1}^{-1}} \le R$,  $\mathbb{E} \| \widehat{g}_t \|_2^2 \le \mathcal{B}$ . Set $\beta_t = \beta$, then Algorithm~\ref{GF} can achieve
%  \begin{align}
%   \frac{1}{T}\left[  \sum_{t=1}^{T} {\mathbb{E}f({\boldsymbol{\mu}}_t) }  \right] - f({\boldsymbol{\mu}}^*)   \le  \frac{ 2bR + 2b\beta\rho(4B_1 + \beta\mathcal{B}) +   4 B_1 (1 + \log T)   + (1+ \log T) \beta \mathcal{B} }{ 4\beta b T} 
%   = \mathcal{O}\left( \frac{\log T}{T} \right)
%  \end{align}
 \begin{equation}
     \resizebox{1\hsize}{!}{ $ \frac{1}{T}\left[  \sum_{t=1}^{T} {\mathbb{E}f({\boldsymbol{\mu}}_t) }  \right] - f({\boldsymbol{\mu}}^*)   \le  \frac{ 2bR + 2b\beta\rho(4B_1 + \beta\mathcal{B}) +   4 B_1 (1 + \log T)   + (1+ \log T) \beta \mathcal{B} }{ 4\beta b T} 
   = \mathcal{O}\left( \frac{\log T}{T} \right)$ }
 \end{equation}
\end{theorem*}

\begin{proof}
From Lemma 1 to Lemma 4, we know that
% \begin{align}
%      \frac{1}{2} \mathbb{E} \| {\boldsymbol{\mu}}^* - {\boldsymbol{\mu}}_{t+1}    \|^2_{\Sigma_{t+1}^{-1}} & \le \frac{1}{2} \mathbb{E}\| {\boldsymbol{\mu}}^* - {\boldsymbol{\mu}}_{t}    \|^2_{\Sigma_{t}^{-1}}  + \beta_t  \mathbb{E}  (\bar{J}({\boldsymbol{\theta}}^*)-\bar{J}({\boldsymbol{\theta}}_t)) + \beta_t \mathbb{E} \left< G_t,  2{\boldsymbol{\Sigma}}_{t} \right> 
%      + \frac{\beta^2_t}{2}  \mathbb{E} \| \Sigma_{t+1} \|_2 \| \widehat{g}_t \|_2^2
% \end{align}
\begin{equation}
     \resizebox{1\hsize}{!}{ $\frac{1}{2} \mathbb{E} \| {\boldsymbol{\mu}}^* - {\boldsymbol{\mu}}_{t+1}    \|^2_{\Sigma_{t+1}^{-1}}  \le \frac{1}{2} \mathbb{E}\| {\boldsymbol{\mu}}^* - {\boldsymbol{\mu}}_{t}    \|^2_{\Sigma_{t}^{-1}}  + \beta_t  \mathbb{E}  (\bar{J}({\boldsymbol{\theta}}^*)-\bar{J}({\boldsymbol{\theta}}_t)) + \beta_t \mathbb{E} \left< G_t,  2{\boldsymbol{\Sigma}}_{t} \right>  + \frac{\beta^2_t}{2}  \mathbb{E} \| \Sigma_{t+1} \|_2 \| \widehat{g}_t \|_2^2$ }
\end{equation}

Sum up both sides from $t=1$ to $t = T$ and rearrange terms,  we get 
\begin{align}
    \sum_{t=1}^{T} { \beta_t \mathbb{E} \left[ \bar{J}({\boldsymbol{\theta}}_t) -\bar{J}({\boldsymbol{\theta}}^**) \right] } & \le \frac{1}{2}  \mathbb{E}\| {\boldsymbol{\mu}}^* - {\boldsymbol{\mu}}_{1}    \|^2_{\Sigma_{1}^{-1}}  - \frac{1}{2}\mathbb{E}\| {\boldsymbol{\mu}}^* - {\boldsymbol{\mu}}_{T+1}    \|^2_{\Sigma_{T+1}^{-1}}   \\ 
    & + \sum_{t=1}^{T} {\beta_t \mathbb{E} \left< G_t,  2{\boldsymbol{\Sigma}}_{t} \right> } + \sum_{t=1}^{T} {  \frac{\beta_t^2}{2}  \mathbb{E} \| \Sigma_{t+1} \|_2 \| \widehat{g}_t \|_2^2 }
\end{align}
Since $\beta_t = \beta$, we can obtain that
\begin{align}
\label{BlackBoundin}
    \frac{1}{T}\left[  \sum_{t=1}^{T} {\mathbb{E} \bar{J}({\boldsymbol{\theta}}_t) }  \right] - \bar{J}({\boldsymbol{\theta}}^*) & \le \frac{\frac{1}{2}  \mathbb{E}\| {\boldsymbol{\mu}}^* - {\boldsymbol{\mu}}_{1}    \|^2_{\Sigma_{1}^{-1}} \!+\!  \sum_{t=1}^{T} {\beta_t \mathbb{E}\left< G_t,  2{\boldsymbol{\Sigma}}_{t} \right> }  \!+\!  \frac{\beta^2}{2} \sum_{t=1}^{T} {    \mathbb{E} \| \Sigma_{t+1} \|_2 \| \widehat{g}_t \|_2^2 } }{T \beta } \\
    & \le \frac{\frac{1}{2}  R +  \sum_{t=1}^{T} {\beta_t  \mathbb{E}\left< G_t,  2{\boldsymbol{\Sigma}}_{t} \right> }  +  \frac{\beta^2}{2}\mathcal{B} \sum_{t=1}^{T} {    \mathbb{E} \| \Sigma_{t+1} \|_2  } }{T \beta } 
      \label{iMinbound}
\end{align}

From Eq.(\ref{SigmaRate}), we know that
\begin{align}
       \|  \Sigma_{t+1} \|_2  \le  \frac{1}{2t\beta b} 
\end{align}
Since $\sum_{t=1}^{T}{\frac{1}{t}}\le 1+ \log T$, we know that $\sum_{t=1}^{T} {    \mathbb{E} \| \Sigma_{t+1} \|_2} \le \| \Sigma_1\|_2 + \frac{1+\log T}{2\beta b} \le \rho  + \frac{1+\log T}{2\beta b}$

In addition, from Lemma~\ref{innerBound}, we know that 
\begin{align}
     \sum_{t=1}^{T} {\beta_t\mathbb{E}\left< G_t,  2{\boldsymbol{\Sigma}}_{t} \right> } \le  2 B_1 \left ( \beta \|\Sigma_1 \|_2 +    \frac{1 + \log T}{2b} \right) \le 2 B_1  \left ( \beta \rho + \frac{1 + \log T}{2b} \right) 
\end{align}

Plug all them into (\ref{iMinbound}), we can get
\begin{align}
    \frac{1}{T}\left[  \sum_{t=1}^{T} {\mathbb{E}\bar{J}({\boldsymbol{\theta}}_t) }  \right] - \bar{J}({\boldsymbol{\theta}}^*) & \le \frac{\frac{1}{2}R +  2 B_1  \left ( \beta \rho + \frac{1 + \log T}{2b} \right) + \frac{\beta^2\rho\mathcal{B}}{2}  + \frac{(1+ \log T)\beta \mathcal{B}}{4b} }{ T\beta} \\
    & =\frac{ 2bR +  8 B_1 b\beta \rho  +    4 B_1 (1 + \log T) + 2 b \beta^2 \rho \mathcal{B}  + (1+ \log T) \beta \mathcal{B} }{ 4\beta b T}\\
    & =  \frac{ 2bR + 2b\beta\rho(4B_1 + \beta\mathcal{B}) +   4 B_1 (1 + \log T)   + (1+ \log T) \beta \mathcal{B} }{ 4\beta b T}   \\
   & = \mathcal{O}\left( \frac{\log T}{T} \right)
\end{align}
Since $f(\boldsymbol{x})$ is a convex function, we know $f(\boldsymbol{\mu}) \le  \bar{J}({\boldsymbol{\mu}},\Sigma) = \mathbb{E}[f(\boldsymbol{x})]$.  Note that  for an optimum point ${\boldsymbol{\mu}}^*$  of $f(\boldsymbol{x})$, ${\boldsymbol{\theta}^*}=({\boldsymbol{\mu}^*},\boldsymbol{0})$ is an optimum of  $\bar{J}({\boldsymbol{\theta}})$, i.e., $f(\boldsymbol{\mu}^*) =\bar{J}({\boldsymbol{\theta}^*})$. Thus, we can obtain that 
\begin{align}
   \frac{1}{T}\left[  \sum_{t=1}^{T} {\mathbb{E}f({\boldsymbol{\mu}}_t) }  \right] - f({\boldsymbol{\mu}}^*)  & \le    \frac{1}{T}\left[  \sum_{t=1}^{T} {\mathbb{E}\bar{J}({\boldsymbol{\theta}}_t) }  \right] \!-\! \bar{J}({\boldsymbol{\theta}}^*) \\ & \le  \frac{ 2bR + 2b\beta\rho(4B_1 + \beta\mathcal{B}) +   4 B_1 (1 + \log T)   + (1+ \log T) \beta \mathcal{B} }{ 4\beta b T}  \\
   & \le \mathcal{O}\left( \frac{\log T}{T} \right)
\end{align}

\end{proof}

\section{Proof of Theorem 5}

\begin{Lemma}
\label{boundZ}
For a $L$-Lipschitz continuous   black box function $f(\boldsymbol{x})$. Let $\widehat{G}_t$ be positive semi-definite matrix such that  $  b I \preceq \widehat{G}_t $ with $b >0$. Suppose  the gradient estimator $\widehat{g}_t$ is defined as 
\begin{align}
    \widehat{g}_t =\Sigma_t^{-\frac{1}{2}}\boldsymbol{z} \left( {f}(\boldsymbol{\mu}_t+\Sigma_t^{\frac{1}{2}}\boldsymbol{z}) -{f}(\boldsymbol{\mu}_t) \right)
\end{align}
where $\boldsymbol{z} \sim \mathcal{N}(\boldsymbol{0},I)$.
Then $\widehat{g}_t$ is an unbiased estimator of $ \nabla _ {\boldsymbol{\mu}} \mathbb{E}_{p}[f({\boldsymbol{x}})] $   and  $ \mathbb{E}  \| \Sigma_{t+1} \|_2 \| \widehat{g}_t \|^2_{2} \le  L^2 \| \Sigma_{t}  \|_2 (d+4)^2  $
\end{Lemma}

\begin{proof}
We first show the unbiased estimator.
\begin{align}
    \mathbb{E}[\widehat{g}_t] 
    & = \mathbb{E} \left[ \Sigma_t^{-\frac{1}{2}}\boldsymbol{z} {f}(\boldsymbol{\mu}_t+\Sigma_t^{\frac{1}{2}}\boldsymbol{z}) \right]    - \mathbb{E} \left[ \Sigma_t^{-\frac{1}{2}}\boldsymbol{z} {f}(\boldsymbol{\mu}_t) \right]\\
    & =   \mathbb{E} \left[ \Sigma_t^{-\frac{1}{2}}\boldsymbol{z} f(\boldsymbol{\mu}_t+\Sigma_t^{\frac{1}{2}}\boldsymbol{z}) \right]    \\
    & =  \mathbb{E}_{p(\boldsymbol{\mu}_t,\Sigma_t)} \left[  \Sigma_t^{-1}(\boldsymbol{x} - \boldsymbol{\mu}_t)  f(\boldsymbol{x})   \right] \\
    & = \nabla _ {\boldsymbol{\mu}} \mathbb{E}_{p}[f({\boldsymbol{x}})]
\end{align}
The last equality holds by   Theorem~\ref{theB}.

Now, we prove the bound of $\mathbb{E}_{p}\| \Sigma_{t+1} \|_2 \| \widehat{g}_t \|^2_2$. 
\begin{align}
  \| \Sigma_{t+1} \|_2 \;  \| \widehat{g}_t \|^2_2  & =  \| \Sigma_{t+1} \|_2 \; \| \Sigma_{t}^{-\frac{1}{2}} \boldsymbol{z} \|^2_2\left( {f}(\boldsymbol{\mu}_t+\Sigma_t^{\frac{1}{2}}\boldsymbol{z}) -{f}(\boldsymbol{\mu}_t) \right)^2   \\
  &  \le  \| \Sigma_{t+1} \|_2 \; \| \Sigma_{t}^{-\frac{1}{2}} \boldsymbol{z} \|^2_2 \; L^2  \| \Sigma_{t}^{\frac{1}{2}} \boldsymbol{z} \|^2_2 \\
  & \le \| \Sigma_{t+1} \|_2 \;  \| \Sigma_{t}^{ - \frac{1}{2} } \|_2^2  \; \|  \boldsymbol{z} \|^2_2 \; L^2  \| \Sigma_{t}^{\frac{1}{2}} \boldsymbol{z} \|^2_2  \\
  & = \| \Sigma_{t+1} \|_2 \;  \| \Sigma_{t}^{ -1 } \|_2  \; \|   \boldsymbol{z} \|^2_2 \;\; L^2  \| \Sigma_{t}^{\frac{1}{2}} \boldsymbol{z} \|^2_2
\end{align}
Since $\| \Sigma_{t+1} \|_2  \le   \| \Sigma_{t}^{ } \|_2$ proved in Lemma 4 (below Eq.(\ref{BlackInq}) ), we get that
\begin{align}
    \| \Sigma_{t+1} \|_2 \;  \| \widehat{g}_t \|^2_2  \le   \|  \boldsymbol{z} \|^2_2 \;\;  L^2  \| \Sigma_{t}^{\frac{1}{2}} \boldsymbol{z} \|^2_2  \le L^2  \| \Sigma_{t}  \|_2 \; \|  \boldsymbol{z} \|^4_2
\end{align}
Since $\mathbb{E}\|  \boldsymbol{z} \|^4_2 \le (d+4)^2$  shown in ~\cite{nesterov2017random}, we can obtain that 
\begin{align}
  \mathbb{E}  \| \Sigma_{t+1} \|_2 \| \widehat{g}_t \|^2_{2} \le  L^2 \| \Sigma_{t}  \|_2 (d+4)^2  
\end{align}

\end{proof}

\begin{theorem*}
For a $L$-Lipschitz continuous convex  black box function $f(\boldsymbol{x})$,  define $\bar{J}({\boldsymbol{\theta}}):=\mathbb{E}_{p({\boldsymbol{x}};{\boldsymbol{\theta}})}[f({\boldsymbol{x}})]$ 
 for Guassian distribution with parameter ${\boldsymbol{\theta} } := \{{\boldsymbol{\mu}} , \Sigma^{\frac{1}{2}}\} \in {\boldsymbol{\Theta} } $ and  ${\boldsymbol{\Theta} }: = \{{\boldsymbol{\mu}} , \Sigma^{\frac{1}{2}}  \big|\; {\boldsymbol{\mu}} \in \mathcal{R}^d ,  \Sigma \in \mathcal{S}^{+} \}   $. Suppose $\bar{J}({\boldsymbol{\theta}})$ be  $\gamma$-strongly convex. Let $\widehat{G}_t$ be positive semi-definite matrix such that  $  b  \boldsymbol{I}  \preceq \widehat{G}_t \preceq \frac{\gamma}{2} \boldsymbol{I}$.  Suppose $\Sigma_1 \in \mathcal{S}^{++} $ and $\|\Sigma_1 \|_2 \le \rho$. Assume furthermore  $ \| \nabla _ {\Sigma = \Sigma_t}\bar{J} \|_{tr} \le B_1$ and $ \| {\boldsymbol{\mu}}^* - {\boldsymbol{\mu}}_{1}    \|^2_{\Sigma_{1}^{-1}} \le R$,   . Set $\beta_t = \beta$ and employ estimator $\widehat{g}_t$ in Eq.(\ref{singleg}), then Algorithm~\ref{GF} can achieve
 \begin{align}
 & \frac{1}{T}\left[  \sum_{t=1}^{T} {\mathbb{E}  f({\boldsymbol{\mu}}_t) }  \right] - f({\boldsymbol{\mu}}^*) \\ & \le    \frac{ 2bR + 2b\beta\rho(4B_1 + 2\beta L^2(d+4)^2) +   4 B_1 (1 + \log T)   + (1+ \log T) \beta L^2(d+4)^2 }{ 4\beta b T} \\
   &   = \mathcal{O} \left( \frac{d^2 \log T}{T} \right)
 \end{align}
\end{theorem*}

\begin{proof}
We are now ready to prove Theorem~\ref{BObound}.

From Lemma~\ref{boundZ}, we know $ \mathbb{E}  \| \Sigma_{t+1} \|_2 \| \widehat{g}_t \|^2_{2} \le  L^2 \| \Sigma_{t}  \|_2 (d+4)^2  $.  Note that $ \|  \Sigma_{t+1} \|_2  \le  \frac{1}{2t\beta b} $ from Eq.(\ref{SigmaRate}), we can obtain that
\begin{align}
    \mathbb{E}  \| \Sigma_{t+1} \|_2 \| \widehat{g}_t \|^2_{2} \le  L^2 \| \Sigma_{t}  \|_2 (d+4)^2 \le \frac{L^2(d+4)^2}{2(t-1)\beta b} 
\end{align}

Plug it into Eq.(\ref{BlackBoundin}), also note that $\| \Sigma_{2} \|_2 \le \| \Sigma_{1} \|_2$, we get that 
\begin{align}
    & \frac{1}{T}\left[  \sum_{t=1}^{T} {\mathbb{E}\bar{J}({\boldsymbol{\theta}}_t) }  \right] - \bar{J}({\boldsymbol{\theta}}^*) \\ & \le \frac{\frac{1}{2}  \mathbb{E}\| {\boldsymbol{\mu}}^* - {\boldsymbol{\mu}}_{1}    \|^2_{\Sigma_{1}^{-1}} +  \sum_{t=1}^{T} {\beta_t \mathbb{E} \left< G_t,  2{\boldsymbol{\Sigma}}_{t} \right> }  + \beta^2 \| \Sigma_1 \|_2 L^2(d+4)^2 + \frac{\beta L^2(d+4)^2}{4b} \sum_{t=1}^{T} { \frac{1}{t} }  }{T \beta } \\
     & \le \frac{\frac{1}{2} R +  \sum_{t=1}^{T} {\beta_t \mathbb{E} \left< G_t,  2{\boldsymbol{\Sigma}}_{t} \right> }  + \beta^2 \| \Sigma_1 \|_2 L^2(d+4)^2 + \frac{\beta L^2(d+4)^2}{4b} (1+\log T)  }{T \beta }
\end{align}

In addition, from Lemma~\ref{innerBound}, we know that 
\begin{align}
     \sum_{t=1}^{T} {\beta_t \mathbb{E}\left< G_t,  2{\boldsymbol{\Sigma}}_{t} \right> } \le  2 B_1 \left ( \beta \|\Sigma_1 \|_2 +    \frac{1 + \log T}{2b} \right) \le 2 B_1  \left ( \beta \rho + \frac{1 + \log T}{2b} \right) 
\end{align}

Then, we can get that
\begin{align}
   &  \frac{1}{T}\left[  \sum_{t=1}^{T} {\mathbb{E}\bar{J}({\boldsymbol{\theta}}_t) }  \right] - \bar{J}({\boldsymbol{\theta}}^*)\\ & \le \frac{\frac{1}{2} R +   2 B_1  \left ( \beta \rho + \frac{1 + \log T}{2b} \right)   + \beta^2 \rho L^2(d+4)^2 + \frac{\beta L^2(d+4)^2}{4b} (1+\log T)  }{T \beta } \\
     & =    \frac{ 2bR + 2b\beta\rho(4B_1 + 2\beta L^2(d+4)^2) +   4 B_1 (1 + \log T)   + (1+ \log T) \beta L^2(d+4)^2 }{ 4\beta b T} \\
     & = \mathcal{O} \left( \frac{d^2 \log T}{T} \right)
\end{align}

Since $f(\boldsymbol{x})$ is a convex function, we know that
\begin{align}
 &  \frac{1}{T}\left[  \sum_{t=1}^{T} {\mathbb{E}f({\boldsymbol{\mu}}_t) }  \right] - f({\boldsymbol{\mu}}^*) \\ &  \le    \frac{1}{T}\left[  \sum_{t=1}^{T} {\bar{J}({\boldsymbol{\theta}}_t) }  \right] - \bar{J}({\boldsymbol{\theta}}^*)   \\
     & \le    \frac{ 2bR + 2b\beta\rho(4B_1 + 2\beta L^2(d+4)^2) +   4 B_1 (1 + \log T)   + (1+ \log T) \beta L^2(d+4)^2 }{ 4\beta b T} \\
     & = \mathcal{O} \left( \frac{d^2 \log T}{T} \right)
\end{align}

\end{proof}

%\newpage
\section{Variance Reduction}

\begin{Lemma}
\label{IboundZ}
For a $L$-Lipschitz continuous   black box function $f(\boldsymbol{x})$. Suppose $\Sigma_t =  \sigma_t^{2} \boldsymbol{I}$  with $\sigma_t >0$ for $t \in \{1,\cdots,T\}$. Suppose  the gradient estimator $\widehat{g}_t$ is defined as 
\begin{align}
\label{orthgonoalg}
    \widehat{g}_t =\frac{1}{N}\sum_{i=1}^{N} {\Sigma_t^{-\frac{1}{2}}\boldsymbol{z}_i \left( {f}(\boldsymbol{\mu}_t+\Sigma_t^{\frac{1}{2}}\boldsymbol{z}_i) -{f}(\boldsymbol{\mu}_t) \right) }
\end{align}
where $\boldsymbol{Z}=[\boldsymbol{z}_1,\cdots,\boldsymbol{z}_N]$ has marginal distribution  $ \mathcal{N}(\boldsymbol{0},\boldsymbol{I})$ and $\boldsymbol{Z}^\top \boldsymbol{Z} = \boldsymbol{I}$.
Then $\widehat{g}_t$ is an unbiased estimator of $ \nabla _ {\boldsymbol{\mu}} \mathbb{E}_{p}[f({\boldsymbol{x}})] $   and  $ \mathbb{E}_{\boldsymbol{Z}}  \| \Sigma_{t+1} \|_2 \| \widehat{g}_t \|^2_{2} \le \frac{ \sigma^2_{t+1} L^2(d+4)^2}{N}   $ for $N \le d$.
\end{Lemma}

\begin{proof}
We first show the unbiased estimator.
\begin{align}
    \mathbb{E}_{\boldsymbol{Z}}[\widehat{g}_t] & = \mathbb{E}_{\boldsymbol{Z}}\left[\frac{1}{N}\sum_{i=1}^{N} {\Sigma_t^{-\frac{1}{2}}\boldsymbol{z}_i \left( {f}(\boldsymbol{\mu}_t+\Sigma_t^{\frac{1}{2}}\boldsymbol{z}_i) -{f}(\boldsymbol{\mu}_t) \right) }\right]  \\
    & = \frac{1}{N}\sum_{i=1}^{N}{ \mathbb{E}_{\boldsymbol{Z}} \left[ \Sigma_t^{-\frac{1}{2}}\boldsymbol{z}_i \left( {f}(\boldsymbol{\mu}_t+\Sigma_t^{\frac{1}{2}}\boldsymbol{z}_i) -{f}(\boldsymbol{\mu}_t) \right)  \right] } \\
    & = \mathbb{E}_{\boldsymbol{z}} \left[ \Sigma_t^{-\frac{1}{2}}\boldsymbol{z} \left( {f}(\boldsymbol{\mu}_t+\Sigma_t^{\frac{1}{2}}\boldsymbol{z}) -{f}(\boldsymbol{\mu}_t) \right)  \right]  \\
    & = \mathbb{E}_{\boldsymbol{z}} \left[ \Sigma_t^{-\frac{1}{2}}\boldsymbol{z} {f}(\boldsymbol{\mu}_t+\Sigma_t^{\frac{1}{2}}\boldsymbol{z}) \right]    - \mathbb{E}_{\boldsymbol{z}} \left[ \Sigma_t^{-\frac{1}{2}}\boldsymbol{z} {f}(\boldsymbol{\mu}_t) \right]\\
    & =   \mathbb{E} \left[ \Sigma_t^{-\frac{1}{2}}\boldsymbol{z} f(\boldsymbol{\mu}_t+\Sigma_t^{\frac{1}{2}}\boldsymbol{z}) \right]    \\
    & =  \mathbb{E}_{p(\boldsymbol{\mu}_t,\Sigma_t)} \left[  \Sigma_t^{-1}(\boldsymbol{x} - \boldsymbol{\mu}_t)  f(\boldsymbol{x})   \right] \\
    & = \nabla _ {\boldsymbol{\mu}} \mathbb{E}_{p}[f({\boldsymbol{x}})]
\end{align}
The last equality holds by   Theorem~\ref{theB}.

Now, we prove the bound of $\mathbb{E}_{p}\| \Sigma_{t+1} \|_2 \| \widehat{g}_t \|^2_2$. 
\begin{align}
&  \;\;\;\; \| \Sigma_{t+1} \|_2 \;  \| \widehat{g}_t \|^2_2 \\ & =   \sigma^2_{t+1} \left\|\frac{1}{N}\sum_{i=1}^{N} {\sigma_t^{-1}\boldsymbol{z}_i \left( {f}(\boldsymbol{\mu}_t+\sigma_t^{}\boldsymbol{z}_i) -{f}(\boldsymbol{\mu}_t) \right) }  \right\|_2^2 \\
  & =  \frac{\sigma^2_{t+1}}{N^2}\sum_{i=1}^{N} { \left\| {\sigma_t^{-1}\boldsymbol{z}_i \left( {f}(\boldsymbol{\mu}_t+\sigma_t^{}\boldsymbol{z}_i) -{f}(\boldsymbol{\mu}_t) \right) }  \right\|_2^2 } \\ &  + \frac{\sigma^2_{t+1}\sigma_t^{-2}}{N^2}\sum_{i=1}^{N}{\sum_{i\ne j}^N{ \boldsymbol{z}_i^\top \boldsymbol{z}_j ( {f}(\boldsymbol{\mu}_t+\sigma_t^{}\boldsymbol{z}_i) \!-\!{f}(\boldsymbol{\mu}_t) )( {f}(\boldsymbol{\mu}_t+\sigma_t^{}\boldsymbol{z}_j) \!-\!{f}(\boldsymbol{\mu}_t) )   }}  \\
  & =  \frac{\sigma^2_{t+1}}{N^2}\sum_{i=1}^{N} { \left\| {\sigma_t^{-1}\boldsymbol{z}_i \left( {f}(\boldsymbol{\mu}_t+\sigma_t^{}\boldsymbol{z}_i) -{f}(\boldsymbol{\mu}_t) \right) }  \right\|_2^2 }   \\
  & \le \frac{\sigma^2_{t+1}\sigma_t^{-2} \sigma_t^{2} L^2}{N^2}\sum_{i=1}^{N} {\|\boldsymbol{z}_i\|_2^4}   = \frac{ \sigma^2_{t+1} L^2}{N^2}\sum_{i=1}^{N} {\|\boldsymbol{z}_i\|_2^4}
\end{align}
Thus,we know that
\begin{align}
    \mathbb{E}_{\boldsymbol{Z}}\left[ \| \Sigma_{t+1} \|_2   \| \widehat{g}_t \|^2_2\right] \le \frac{ \sigma^2_{t+1} L^2}{N^2}\mathbb{E}_{\boldsymbol{Z}}\sum_{i=1}^{N} {\|\boldsymbol{z}_i\|_2^4} = \frac{ \sigma^2_{t+1} L^2}{N} \mathbb{E}_{\boldsymbol{z}}[\|\boldsymbol{z}\|_2^4]
\end{align}
Since $\mathbb{E}_{\boldsymbol{z}}\|  \boldsymbol{z} \|^4_2 \le (d+4)^2$  shown in ~\cite{nesterov2017random}, we can obtain that 
\begin{align}
  \mathbb{E}_{\boldsymbol{Z}}  \| \Sigma_{t+1} \|_2 \| \widehat{g}_t \|^2_{2} \le \frac{ \sigma^2_{t+1} L^2(d+4)^2}{N}  
\end{align}

\end{proof}

\begin{theorem*}
For a $L$-Lipschitz continuous convex  black box function $f(\boldsymbol{x})$,  define $\bar{J}({\boldsymbol{\theta}}):=\mathbb{E}_{p({\boldsymbol{x}};{\boldsymbol{\theta}})}[f({\boldsymbol{x}})]$ 
 for Guassian distribution with parameter ${\boldsymbol{\theta} } := \{{\boldsymbol{\mu}} , \sigma_t\boldsymbol{I} \in {\boldsymbol{\Theta} } $ and  ${\boldsymbol{\Theta} }: = \{{\boldsymbol{\mu}} , \Sigma^{\frac{1}{2}}  \big|\; {\boldsymbol{\mu}} \in \mathcal{R}^d ,  \Sigma \in \mathcal{S}^{+} \}   $. Suppose $\bar{J}({\boldsymbol{\theta}})$ be  $\gamma$-strongly convex. Let $\widehat{G}_t = b\boldsymbol{I}$ with $b\le \frac{\gamma}{2}$.  Suppose $\|\Sigma_1 \|_2  \le \rho = \frac{1}{d} $. Assume furthermore  $ \| \nabla _ {\Sigma = \Sigma_t}\bar{J} \|_{tr} \le B_1$ and $ \| {\boldsymbol{\mu}}^* - {\boldsymbol{\mu}}_{1}    \|^2_{\Sigma_{1}^{-1}} \le R$,   . Set $\beta_t = \beta$ and employ orthogonal estimator $\widehat{g}_t$ in Eq.(\ref{orthgonoalg}) with $N=d$, then Algorithm~\ref{GF} can achieve
 \begin{align}
 & \frac{1}{T}\left[  \sum_{t=1}^{T} {\mathbb{E}f({\boldsymbol{\mu}}_t) }  \right] - f({\boldsymbol{\mu}}^*) \\ & \le    \frac{ 2bR + 2b\beta(4B_1/d + 2\beta L^2(d+4)^2/d) +   4 B_1 (1 + \log T)   + (1+ \log T) \beta L^2(d+4)^2/d }{ 4\beta b T} \\
  &   = \mathcal{O} \left( \frac{d \log T}{T} \right)
 \end{align}
\end{theorem*}

\begin{proof}
The proof is similar to the proof of Theorem~\ref{BObound}, 

From Lemma~\ref{IboundZ} and $N=d$, we know $ \mathbb{E}  \| \Sigma_{t+1} \|_2 \| \widehat{g}_t \|^2_{2} \le  \frac{ \sigma^2_{t+1} L^2(d+4)^2}{d}   $.  Note that $\sigma^2_{t+1}= \|  \Sigma_{t+1} \|_2  \le  \frac{1}{2t\beta b} $ from Eq.(\ref{SigmaRate}), we can obtain that
\begin{align}
    \mathbb{E}  \| \Sigma_{t+1} \|_2 \| \widehat{g}_t \|^2_{2} \le \frac{ \sigma^2_{t+1} L^2(d+4)^2}{d} \le \frac{L^2(d+4)^2}{2t\beta b d} 
\end{align}

Plug it into Eq.(\ref{BlackBoundin}), we get that 
\begin{align}
  &   \frac{1}{T}\left[  \sum_{t=1}^{T} {\mathbb{E}\bar{J}({\boldsymbol{\theta}}_t) }  \right] - \bar{J}({\boldsymbol{\theta}}^*) \\ & \le \frac{\frac{1}{2}  \mathbb{E}\| {\boldsymbol{\mu}}^* - {\boldsymbol{\mu}}_{1}    \|^2_{\Sigma_{1}^{-1}} +  \sum_{t=1}^{T} {\beta_t\mathbb{E} \left< G_t,  2{\boldsymbol{\Sigma}}_{t} \right> }  + \beta^2 \| \Sigma_1 \|_2 L^2(d+4)^2 + \frac{\beta L^2(d+4)^2}{4bd} \sum_{t=1}^{T} { \frac{1}{t} }  }{T \beta } \\
     & \le \frac{\frac{1}{2} R +  \sum_{t=1}^{T} {\beta_t \mathbb{E}\left< G_t,  2{\boldsymbol{\Sigma}}_{t} \right> }  + \beta^2 \| \Sigma_1 \|_2 L^2(d+4)^2 + \frac{\beta L^2(d+4)^2}{4bd} (1+\log T)  }{T \beta }
\end{align}

In addition, from Lemma~\ref{innerBound}, we know that 
\begin{align}
     \sum_{t=1}^{T} {\beta_t\mathbb{E}\left< G_t,  2{\boldsymbol{\Sigma}}_{t} \right> } \le  2 B_1 \left ( \beta \|\Sigma_1 \|_2 +    \frac{1 + \log T}{2b} \right) \le 2 B_1  \left ( \beta \rho + \frac{1 + \log T}{2b} \right) 
\end{align}

Then, we can get that
\begin{align}
  &   \frac{1}{T}\left[  \sum_{t=1}^{T} {\mathbb{E}\bar{J}({\boldsymbol{\theta}}_t) }  \right] - \bar{J}({\boldsymbol{\theta}}^*) \\ & \le \frac{\frac{1}{2} R +   2 B_1  \left ( \beta \rho + \frac{1 + \log T}{2b} \right)   + \beta^2 \rho L^2(d+4)^2 + \frac{\beta L^2(d+4)^2}{4bd} (1+\log T)  }{T \beta } \\
     & =    \frac{ 2bR + 2b\beta\rho(4B_1 + 2\beta L^2(d+4)^2) +   4 B_1 (1 + \log T)   + (1+ \log T) \beta L^2(d+4)^2/d }{ 4\beta b T} \\
      & =    \frac{ 2bR + 2b\beta(4B_1/d + 2\beta L^2(d+4)^2/d) +   4 B_1 (1 + \log T)   + (1+ \log T) \beta L^2(d+4)^2/d }{ 4\beta b T} \\
     & = \mathcal{O} \left( \frac{d \log T}{T} \right)
\end{align}

Since $f(\boldsymbol{x})$ is a convex function, we know that
\begin{align}
   \frac{1}{T}\left[  \sum_{t=1}^{T} {\mathbb{E}f({\boldsymbol{\mu}}_t) }  \right] - f({\boldsymbol{\mu}}^*) &  \le    \frac{1}{T}\left[  \sum_{t=1}^{T} {\mathbb{E}\bar{J}({\boldsymbol{\theta}}_t) }  \right] - \bar{J}({\boldsymbol{\theta}}^*)   
      \le \mathcal{O} \left( \frac{d \log T}{T} \right)
\end{align}

\end{proof}

\section{Discrete Update}
\label{DiscreteUpdateD}

For function $f(\boldsymbol{x})$ over binary variable $\boldsymbol{x} \in \{0,1\}^d$, 
we employ Bernoulli distribution with parameter $\boldsymbol{\boldsymbol{p}}= [p_1, \cdots, p_d  ]^\top$ as the underlying distribution,  where $p_i$ denote the probability of  $x_i=1$. 
The gradient of $ \mathbb{E}_{p}[f({\boldsymbol{x}})]$ w.r.t ${\boldsymbol{p}}$ can be derived as follows:
\begin{align}
    \nabla _ {\boldsymbol{p}} \mathbb{E}_{p}[f({\boldsymbol{x}})] & =  \mathbb{E}_{p}[f({\boldsymbol{x}})\nabla _ {\boldsymbol{p}} \log (p({\boldsymbol{x}};\boldsymbol{p}))   ]  \\ 
    & = \sum_{\boldsymbol{x} \in \{0,1\}^d} {\prod_{i=1}^d{p_i^{\boldsymbol{1}(x_i=1)}(1-p_i)^{\boldsymbol{1}(x_i=0)}} f(\boldsymbol{x}) \nabla _ {\boldsymbol{p}}\log \big( \prod_{i=1}^d{p_i^{\boldsymbol{1}(x_i=1)}(1-p_i)^{\boldsymbol{1}(x_i=0)}} \big)  }   \\
    & = \mathbb{E}_{p}\left[ f(\boldsymbol{x})\nabla _ {\boldsymbol{p}}\big(\sum_{i=1}^d{{\boldsymbol{1}(x_i=1)}\log p_i + {\boldsymbol{1}(x_i=0)}\log(1-p_i) }\big)   \right] \\
    & = \mathbb{E}_{p}\left[  f({\boldsymbol{x}})\boldsymbol{h} \right]
\end{align}
where $\boldsymbol{h}_i = \frac{1}{p_i}\boldsymbol{1}(\boldsymbol{x}_i=1) - \frac{1}{1-p_i}\boldsymbol{1}(\boldsymbol{x}_i=0) $.

For  function $f(\boldsymbol{x})$ over  discrete variable $\boldsymbol{x} \in \{1,\cdots,K\}^d$, 
we employ categorical distribution with parameter $\boldsymbol{\boldsymbol{P}}= [\boldsymbol{p}_1, \cdots, \boldsymbol{p}_d  ]^\top$ as the underlying distribution, where the ${ij}$-th element of $\boldsymbol{P}$ (i.e., $\boldsymbol{P}_{ij}$) denote the probability of  $\boldsymbol{x}_i=j$. 
The gradient of $ \mathbb{E}_{p}[f({\boldsymbol{x}})]$ w.r.t ${\boldsymbol{P}}$ can be derived as follows:
\begin{align}
    \nabla _ {\boldsymbol{P}} \mathbb{E}_{p}[f({\boldsymbol{x}})] & =  \mathbb{E}_{p}[f({\boldsymbol{x}})\nabla _ {\boldsymbol{P}} \log (p({\boldsymbol{x}};\boldsymbol{P}))   ]  \\ 
    & = \sum_{\boldsymbol{x} \in \{1, \cdots, K\}^d} {\prod_{i=1}^d{ \prod_{j=1}^K {\boldsymbol{P}_{ij}^{\boldsymbol{1}(\boldsymbol{x}_i=j)}} } f(\boldsymbol{x}) \nabla _ {\boldsymbol{P}}\log \big( \prod_{i=1}^d{ \prod_{j=1}^K {\boldsymbol{P}_{ij}^{\boldsymbol{1}(\boldsymbol{x}_i=j)}} }  \big)  }   \\
    & = \mathbb{E}_{p}\left[ f(\boldsymbol{x})\nabla _ {\boldsymbol{P}}\big(\sum_{i=1}^d{ \sum_{j=1}^K {\boldsymbol{1}(\boldsymbol{x}_i=j) \log \boldsymbol{P}_{ij} }  }\big)   \right] \\
    & = \mathbb{E}_{p}\left[  f({\boldsymbol{x}})\boldsymbol{H} \right]
\end{align}
where $\boldsymbol{H}_{ij} = \frac{1}{\boldsymbol{P}_{ij}}\boldsymbol{1}(\boldsymbol{x}_{i}=j)$.

\section{Test Problems}

\begin{table*}[h]
\centering
\caption{Test functions }
\label{testF}
\begin{tabular}{ll}
\hline
name                          & function     \\ \hline
 Ellipsoid &  $ f({\boldsymbol{x}}):=\sum\nolimits_{i = 1}^d {10^{\frac{6(i-1)}{d-1}}x_i^2} $   \\
 
Discus  &  $f({\boldsymbol{x}}):= 10^6x_1 + \sum\nolimits_{i = 2}^d {x_i^2} $  \\

$\ell_1$-Ellipsoid &  $ f({\boldsymbol{x}}):=\sum\nolimits_{i = 1}^d {10^{\frac{6(i-1)}{d-1}}|x_i|} $   \\

$\ell_\frac{1}{2}$-Ellipsoid &  $ f({\boldsymbol{x}}):=\sum\nolimits_{i = 1}^d {10^{\frac{6(i-1)}{d-1}}|x_i|^{\frac{1}{2}}} $   \\

Levy                          & $f({\boldsymbol{x}}):= \begin{array}{l}
{\sin ^2}(\pi {w_1}) + \sum\limits_{i = 1}^{d - 1} {{{({w_i} - 1)}^2}(1 + 10{{\sin }^2}(\pi {w_i} + 1))}  + {({w_d} - 1)^2}(1 + {\sin ^2}(2\pi {w_d}))\\
{\rm{where}} \; {w_i} = 1 + ({x_i} - 1)/4,\;i \in \{ 1,...,d\} 
\end{array}$       \\

Rastrigin10 & $f({\boldsymbol{x}}):= 10d + \sum\limits_{i = 1}^{d } { (10^{\frac{i-1}{d-1}}x_i)^2 -10\cos{\big( 2\pi 10^{\frac{i-1}{d-1}}x_i } \big) } $  

\\ \hline
\end{tabular}
\end{table*}

\section{Evaluation of ES and Standard Reinforce Gradient Descent for Black-box Optimization }

In the experiments in section~\ref{experiments} in the paper, ES, CMAES, and our INGO employ the same initialization of mean and variance (default initialization in CMAES). In all experiments, we employ the default step-size of ES in  (Salimans et al.) , i.e.,  0.01.  
 We further provide ES with different step-size (1, 0.1,  0.01, 0.001, 0.0001) on minimization of benchmark test  functions in  the  Figure~\ref{ESstep}.  The dimension is set to d=100. The experimental results (mean value over 20 independent runs) are shown in Figure~\ref{ESstep}. We can see that a small stepsize ($10^{-4}$) converge slowly, while a large stepsize ($10^{0}$) may lead to diverge. The default stepsize ($10^{-2}$) used in the paper seems to be a reasonable choice.

The black-box benchmark test functions are ill-conditioned or nonconvex. For the ill-conditioned ellipsoid test functions, the condition number is $10^6$, and different variables have different scales.  The experimental results show that it is challenging to optimize them well without adaptively update covariance and mean.

We further evaluate the ordinary gradient descent with gradients  estimated by function queries, i.e., standard reinforce type descent. We use the diagonal covariance matrix same as our fast algorithm. All employ the same initialization of mean and variance (default initialization in CMAES) same as setting in the paper. The dimension is also set to d=100.  We evaluate different stepsizes. 
The experimental results  (mean value over 20 independent runs) are shown in Figure~\ref{RGD}. It shows that directly updating with estimated gradients does not work well on test functions. It tends to diverge even for stepsize $10^{-6}$.

This is because the gradient estimator of mean proportional to the inverse of the standard deviation, i.e., $\widehat{g}=f(x)(x-\mu) / \sigma^2 = f(x)z / \sigma$ (elementwise for diagonal case).  Similarly, the update of the variance is also proportional to the inverse of the variance. As a result, the smaller the variance, the larger the update. Thus, directly updating with the estimated gradient is unstable when variance is small. For black-box optimization, a large variance means that more focus on exploration, while a small variance means that more focus on exploitation. An adaptive method adaptively controls the balance between exploration and exploitation. The unstable problem of directly updating with the estimated gradient prevent exploitation (high precision search). The unstable problem of directly applying reinforce type descent for black-box optimization is also discussed in \cite{NES}.

In contrast, the update of mean in our method (e.g., Alg.2) is properly scaled by $\sigma^2$. Moreover, our method updates the  inverse of variance instead of variance itself, which provides us a stable update of variance independent of the scale of variance. Thus, our method can update properly when the algorithm adaptively reduces variance for high precision search.

%\newpage

\begin{figure*}[t]
\centering
\subfigure[ES on Ellipsoid]{
\includegraphics[width=0.3\linewidth]{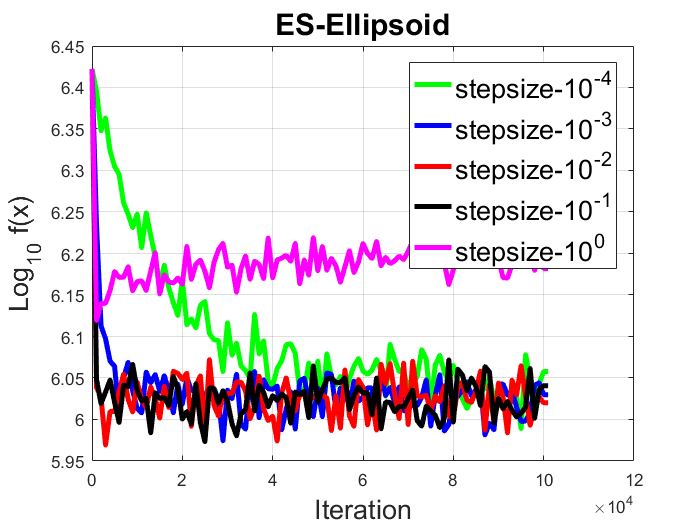}}
\subfigure[ES on $l_1$ Ellipsoid]{
\includegraphics[width=0.3\linewidth]{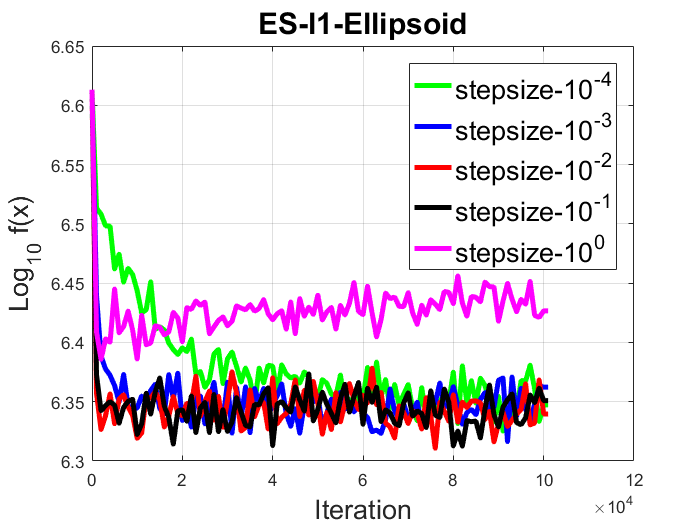}}
\subfigure[ES on $l_{\frac{1}{2}}$ Ellipsoid]{
\includegraphics[width=0.3\linewidth]{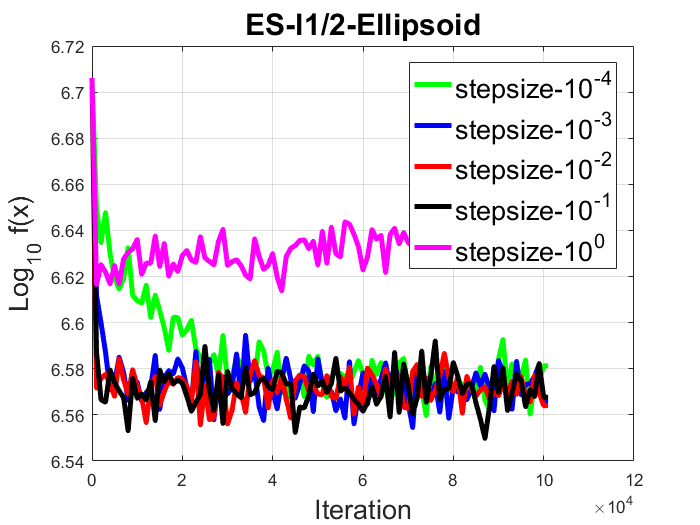}}
\subfigure[ES on Discus]{
\includegraphics[width=0.3\linewidth]{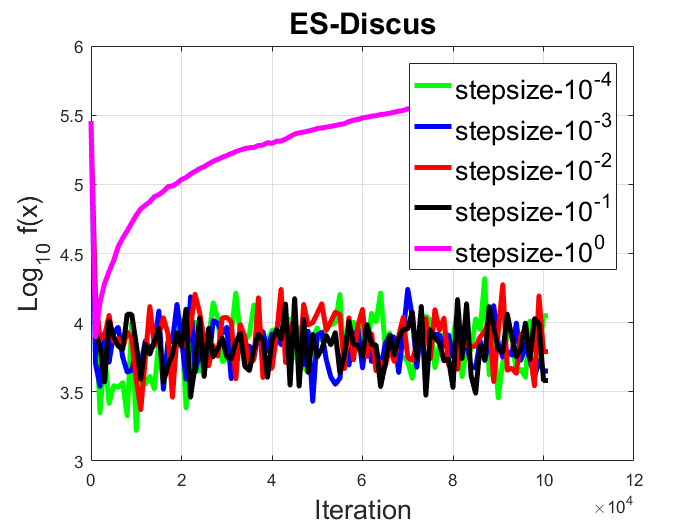}}
\subfigure[ES on Levy]{
\includegraphics[width=0.3\linewidth]{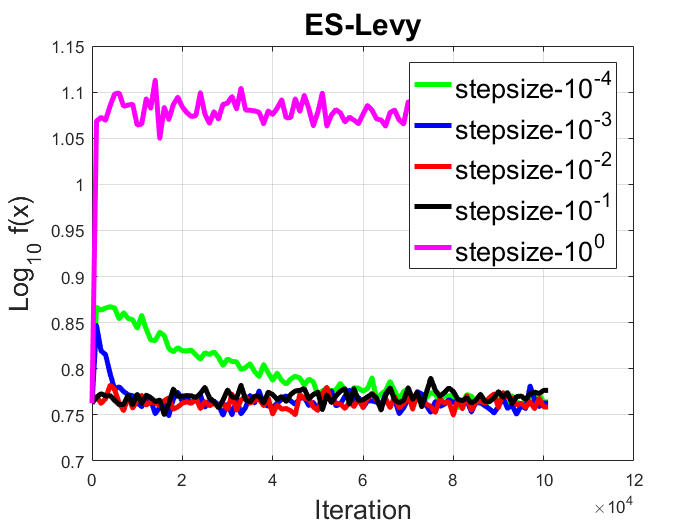}}
\subfigure[ES on Rastrigin10]{
\includegraphics[width=0.3\linewidth]{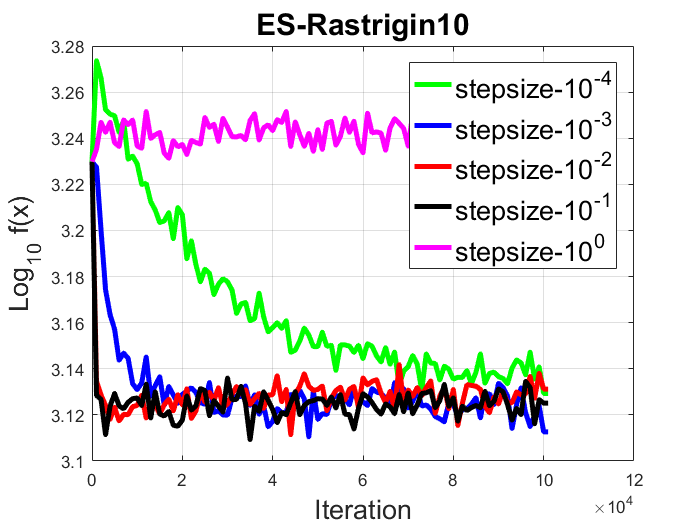}}
\caption{ES with different stepsize on different test functions}
\label{ESstep}
\end{figure*}

\begin{figure*}[t]
\centering
\subfigure[ Reinforce GD on Ellipsoid]{
\label{Ellipsoid}
\includegraphics[width=0.3\linewidth]{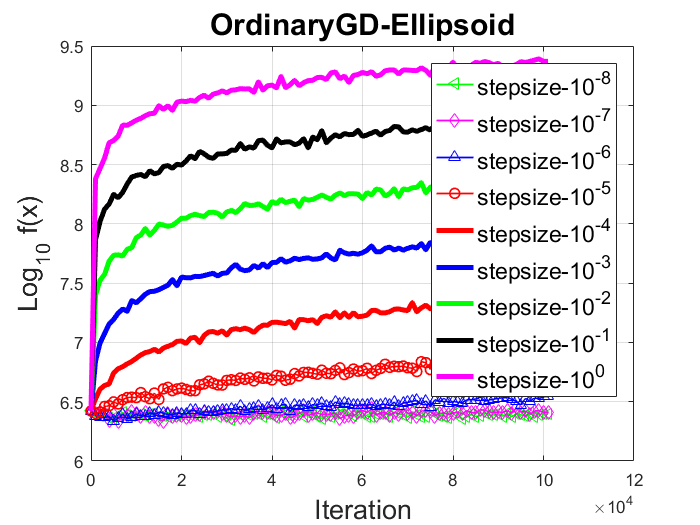}}
\subfigure[ Reinforce GD on $l_1$ Ellipsoid]{
\label{l1Ellipsoid}
\includegraphics[width=0.3\linewidth]{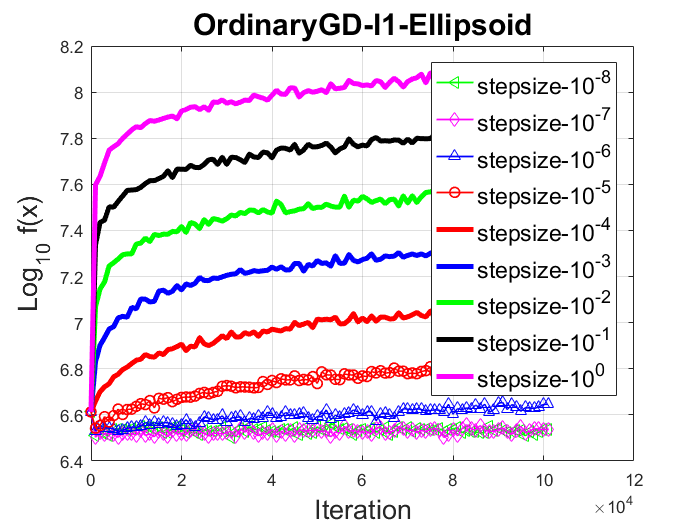}}
%\caption{ }
\subfigure[ Reinforce GD on $l_{\frac{1}{2}}$ Ellipsoid]{
\includegraphics[width=0.3\linewidth]{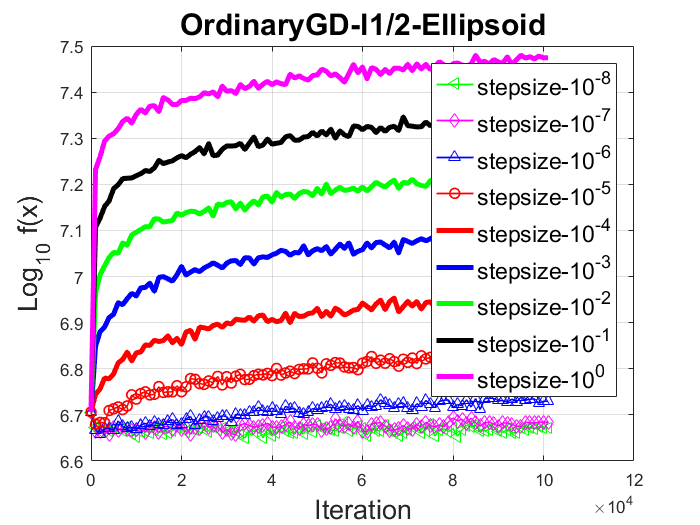}}
%\caption{ }
\subfigure[ Reinforce GD on Discus]{
\includegraphics[width=0.3\linewidth]{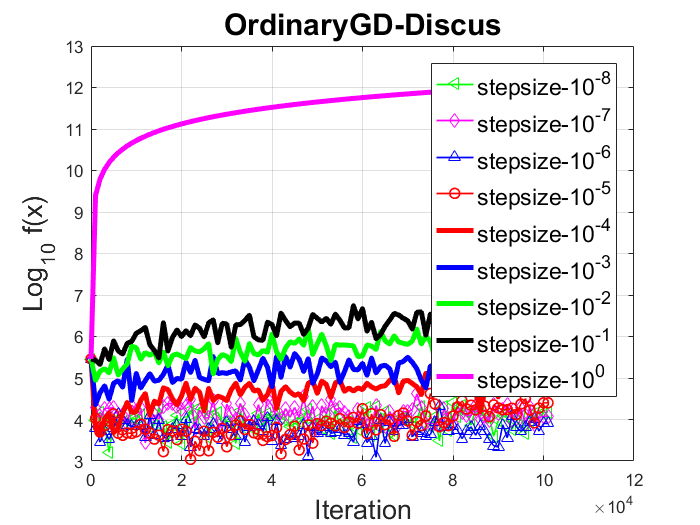}}
\subfigure[ Reinforce GD on Levy]{
\label{Levy}
\includegraphics[width=0.3\linewidth]{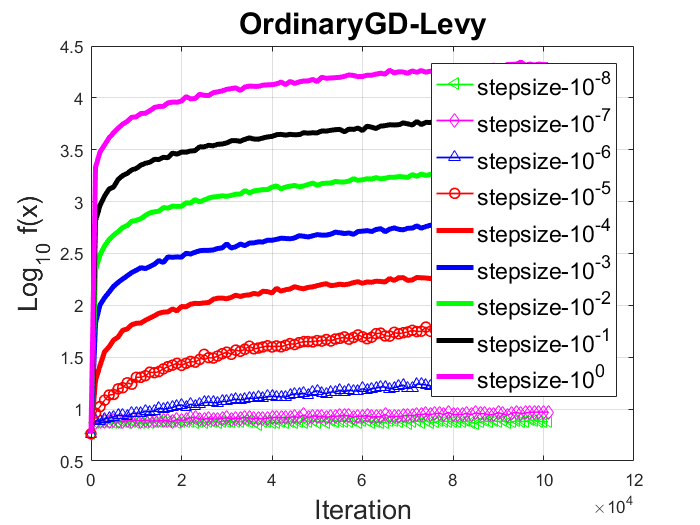}}
%\caption{ }
\subfigure[ Reinforce GD on Rastrigin10]{
\label{Rastrigin10}
\includegraphics[width=0.3\linewidth]{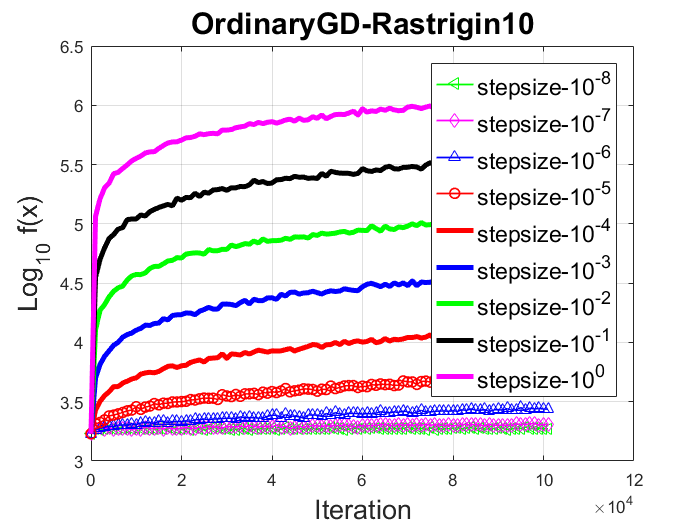}}
\caption{Standard Reinforce Gradient Descent with different stepsize on different test functions}
\label{RGD}
\end{figure*}

\hfill

%\newpage

%% file: neurips_2020.bbl
\begin{thebibliography}{10}

\bibitem{akimoto2010bidirectional}
Youhei Akimoto, Yuichi Nagata, Isao Ono, and Shigenobu Kobayashi.
\newblock Bidirectional relation between cma evolution strategies and natural
  evolution strategies.
\newblock In {\em International Conference on Parallel Problem Solving from
  Nature}, pages 154--163. Springer, 2010.

\bibitem{amari1998natural}
Shun-Ichi Amari.
\newblock Natural gradient works efficiently in learning.
\newblock {\em Neural computation}, 10(2):251--276, 1998.

\bibitem{amari2016information}
Shun-ichi Amari.
\newblock {\em Information geometry and its applications}, volume 194.
\newblock Springer, 2016.

\bibitem{azoury2001relative}
Katy~S Azoury and Manfred~K Warmuth.
\newblock Relative loss bounds for on-line density estimation with the
  exponential family of distributions.
\newblock {\em Machine Learning}, 43(3):211--246, 2001.

\bibitem{back1991survey}
Thomas Back, Frank Hoffmeister, and Hans-Paul Schwefel.
\newblock A survey of evolution strategies.
\newblock In {\em Proceedings of the fourth international conference on genetic
  algorithms}, volume~2. Morgan Kaufmann Publishers San Mateo, CA, 1991.

\bibitem{balasubramanian2018zeroth}
Krishnakumar Balasubramanian and Saeed Ghadimi.
\newblock Zeroth-order (non)-convex stochastic optimization via conditional
  gradient and gradient updates.
\newblock In {\em Advances in Neural Information Processing Systems}, pages
  3455--3464, 2018.

\bibitem{BORL}
Juan~Cruz Barsce, Jorge~A Palombarini, and Ernesto~C Mart{\'\i}nez.
\newblock Towards autonomous reinforcement learning: Automatic setting of
  hyper-parameters using bayesian optimization.
\newblock In {\em Computer Conference (CLEI), 2017 XLIII Latin American}, pages
  1--9. IEEE, 2017.

\bibitem{bull}
Adam~D Bull.
\newblock Convergence rates of efficient global optimization algorithms.
\newblock {\em Journal of Machine Learning Research (JMLR)},
  12(Oct):2879--2904, 2011.

\bibitem{choromanskicomplexity}
Krzysztof Choromanski, Aldo Pacchiano, Jack Parker-Holder, and Yunhao Tang.
\newblock From complexity to simplicity: Adaptive es-active subspaces for
  blackbox optimization.
\newblock {\em arXiv:1903.04268}, 2019.

\bibitem{choromanski2019complexity}
Krzysztof Choromanski, Aldo Pacchiano, Jack Parker-Holder, Yunhao Tang, and
  Vikas Sindhwani.
\newblock From complexity to simplicity: Adaptive es-active subspaces for
  blackbox optimization, 2019.

\bibitem{choromanski2018structured}
Krzysztof Choromanski, Mark Rowland, Vikas Sindhwani, Richard~E Turner, and
  Adrian Weller.
\newblock Structured evolution with compact architectures for scalable policy
  optimization.
\newblock In {\em ICML}, pages 969--977, 2018.

\bibitem{domke2019provable}
Justin Domke.
\newblock Provable smoothness guarantees for black-box variational inference.
\newblock {\em arXiv preprint arXiv:1901.08431}, 2019.

\bibitem{hansen2006cma}
Nikolaus Hansen.
\newblock The cma evolution strategy: a comparing review.
\newblock In {\em Towards a new evolutionary computation}, pages 75--102.
  Springer, 2006.

\bibitem{khan2017conjugate}
Mohammad~Emtiyaz Khan and Wu~Lin.
\newblock Conjugate-computation variational inference: Converting variational
  inference in non-conjugate models to inferences in conjugate models.
\newblock {\em arXiv preprint arXiv:1703.04265}, 2017.

\bibitem{khan2018fast}
Mohammad~Emtiyaz Khan and Didrik Nielsen.
\newblock Fast yet simple natural-gradient descent for variational inference in
  complex models.
\newblock In {\em 2018 International Symposium on Information Theory and Its
  Applications (ISITA)}, pages 31--35. IEEE, 2018.

\bibitem{khan2018fast2}
Mohammad~Emtiyaz Khan, Didrik Nielsen, Voot Tangkaratt, Wu~Lin, Yarin Gal, and
  Akash Srivastava.
\newblock Fast and scalable bayesian deep learning by weight-perturbation in
  adam.
\newblock In {\em ICML}, 2018.

\bibitem{Liu2019}
Guoqing Liu, Li~Zhao, Feidiao Yang, Jiang Bian, Tao Qin, Nenghai Yu, and
  Tie-Yan Liu.
\newblock Trust region evolution strategies.
\newblock In {\em AAAI}, 2019.

\bibitem{robot}
Daniel~J Lizotte, Tao Wang, Michael~H Bowling, and Dale Schuurmans.
\newblock Automatic gait optimization with gaussian process regression.
\newblock In {\em IJCAI}, volume~7, pages 944--949, 2007.

\bibitem{lyu2017spherical}
Yueming Lyu.
\newblock Spherical structured feature maps for kernel approximation.
\newblock In {\em Proceedings of the 34th International Conference on Machine
  Learning (ICML)}, pages 2256--2264, 2017.

\bibitem{lyu2019efficient}
Yueming Lyu, Yuan Yuan, and Ivor~W Tsang.
\newblock Efficient batch black-box optimization with deterministic regret
  bounds.
\newblock {\em arXiv preprint arXiv:1905.10041}, 2019.

\bibitem{drugdiscovery}
Diana~M Negoescu, Peter~I Frazier, and Warren~B Powell.
\newblock The knowledge-gradient algorithm for sequencing experiments in drug
  discovery.
\newblock {\em INFORMS Journal on Computing}, 23(3):346--363, 2011.

\bibitem{nesterov2017random}
Yurii Nesterov and Vladimir Spokoiny.
\newblock Random gradient-free minimization of convex functions.
\newblock {\em Foundations of Computational Mathematics}, 17(2):527--566, 2017.

\bibitem{ollivier2017information}
Yann Ollivier, Ludovic Arnold, Anne Auger, and Nikolaus Hansen.
\newblock Information-geometric optimization algorithms: A unifying picture via
  invariance principles.
\newblock {\em The Journal of Machine Learning Research (JMLR)},
  18(1):564--628, 2017.

\bibitem{raskutti2015information}
Garvesh Raskutti and Sayan Mukherjee.
\newblock The information geometry of mirror descent.
\newblock {\em IEEE Transactions on Information Theory}, 61(3):1451--1457,
  2015.

\bibitem{rezende2014stochastic}
Danilo~Jimenez Rezende, Shakir Mohamed, and Daan Wierstra.
\newblock Stochastic backpropagation and approximate inference in deep
  generative models.
\newblock In {\em ICML}, 2014.

\bibitem{salimans2017evolution}
Tim Salimans, Jonathan Ho, Xi~Chen, Szymon Sidor, and Ilya Sutskever.
\newblock Evolution strategies as a scalable alternative to reinforcement
  learning.
\newblock {\em arXiv preprint arXiv:1703.03864}, 2017.

\bibitem{Nips2012practical}
Jasper Snoek, Hugo Larochelle, and Ryan~P Adams.
\newblock Practical bayesian optimization of machine learning algorithms.
\newblock In {\em NeurIPS}, pages 2951--2959, 2012.

\bibitem{srinivas1994genetic}
Mandavilli Srinivas and Lalit~M Patnaik.
\newblock Genetic algorithms: A survey.
\newblock {\em computer}, 27(6):17--26, 1994.

\bibitem{gpucb}
Niranjan Srinivas, Andreas Krause, Sham~M Kakade, and Matthias Seeger.
\newblock Gaussian process optimization in the bandit setting: No regret and
  experimental design.
\newblock In {\em ICML}, 2010.

\bibitem{engineerDisign}
G~Gary Wang and Songqing Shan.
\newblock Review of metamodeling techniques in support of engineering design
  optimization.
\newblock {\em Journal of Mechanical design}, 129(4):370--380, 2007.

\bibitem{NES}
Daan Wierstra, Tom Schaul, Tobias Glasmachers, Yi~Sun, Jan Peters, and
  J{\"u}rgen Schmidhuber.
\newblock Natural evolution strategies.
\newblock {\em The Journal of Machine Learning Research (JMLR)},
  15(1):949--980, 2014.

\bibitem{xu2011deterministic}
Zhiqiang Xu.
\newblock Deterministic sampling of sparse trigonometric polynomials.
\newblock {\em Journal of Complexity}, 27(2):133--140, 2011.

\end{thebibliography}
